\documentclass{article}
\usepackage{amsmath, amssymb, amsthm} 
\usepackage{amsfonts}
\usepackage{graphicx}
\usepackage{fullpage}
\usepackage{booktabs}
\usepackage{natbib}
\usepackage{thm-restate}

\newcommand{\MSE}[1]{\mathrm{MSE}(#1)}
\usepackage[utf8]{inputenc}
\usepackage{amsmath,amssymb}
\usepackage{bbm}

\usepackage{hyperref}
\usepackage{cleveref}

\usepackage{amsthm}
\newtheorem{theorem}{Theorem}[section]
\newtheorem{lemma}[theorem]{Lemma}

\newtheorem{definition}[theorem]{Definition}
\newtheorem{corollary}[theorem]{Corollary}

\newtheorem{proposition}[theorem]{Proposition}
\newtheorem{remark}[theorem]{Remark}

\newtheorem{assumption}{Assumption}

\usepackage{xcolor}

\usepackage{comment}
\usepackage{mathtools}

\usepackage[shortlabels]{enumitem}

\newcommand{\E}{\mathbb{E}}

\usepackage[shortlabels]{enumitem}

\newcommand{\ignore}[1]{}

\usepackage[shortlabels]{enumitem}

\usepackage[shortlabels]{enumitem}

\usepackage{caption}
\usepackage{xcolor}

\DeclarePairedDelimiter{\norm}{\|}{\|}

\newcommand{\cD}{\mathcal{D}}

\newcommand{\cG}{\mathcal{G}}
\newcommand{\cH}{\mathcal{H}}

\newcommand{\cX}{\mathcal{X}}

\usepackage{nicematrix}

\newcommand{\R}{\mathbb{R}}

\usepackage{algorithm}
\usepackage[noend]{algpseudocode}

\title{Networked Information Aggregation via Machine Learning}
\author{Michael Kearns \and Aaron Roth \and Emily Ryu} 
\date{\today}

\begin{document}
\maketitle
\begin{abstract}
We study a distributed learning problem in which learning agents are embedded in a directed acyclic graph (DAG). There is a fixed and arbitrary distribution over feature/label pairs, and each agent or vertex in the graph is able to directly observe only a subset of the features --- potentially a different subset for every agent. The agents learn sequentially in some order consistent with a topological sort of the DAG, committing to a model mapping observations to predictions of the real-valued label. Each agent observes the predictions of their parents in the DAG, and trains their model using both the features of the instance that they directly observe, and the predictions of their parents as additional features. We ask when this process is sufficient to achieve \emph{information aggregation}, in the sense that some agent in the DAG is able to learn a model whose error is competitive with the best model that could have been learned (in some hypothesis class) with direct access to \emph{all} features, despite the fact that no single agent in the network has such access. We give upper and lower bounds for this problem for both linear and general hypothesis classes. Our results identify the \emph{depth} of the DAG as the key parameter: information aggregation can occur over sufficiently long paths in the DAG, assuming that all of the relevant features are well represented along the path, and there are distributions over which information aggregation cannot occur even in the linear case, and even in arbitrarily large DAGs that do not have sufficient depth (such as a hub-and-spokes topology in which the spoke vertices collectively see all the features). We complement our theoretical results with a comprehensive set of experiments.  
\end{abstract}

\thispagestyle{empty} \setcounter{page}{0}
\clearpage

 \tableofcontents
 \thispagestyle{empty} \setcounter{page}{0}
 \clearpage

\section{Introduction}

There is a rich literature dating back to \cite{degroot1974reaching} studying \emph{social learning in networks}. Motivating this literature are settings in which information is distributed heterogeneously across different parties who are all trying to solve the same prediction problem. The parties make predictions publicly, and those predictions are observed by others --- and so may in turn influence others' subsequent predictions --- but the parties do not directly share their observations. This literature seeks to understand when social learning of this sort is able to recover estimates of the underlying ground truth that are as accurate as if all of the information had been aggregated and used by a central party to learn. 

This literature has focused on either extremely simple learning heuristics (e.g. so-called DeGroot dynamics \citep{degroot1974reaching} that involve each party naively averaging the opinions/predictions of their neighbors in the graph), or has studied implausibly powerful learners (e.g. perfect Bayesians who communicate their entire posterior distribution to their neighbors). Most theoretical analyses in this literature assume for tractability that the learning instance is extremely simple (e.g. independent and unbiased signals of a real valued ground state). In this paper, we revisit the problem of learning on networks from a machine learning perspective. There is an arbitrary distribution over feature/label pairs, and the features (which may be correlated with one another in arbitrary ways) are arbitrarily distributed across the agents, not necessarily forming a partition. Rather than using simple averaging heuristics or perfect Bayes optimal learning, agents solve empirical risk minimization problems (e.g. least squares regression) using their neighbors' predictions as features. We ask: when is this process sufficient to converge to a model that is as accurate as one that could have been obtained had a single party solved a risk minimization problem over the same model class, with access to all of the relevant features, despite the fact that no single party had access to all features?

We formalize this question as follows: There is a distribution $\cD \in \Delta (\cX \times \mathbb{R})$ over feature/label pairs. $\cX \subseteq \mathbb{R}^d$ is a $d$-dimensional feature space. Agents $i$ are embedded in a directed acyclic graph, and each directly observes only some subset $S_i \subseteq [d]$ of the features of each example, which we write as $x_{S_i}$ for an example $x \in \cX$. In some topologically sorted ordering of the agents, the agents $i$ sequentially train predictors $f_i$ that are functions both of their own observed features $x_{S_i}$ and of the predictions made by their parents in the graph. We ask: 
\begin{quote}
\textit{Under what circumstances (and how quickly) does this process result in predictions that successfully aggregate information across agents, in the sense that they are as accurate as the best predictions that could be made as a function of the union of all of the features, despite the fact that every agent in the learning process has only a much more limited view? }
\end{quote}

\subsection{Our Results and Techniques}
We first study a linear regression setting in which agents learn linear models of their observed features $x_{S_i}$ and the predictions $\hat y_j$ of their predecessors $j$. 

\paragraph{Intuition for Non-Triviality.}
It is tempting to think that because any particular linear model is additively separable across features, then the presence of a connected subgraph (or perhaps a chain) of agents, the union of whose observations comprise the whole feature space, should be sufficient to learn the optimal linear model. This is not true, as the following example (which we extend to a more general lower bound in  Section \ref{sec:lower-bounds}) demonstrates. Consider a $d = 2$ dimensional instance in which $x_1$ and $y$ are both independent Gaussian random variables, and $x_2 = y - x_1$. Given access to both features, there is a perfect linear predictor: $y = x_1 + x_2$. But suppose we have two agents arranged in a line: $A_1 \to A_2$, with agent 1 observing the first feature and agent 2 observing the second feature. Despite the fact that the optimal \emph{joint} predictor puts unit weight on both $x_1$ and $x_2$, the optimal predictor of $y$ based on the first feature $x_1$ alone puts no weight on it, and is instead the constant function $f_1(x) = 0$. This is because $x_1$ is independent of $y$ --- it contains no signal on its own, without access to $x_2$. Hence, agent 1's prediction --- which is all that gets passed to agent 2 --- contains no information at all about $x_1$. Thus agent 2 must make her prediction as a function of $x_2$ alone: As a result, she will not be able to make predictions that have mean squared error (MSE) below $1/2$, despite the fact that there is a perfect linear predictor that jointly uses her feature and that of her direct predecessor. As a result, even for linear predictors, in order to compete with the best model defined on all features, it is insufficient to simply have a subtree or path in the graph that includes agents whose observations span all of the features.

\paragraph{Upper Bound in the Linear Case.}
Nevertheless, in Section \ref{sec:linear} we show a general positive result for linear learners on graphs that have depth $D$ --- i.e. contain paths of length $D$. Consider any assignment of features to agents such that every contiguous interval of $M$ agents along a path cover all $d$ features: for example, if in a path each agent sees a single feature in round-robin order, then $M = d$; if each agent sees a random subset of half of the features, then with high probability $M = \tilde O(\log d)$. Then, as we show in Corollary \ref{cor:best_linear_error}, the final agent in the path learns a model that is competitive with the best linear model defined on all features, up to error $\eta = O\left(\frac{M}{\sqrt{D}} \right)$. In other words, roughly speaking, to get excess error at most $\eta$ relative to the best linear model on all features, this bound tells us that it suffices for  the graph to contain a path of agents with length scaling as $1/\eta^2$. We complement this with a lower bound in Section \ref{sec:lower_bound_path}, demonstrating a family of instances consisting of single paths on which $\eta = \Omega\left(\frac{M}{D}\right)$ ---  i.e. in which seeing each feature at least $1/\eta$ times is unavoidable.

We sketch our analysis here, because it isolates two key ``orthogonality'' requirements of the learning process that lets us generalize  our theorems beyond linear regression. For notational brevity, in the definitions below
and elsewhere we use $f$ and $g$ as shorthand for $f(z)$ and $g(z)$, where in general $z$ consists of the 
features directly observable to an agent/model as well as the predictions of their parents,
and the expectations are with respect to the distribution on $x,y$.
\begin{enumerate}
    \item \textbf{Multiaccuracy}: A predictor $f$ is said to satisfy \emph{multiaccuracy} \citep{multicalibration,multiaccuracy} with respect to a collection of functions $G$ if each function $g \in  \cG$ is orthogonal to $f$'s residuals:
    $$\E[g(f-y)] = 0$$
    \item \textbf{Self-Orthogonality}:  A predictor $f$ is said to satisfy self-orthogonality if it is orthogonal to its own residuals:
    $$\E[f(f-y)] = 0$$
    This is what \cite{gopalan2022low} call ``degree 2 calibration''. 
\end{enumerate}
We use the fact that a predictor $f$ that is multiaccurate with respect to a collection of functions $G$ and also self-orthogonal has mean squared error that is at most that of the optimal function in $G$ --- this is a consequence of the simple MSE decomposition that we state in Lemma \ref{lem:mse_decomposition}:
$$\E[(f-y)^2] \leq \min_{g \in G}\E[(g-y)^2]$$
and is closely related to a theorem from \cite{gopalan2022low} showing that ``degree 2 multicalibration'' suffices for error optimality within a class $G$ in squared error. Our requirements are slightly weaker\footnote{In fact, their proof (Propostion A.1 of \cite{gopalan2022low}) also only uses these two conditions. We call out the weaker statement here explicitely because these weaker properties --- but not full degree 2 multicalibration --- follow ``for free'' from linear regression, which is important for our application.} --- multiaccuracy is what \cite{gopalan2022low} call degree 1 multicalibration, and so we need only ``degree 1 multicalibration'' and ``degree 2 calibration'', echoing recent work on relaxing multicalibration requirements to ``calibrated multiaccuracy'' requirements \citep{gopalan2023loss,casacuberta2025global}.

Linear regression always produces a predictor $f$ that is multiaccurate with respect to all linear functions of its inputs --- this is a direct consequence of the first order optimality conditions of squared error, and is a standard method for obtaining multiaccurate predictions \citep{gopalan2023loss}. Since $f$ is itself a linear function of its inputs, this in particular means that linear regression also  always produces a self-orthogonal predictor.   

Now recall that one of the inputs to agent $i$'s model $f_i$ is the prediction $f_{j}$ of any agent $j$ that is a parent of $i$ in the graph. Hence, it follows from $f_i$'s multiaccuracy guarantee that $\E[f_{j}(f_i-y)] = 0$. Lemma \ref{lem:mse_improvement_closeness_dag} shows that this condition implies that:
\[ \mathbb{E}[(f_i - f_{j})^2] = \mathbb{E}[(f_{j} - y)^2] - \mathbb{E}[(f_i - y)^2]. \]

In other words, if agent $i$'s model $f_i$ has MSE that is close to the MSE of it's parent agent $j$'s model $f_{j}$, then \emph{their predictions must also be close}. Since each agent $i$ uses their parents' predictions as inputs, it also follows that MSE must be monotonically decreasing along any path in the graph. Because MSE is always non-negative, this monotonicity property implies that if  the length of any path is sufficiently large, there must exist a long subsequence of agents $A_\ell \to \ldots \to A_j$ whose MSE is very similar: $|\E[(f_j-y)^2] - \E[(f_\ell-y)^2]  \leq \epsilon$. The above Lemma therefore implies that the predictions of every agent along this subsequence are also very close: $\sum_{i=\ell+1}^j \E[(f_i-f_{i-1})^2] \leq \epsilon$. Since each agent  $i$ produces a predictor  $f_i$ that  is multiaccurate with respect to their own features $S_i$, but every other agent in this stretch has a very similar predictor, this means that  for every agent $i$ in this stretch, its predictor is in fact approximately multiaccurate with respect to  the \emph{union} of all features available to any agent in the  stretch. Since  the models $f_i$ are also  all self-orthogonal, if the stretch is long enough to span  \emph{all} of the features, then the predictors must in fact be competitive with the best linear predictor  trained jointly on all features --- i.e. information  must have been successfully aggregated. We perform our main analysis in the distributional regime (where we assume each agent can directly optimize directly on the distribution), but prove uniform convergence results that let us extend our results to the finite sample regime in Appendix \ref{sec:generalization-linear}. 

\paragraph{Beyond Linear Models} Linear regression produces models that are multiaccurate with respect to their input features --- crucially including the prediction of their predecessor --- and self-orthogonal. But those were the \emph{only} two properties of linear regression that drove our analysis. Hence, to generalize our results to more complicated function classes $\cH$, it suffices to be able to guarantee these two properties. This is what we do in Section \ref{sec:general_classes}. In this section we now assume that each agent $A_i$ has the ability to perform MSE optimization over an arbitrary hypothesis class $H_i$, which could e.g. represent decision trees or neural networks defined over the subset of features directly observable by them. The algorithm we analyze (``Greedy Orthogonal Regression'') then proceeds as follows: each agent $i$ will iteratively build up a set of features $\mathcal{F}_i$, initially consisting of only the predictions of their parents. Iteratively, agent $i$ will train a linear regression model over the features $\mathcal{F}_i$, and then perform squared error regression over their hypothesis class $H_i$ (which again may take directly observable features as inputs) to minimize the prediction error of a model $h$ with respect to the residuals of their current linear regression model --- or equivalently to find a model $h \in \cH_i$ that maximizes the correlation with the residuals of their linear regression model. If the correlation of the newly discovered model exceeds a threshold $\Delta$, then $h$ is added to the feature set $\mathcal{F}_i$, and agent $i$ re-solves the linear regression problem over this newly expanded feature set. Finally, at termination agent $i$ uses the optimal linear regression model over their final feature set. This resembles gradient-boosting-like  algorithms that have previously been used to obtain multiaccurate predictors \cite{multicalibration, multiaccuracy} but differs in that rather than simply adding a new component to an additive model at each iteration, it re-solves for the weights of previously added components as well. The result is that in addition to guaranteeing that the learned model is approximately multi-accurate with respect to $H_i$, it guarantees that it is exactly orthogonal with respect to its parents predictions, and exactly self-orthogonal --- the conditions we need to drive our analysis. Our final guarantee is similar to what we offer in the linear case: in any DAG that has sufficient depth $D$, learners are able to compete up to excess error $\eta$ with a benchmark that corresponds to the MSE of the best predictor in the span of $M$ hypothesis classes $H_i$ appearing in sequence in some path, where $\eta$ scales with $O\left(\frac{M}{\sqrt{D}} \right)$. Once again we prove our main results in distribution, and then give generalization bounds that extend our theorems to the finite sample setting in Appendix \ref{sec:nonlinear_generalization}.  

\paragraph{Lower Bounds} In Section \ref{sec:lower-bounds} we prove two lower bounds, both based on the same lower bound distribution (Definition \ref{def:lb}). Our upper bounds, informally speaking, show that if a DAG contains a path with length scaling as $O(1/\eta^2)$ that also contains each variable many times, then this is sufficient for an agent at the end of this path to have found a predictor that has excess error at most $\eta$ compared to the optimal predictor. We first give a lower bound construction (in Section \ref{sec:lower_bound_path}) witnessing an instance on which it really is necessary to see the entire sequence of variables $\Omega(1/\eta)$ many times in sequence before any agent is able to learn a predictor with excess error $\eta$.

Then, in Section \ref{sec:lower_bound_general}, we show that \emph{depth} is the a crucial property of a DAG for information aggregation. Specifically, we show that for any depth $D$, there is a distribution (with dimension $D+1$) that simultaneously has the property that:
\begin{enumerate}
    \item There is a perfect linear predictor (i.e. with MSE 0) depending on all of the variables, but
    \item For \emph{any} DAG of arbitrary topology and depth at most $D$, and for \emph{any} assignment of single variables to the agents in this DAG, \emph{no} agent is able to learn a predictor with MSE lower than $1/(D+1)$.
\end{enumerate}

We note that our lower bound implies that even in a hub-and-spokes topology in which the spoke vertices collectively see all the features, in the worst case there is a constant lower bound on the error of the hub --- even in the linear case in which least squares model on all features has zero error.

Our upper bounds promise diminishing excess error as a function of the depth of a graph --- this lower bound shows that no similar bound is possible for graph properties (like e.g. subtree size) that do not imply growing depth --- at least in the worst case. However, our experimental results (summarized next) demonstrate that on real datasets and natural graph topologies and feature sets, other topological properties can indeed strongly correlate with agent performance. 

\paragraph{Experiments} 
In Section~\ref{sec:experiments}, we illustrate our theoretical framework and results with
an experimental study on two real datasets. We consider both simple chain and tree topologies for
the agents, and for the latter consider both top-down and bottom-up aggregation of information.
Our findings confirm and illustrate the theory, with fast convergence to the optimal model with
the depth of the topology, but also show empirically that quantities such as subtree size can
correlate strongly performance, suggesting more nuanced theoretical analyses that move beyond
the worst case.

\subsection{Additional Related Work}
\paragraph{Opinion Dynamics and Consensus Models}
There is a large literature on opinion dynamics in networks, beginning with \cite{degroot1974reaching}. In the \cite{degroot1974reaching} model, each agent in a graph begins with some belief (a distribution over a finite number of states), and then at each iteration simply broadcasts their belief to the neighbors in the graph, and then updates their belief as a fixed convex combination of the beliefs of their neighbors --- i.e. using a simple heuristic update rule, rather than one designed to solve inference problems. \cite{degroot1974reaching} characterizes when this dynamic leads to consensus. \cite{bala1998learning} and \cite{ellison1993rules} also study heuristic models of opinion dynamics on graphs. \cite{golub2010naive} study information aggregation under DeGroot dynamics: they study a setting in which the true state is an unknown scalar value, and each agent receives an independent unbiased estimate of the state. They study conditions on the graph such that in the limit, DeGroot dynamics converge to the true state. \cite{akbarpour2017information} study a variant in which agents arrive and depart over time (so information flows ``forward'' as in our model) and study several different learning DeGroot style learning heuristics.

\paragraph{Bayesian Learning in Graphs}
Another line of work with roots going back to \citep{aumann1976agreeing,banerjee1992simple,bikhchandani1992theory} studies the dynamics of Bayesian learning --- i.e. settings in which all agents have a correct prior on the set of all joint observations and outcomes, and are able to perform perfect Bayesian updates. Aumann's classic ``agreement theorem'' \citep{aumann1976agreeing} states that two such Bayesians, who repeatedly exchange their posterior expectations (and update their posteriors as a result of the exchange) will in the limit converge to the same posterior. \citep{geanakoplos1982we} and later \citep{aaronson2005complexity} extend this limiting analysis to interactions of finite length, between multiple agents arranged on a communication network with concrete rates at which approximate agreement is reached. As was already noted in \cite{geanakoplos1982we} agreement among Bayesians does not imply information aggregation --- i.e. through exchanging beliefs, Bayesians might \emph{agree} on a posterior that is less accurate than they would have arrived at had they directly shared their observations. \cite{bo2023agreement} and \cite{kong2023false} give conditions on the prior distribution under which Bayesian agreement does imply information aggregation.

\cite{banerjee1992simple} and \cite{bikhchandani1992theory} study simple models in which Bayesian agents act sequentially, each taking an action that maximizes their utility in expectation over their posterior beliefs of a single, static state. In these models, agents observe the actions of all previous agents, and the main phenomenon of interest is ``herding'' --- i.e. when an infinite sequence of agents takes the wrong action, as the cumulative decisions of their predecessors overwhelm the strength of their own signal. An interesting feature of our model is that herding/information cascades cannot occur indefinitely --- this is a consequence of our information aggregation theorems, and a qualitative distinction between our work and models that study binary action games. \cite{acemoglu2011bayesian} extend the earlier herding results to general networks in which agents do not observe the actions of all predecessors, but only those adjacent to them in the network. They study conditions (expansion properties) under which information aggregation occurs in the infinite limit. \cite{mossel2015strategic,mossel2020social} study a variant of this style of model in which agents act repeatedly rather than sequentially, and study information aggregation in equilibrium.  \cite{jadbabaie2012non,shahrampour2013exponentially} study a related model in which agents make Bayesian updates with respect to their own signals, but then broadcast their posterior and simply average their posterior beliefs with those of their neighbors, in the style of a \cite{degroot1974reaching} model. They similarly give conditions sufficient for information aggregation. 

These models differ from ours in a number of ways. In their models agents, are trying to predict some static single ``state'' that they all have a signal about --- in contrast, in our model, there is a distribution on outcomes $y$ that is correlated with observed features $x$, and agents are interested in learning a model mapping $x$ to $y$ to minimize MSE. In some sense our model can be viewed as simultaneously analyzing an infinite number of interactions, each defined by a single $(x,y)$ pair.  To make the models tractable to analyze, this literature generally restricts itself to extremely simple distributions, in which the state of the world and action space are both binary and the observations made by each agent are sampled independently conditional on the underlying state (\cite{jadbabaie2012non} is an exception to this, allowing complex correlated signals, although \cite{shahrampour2013exponentially} once again assume independent signals in order to derive concrete rates). In high dimensional feature spaces such as the ones we consider, computing Bayesian updates (even just with respect to one's own observations as in \citep{jadbabaie2012non}) is both computationally and statistically intractable.  In contrast, we make no assumptions at all on the underlying distribution, which supports high dimensional features which can be arbitrarily correlated with one another and with the label. Rather than giving guarantees of Bayes optimality we give guarantees relative to the best model in a fixed hypothesis class. Our analysis uses only orthogonality conditions that would be satisfied by Bayesians, and so we are studying a tractable \emph{relaxation} of Bayesian rationality assumptions, just as \cite{collina2025tractable,collina2025collaborative} are. We also give results with concrete rates for finite graphs rather than limiting results for infinite graphs.

\paragraph{Information Aggregation Using Tractable Calibration Conditions}
Technically, our analysis is most closely related to a recent line of work that uses conditional calibration conditions for ensembling and information aggregation \cite{roth2023reconciling,globus2023multicalibration,globus2024model,collina2025collaborative,collina2025tractable}. Two recent works of \cite{collina2025tractable,collina2025collaborative} in particular show how to recover the same quantitative results in the style of Aumann's agreement theorem \citep{aumann1976agreeing,aaronson2005complexity} and its information aggregation properties \cite{bo2023agreement} without assuming full Bayesian rationality, but instead assuming only simple \emph{calibration conditions} that can be tractably enforced even in sequential adversarial settings \citep{noarov2023high,arunachaleswaran2025elementary}. Our work is distinguished from this literature in two ways: first,  this literature considers agents who engage in an extended interaction in which each agent participates multiple times, solving many learning problems across this interaction. In contrast, in our model, each agent learns only once, and then fixes their model: information aggregation needs to occur as predictions percolate through a network of agents, no pair of which might observe the same set of features. Second, this literature assumes  stronger conditional calibration conditions (``multi-calibration'' \citep{multicalibration}) that require learning  more complicated non-linear models, even to compete with linear benchmark classes. In contrast, our analysis requires only much weaker orthogonality conditions (also known as ``multi-accuracy'' \citep{multicalibration,multiaccuracy}) which not only are satisfiable by linear models but are \emph{already} satisfied by least squares regression, and so need no further algorithmic intervention. Our model could be used to simulate the repeated interaction studied in \cite{collina2025collaborative,collina2025tractable} by constructing a path graph in which each agent was repeatedly represented by alternating nodes --- for squared error minimization over additively separable classes in the batch setting, our results would recover some of those of \cite{collina2025collaborative} with substantially simpler learning algorithms (note however that \cite{collina2025collaborative} also studies more complex settings with abstract action spaces and utility functions, the online adversarial setting, and a more general weak learning condition).

\section{Preliminaries}

We consider a distributed learning setting with a set of $N$ agents, $\mathcal{A} = \{A_1, \ldots, A_N\}$. The agents are organized in a Directed Acyclic Graph (DAG), $G = (\mathcal{A}, E)$, where an edge $(A_j, A_i) \in E$ indicates that agent $A_i$ receives information from agent $A_j$. We sometimes also write $A_j \to A_i$ to denote this relationship. We denote the set of parents of agent $A_i$ as $\text{Pa}(i) = \{A_j \mid (A_j, A_i) \in E\}$. The agents learn in an order consistent with a topological sort of the DAG, with ties in the topological ordering broken arbitrarily.

Let $[d] = \{1, 2, \ldots, d\}$ be the set of indices for $d$ total features. Each agent $A_i \in \mathcal{A}$ is associated with a specific subset of these features, $S_i \subseteq [d]$.
We assume data points $(x, y)$ are drawn from an underlying distribution $\mathcal{D}$, where $x \in \mathbb{R}^d$ is a feature vector and $y \in \mathbb{R}$ is the target label. For any agent $A_i$, its local view of the features is $x_{S_i}$, which is the sub-vector of $x$ corresponding to the features indexed by $S_i$.

The agents learn  models in a sequential manner. Initially we restrict attention to \emph{linear} models, and then generalize this in Section \ref{sec:general_classes}. Each agent $A_i$ aims to train a model $f_i$ to minimize the squared error $\mathbb{E}_{(x,y) \sim \mathcal{D}} [(y - \hat{y}_i)^2]$, where $\hat{y}_i$ is the prediction made by agent $A_i$.

Each agent $A_i$ observes its local features $x_{S_i}$ and the set of predictions from its parents, $\{\hat{y}_j\}_{j \in \text{Pa}(i)}$. The model $f_i$ for agent $A_i$ is a linear function of its inputs:
\[ f_i(x_{S_i}, \{\hat{y}_j\}_{j \in \text{Pa}(i)}) = w_i^T x_{S_i} + \sum_{j \in \text{Pa}(i)} v_{ij} \hat{y}_j \]
where $w_i$ and the coefficients $v_{ij}$ are learnable parameters, chosen to minimize the mean squared error $\mathbb{E}[(y - f_i)^2]$. The prediction of agent $A_i$ is $\hat{y}_i = f_i(x_{S_i}, \{\hat{y}_j\}_{j \in \text{Pa}(i)})$. For any agent $A_i$ with no parents (a root of the DAG), its model is simply $f_i(x_{S_i}) = w_i^T x_{S_i}$.

Our primary interest lies in understanding the conditions under which the predictions $\hat{y}$ of some agents $A_i$ in the graph achieve accuracy comparable to that of an optimal linear model $f^*(x) = w^{*T}x$ trained centrally on all $d$ features. The optimal benchmark model $f^*$ minimizes $\mathbb{E}[(y - w^T x)^2]$.

We measure the accuracy of predictors by their mean squared error:

\begin{definition}[Mean Squared Error] The performance of a predictor $f$ is measured by its Mean Squared Error, which we denote by $\MSE{f}$:
\[ \MSE{f} = \mathbb{E}[(f(x) - y)^2]. \]
\end{definition}

We sometimes take norms of random variables. 
\begin{definition}[The $L_2$ norm of a random variable]
The $L_2$ norm of a random variable $Z$ is defined as the square root of its expected squared magnitude:
\[ \left\| Z \right\|_{L_2} := \sqrt{\mathbb{E}[Z^2]}. \]
This norm is used to measure the magnitude of random variables or the distance between them.
\end{definition}

Multiaccuracy is an orthogonality condition that is central to our analysis \cite{multicalibration, multiaccuracy}:

\begin{definition}[Multiaccuracy] A predictor $f: \mathbb{R}^d \to \mathbb{R}$ is said to be $\epsilon$-\emph{multiaccurate} with respect to a collection of real-valued functions $\mathcal{G} = \{g | g: \mathbb{R}^d \to \mathbb{R}\}$ if, for every function $g \in \mathcal{G}$, the following condition holds:
\[ \left|\mathbb{E}_{(x,y) \sim \mathcal{D}} [g(x)(f(x) - y)] \right|\leq \epsilon \]
This means that the predictor's errors $f(x)-y$ have low correlation with every function $g$ in the collection $\mathcal{G}$. If $\epsilon = 0$ we simply say that $f$ is multiaccurate with respect to $\mathcal{G}$.
\end{definition}

Another important condition is self-orthogonality --- i.e. self-referential multi-accuracy. This condition is implied by (but is substantially weaker than) calibration.
\begin{definition}
    A predictor $f: \mathbb{R}^d \to \mathbb{R}$ is said to be $\epsilon$-\emph{self-orthogonal} if:
\[ \left|\mathbb{E}_{(x,y) \sim \mathcal{D}} [f(x)(f(x) - y)] \right|\leq \epsilon \]
If $\epsilon = 0$ we simply say that $f$ is self-orthogonal.
\end{definition}
Self orthogonality is what \cite{gopalan2023loss} define as ``degree 2 calibration''.

\section{Linear Learners}
\label{sec:linear}

In this section we establish upper bounds for the accuracy of linear learners in a DAG by focusing our analysis on long \emph{paths} within the DAG. We will later see in Section \ref{sec:lower-bounds} that our focus on depth is well justified. 

Throughout our analysis, we work in the distributional regime (i.e. in the population or infinite-sample limit), assuming that agents can directly optimize over the true distribution $\mathcal{D}$. This allows us to isolate the structural aspects of the equilibrium and learning dynamics. Specifically, under this assumption, linear regression produces predictors that satisfy exact multiaccuracy and exact self-orthogonality, the two key properties that drive our analysis. The finite-sample and out-of-sample generalization guarantees corresponding to these results are established separately in Appendices~\ref{sec:generalization_prelim}–\ref{sec:generalization-linear}.

\subsection{Structural Properties of Squared Error Regression}
We start by establishing several useful structural properties of linear least squares regression. The first is that a predictor learned by solving a linear least squares regression problem yields a predictor whose residuals are orthogonal to each of its features, and is therefore ``multiaccurate'' with respect to all linear functions of its features. This is a standard property of linear regression and has also been noted as a means to obtain multiaccuracy \cite{gopalan2023loss}. This follows directly from the first-order optimality conditions of squared loss; we provide a proof here for completeness.

\begin{lemma}
\label{lem:lsq_multiaccurate_linear}
If $f(x)$ is the least squares estimator for a target $y$ using features $x \in \mathbb{R}^k$ (i.e., $f(x) = w^{*T}x$ where $w^* = \arg\min_w \mathbb{E}[(y - w^T x)^2]$), then the residuals of $f(x)$ are orthogonal to each of its features. In other words for each index $j \in [k]$:
\[ \mathbb{E}[x_j(y - f(x))] = 0 \]
This implies that $f$ is multiaccurate with respect to the collection of all linear functions $g(x) = a^T x$ for any $a \in \mathbb{R}^k$.
\end{lemma}

\begin{proof}
The least squares estimator $f(x) = w^{*T}x$ is found by minimizing the expected squared error $L(w) = \mathbb{E}[(y - w^T x)^2]$.
The first-order optimality condition requires the gradient of $L(w)$ with respect to $w$ to be zero at $w = w^*$.
The gradient is:
\[ \nabla_w L(w) = \mathbb{E}[\nabla_w (y - w^T x)^2] = \mathbb{E}[2(y - w^T x)(-x)] = -2 \mathbb{E}[x(y - w^T x)] \]
Setting $\nabla_w L(w^*) = 0$ we obtain the vector equation:
\[ \mathbb{E}[x(y - f(x))] = 0 \]
which establishes that the residuals of $f(x)$ are orthogonal to each of its features. 

Now, consider an arbitrary linear function $g(x) = a^T x = \sum_{j=1}^k a_j x_j$, where $a \in \mathbb{R}^k$ is a vector of coefficients.
We want to show that  $\mathbb{E}[g(x)(f(x) - y)] = 0$. We can compute:
\begin{align*}
\mathbb{E}[g(x)(f(x) - y)] &= \mathbb{E}\left[\left(\sum_{j=1}^k a_j x_j\right)(f(x) - y)\right] \\
&= \mathbb{E}\left[\sum_{j=1}^k a_j x_j (f(x) - y)\right] \\
&= \sum_{j=1}^k a_j \mathbb{E}[x_j (f(x) - y)] \quad \text{(by linearity of expectation)} \\
&= 0
\end{align*}
where the last inequality follows from our first result. 

This holds for any choice of vector $a$, and thus for any linear function $g(x) = a^T x$. Therefore, the least squares estimator $f(x)$ is multiaccurate with respect to the collection of all linear functions of $x$.
\end{proof}

An immediate but useful corollary of Lemma \ref{lem:lsq_multiaccurate_linear} is that the solution to a least squares estimation problem gives a model $f$ that is orthogonal to its own residuals:
\begin{corollary}
\label{cor:orthogonal}
    If $f(x)$ is the least squares estimator for a target $y$ using features $x \in \mathbb{R}^k$ (i.e., $f(x) = w^{*T}x$ where $w^* = \arg\min_w \mathbb{E}[(y - w^T x)^2]$), then $f(x)$ is \emph{self-orthogonal}:
    $$\mathbb{E}[f(x)(f(x) - y)] = 0$$
\end{corollary}
\begin{proof}
    Lemma \ref{lem:lsq_multiaccurate_linear} establishes that $f$ is multi-accurate with respect to the class of all linear functions $g(x) = a^Tx$ --- and $f$ itself is such a function.
\end{proof}

Next we observe a useful identity that we will make use of several times: for any pair of predictors $f$ and $g$, we can express the MSE of $f$ in terms of:
\begin{enumerate}
    \item The MSE of $g$,
    \item The correlation that $g$ has with the residuals of $f$ (i.e. the multiaccuracy error that $f$ has with respect to $g$),
    \item The correlation that $f$ has with its own residuals, and
    \item The expected squared \emph{disagreement} between $f$ and $g$.
\end{enumerate}
This decomposition does not depend on linearity. 

\begin{lemma}[MSE Decomposition]
\label{lem:mse_decomposition}
For any two predictors $f$ and $g$, we have:
\[ \MSE{f} = \MSE{g} - 2\mathbb{E}[g(f - y)] + 2\mathbb{E}[f(f - y)]  - \mathbb{E}[(f-g)^2]. \]
\end{lemma}
\begin{proof}
We start by expanding $\MSE{f} = \mathbb{E}[(f-y)^2]$. We introduce $g$ by writing $f-y = (f-g) + (g-y)$:
\begin{align*} \MSE{f} &= \mathbb{E}[((f-g) + (g-y))^2] \\&= \mathbb{E}[(f-g)^2 + (g-y)^2 + 2(f-g)(g-y)] \\&= \mathbb{E}[(f-g)^2] + \MSE{g} + 2\mathbb{E}[(f-g)(g-y)]. \end{align*}
We rewrite the term $2\mathbb{E}[(f-g)(g-y)]$. Since $g-y = (f-y) - (f-g)$, we have:
\begin{align*}2\mathbb{E}[(f-g)(g-y)] &= 2\mathbb{E}[(f-g)((f-y) - (f-g))] \\ &= 2\mathbb{E}[(f-g)(f-y) - (f-g)^2] \\ &= 2\mathbb{E}[(f-y)(f-g)] - 2\mathbb{E}[(f-g)^2]. \end{align*} Substituting this back into our expression for $\MSE{f}$: \begin{align*} \MSE{f} &= \mathbb{E}[(f-g)^2] + \MSE{g} + (2\mathbb{E}[(f-y)(f-g)] - 2\mathbb{E}[(f-g)^2]) \\ &= \MSE{g} + 2\mathbb{E}[(f-y)(f-g)] - \mathbb{E}[(f-g)^2]. \end{align*}
\end{proof}

From the decomposition in Lemma \ref{lem:mse_decomposition} we can derive several useful properties. First, if we have a model $f$ that is approximately multiaccurate with respect to a class of functions $\mathcal{G}$, and is also approximately self-orthogonal, then $f$ must have MSE bounded by the MSE of the \emph{best} function in $\mathcal{G}$.

\begin{lemma}[Approximate Multiaccuracy and Error]
\label{lem:multiaccurate_error_bound}
Let $f(x)$ be a predictor for a target $y$. Let $\mathcal{G}$ be a collection of functions, and let $g^*(x) = \arg\min_{g \in \mathcal{G}} \mathbb{E}[(g(x)-y)^2]$ be the function in $\mathcal{G}$ that minimizes expected squared error.

Suppose $f(x)$ satisfies:
\begin{enumerate}
    \item \textbf{Approximate Self-Orthogonality:} $|\mathbb{E}[f(x)(f(x)-y)]| = \delta_f$.
    \item \textbf{Approximate Multiaccuracy w.r.t. $\mathcal{G}$:} $\max_{g \in \mathcal{G}}|\mathbb{E}[g(x)(f(x)-y)]| = \epsilon_g$.
\end{enumerate}
Then, its mean squared error is bounded as follows:
\[ \mathbb{E}[(f(x)-y)^2] \le \mathbb{E}[(g^*(x)-y)^2] + 2\epsilon_g + 2\delta_f. \]
\end{lemma}

\begin{proof}
The result follows directly from Lemma~\ref{lem:mse_decomposition} instantiated with $f$ and $g^*$:
\begin{eqnarray*}
    \MSE{f} &=& \MSE{g^*} - 2\mathbb{E}[g^*(f - y)] + 2\mathbb{E}[f(f - y)]  - \mathbb{E}[(f-g)^2] \\
    &\leq&\MSE{g^*} - 2\mathbb{E}[g^*(f - y)] + 2\mathbb{E}[f(f - y)]
\end{eqnarray*}
where the inequality follows from the fact that $\mathbb{E}[(f-g)^2]$ is non-negative. 
Substituting the given conditions, $|\mathbb{E}[f(x)(f(x)-y)]| = \delta_f$ and $|\mathbb{E}[g^*(x)(f(x)-y)]| \leq \epsilon_g$, immediately yields:
\[ \MSE{f} \le \MSE{g^*} + 2\delta_f + 2\epsilon_g. \]
This is the desired result.
\end{proof}

Before we move on to our analysis of learners in a graph, we extract one more simple but very useful consequence of the decomposition in Lemma \ref{lem:mse_decomposition}. If a model $f$ is self-orthogonal as well as being multiaccurate with respect to $g$, then not only does $f$ have weakly lower MSE compared to $g$, but the extent to which the actual predictions made by $f$ differ from those made by $g$ can be bounded by the improvement in MSE that $f$ obtains over $g$. Thus two such predictors with similar MSE must actually make very similar predictions. This next lemma can be viewed  as a``pythagorean theorem'' for orthogonal projections and is a corollary of other similar pythagorean theorems --- e.g. \cite{gopalan2022low} gives one for ``degree 2 multicalibration'' --- but we give the (very simple) proof here for completeness.
\begin{lemma}
\label{lem:stability}
    Let $f(x)$ and $g(x)$ be predictors for a target $y$. Suppose $f(x)$ satisfies:
\begin{enumerate}
    \item \textbf{Self-Orthogonality:} $\mathbb{E}[f(x)(f(x)-y)] = 0$.
    \item \textbf{Multiaccuracy w.r.t. $g$:} $\mathbb{E}[g(x)(f(x)-y)] = 0$.
\end{enumerate}
Then:
$$\mathbb{E}[(f-g)^2] = \MSE{g} - \MSE{f}$$
\end{lemma}
\begin{proof}
    This follows directly from Lemma \ref{lem:mse_decomposition} and the hypotheses of the Lemma. We have:
    \begin{eqnarray*}
        \MSE{f} &=& \MSE{g} - 2\mathbb{E}[g(f - y)] + 2\mathbb{E}[f(f - y)]  - \mathbb{E}[(f-g)^2] \\
        &=& \MSE{g} - \mathbb{E}[(f-g)^2]
    \end{eqnarray*}
Solving for $\mathbb{E}[(f-g)^2]$ gives the result.
\end{proof}

\subsection{Analysis of Learners Along a Path}
\label{sec:linear-path}
We will now apply these structural results to analyze the evolution of MSE for a sequence of learners arranged in a path in an arbitrary DAG.

First we observe that the agents in a DAG learn predictors that are multiaccurate with respect to linear functions on their own inputs \emph{as well as the predictions of their own parents}.

\begin{lemma}[Per-Step Multiaccuracy in DAG]
\label{lem:per_step_multiaccuracy_dag}
For any agent $A_i$ in the DAG, its predictor $\hat{y}_i$ is multiaccurate with respect to the set of all linear functions of its inputs. Specifically, for any linear function $g(x_{S_i}, \{\hat{y}_j\}_{j \in \text{Pa}(i)})$:
\[ \mathbb{E}[g(x_{S_i}, \{\hat{y}_j\}_{j \in \text{Pa}(i)})(\hat{y}_i - y)] = 0. \]
This implies the following key orthogonality conditions:
\begin{enumerate}
    \item For each local feature $x_k$ where $k \in S_i$: $\mathbb{E}[x_k (\hat{y}_i - y)] = 0$.
    \item For each parent's prediction $\hat{y}_j$ where $j \in \text{Pa}(i)$: $\mathbb{E}[\hat{y}_j (\hat{y}_i - y)] = 0$.
\end{enumerate}
\end{lemma}

\begin{proof}
This is a direct consequence of Lemma~\ref{lem:lsq_multiaccurate_linear} and the fact that each agent $i$ in a DAG uses the predictions of its parents as features.
\end{proof}

Next we note that the MSE of agents along any path in a DAG must be monotonically decreasing:

\begin{lemma}[Monotonically Non-Increasing Error along a Path]
\label{lem:mse_non_increasing_dag}
For any path $P = (A_{p_1}, \ldots, A_{p_k})$ in the DAG, the mean squared error (MSE) of a predictor $\hat{y}_{p_i}$ is less than or equal to the MSE of its predecessor's predictor on the path, $\hat{y}_{p_{i-1}}$, for any $i \in \{2, \ldots, k\}$.
\[ \mathbb{E}[(\hat{y}_{p_i} - y)^2] \le \mathbb{E}[(\hat{y}_{p_{i-1}} - y)^2] \]
\end{lemma}

\begin{proof}
By definition, agent $A_{p_i}$ chooses parameters for its model $$f_{p_i}(x_{S_{p_i}}, \{\hat{y}_j\}_{j \in \text{Pa}(p_i)}) = w_{p_i}^T x_{S_{p_i}} + \sum_{j \in \text{Pa}(p_i)} v_{p_i,j} \hat{y}_j$$ to minimize the expected squared error $\mathbb{E}[(y - f_{p_i})^2]$. 

Since $A_{p_i}$'s predecessor in the path, $A_{p_{i-1}}$ is one of its parents, one model available to $A_{p_i}$ is a ``pass-through'' model that puts weight $1$ on $\hat y_{p_{i-1}}$ (i.e. sets $v_{p_i, p_{i-1}} = 1$ and sets  $v_{p_i, j} = 0$ for all other parents $j \in \text{Pa}(p_i) \setminus \{p_{i-1}\}$ and sets $w_{p_i} = \mathbf{0}$). This model necessarily has MSE identical to that of $A_{p_{i-1}}$. Since agent $i$ chooses their model to minimize MSE across all linear models available to them, they obtain MSE weakly lower than this. 
\end{proof}

Since MSE is monotonically decreasing along any path and is also bounded below by 0, this immediately implies that over sufficiently long paths, there must be long subsequences of agents whose MSE is very similar. But for our information aggregation argument, we would like their predictions themselves to be similar. This is what we establish in the next lemma: for any two adjacent agents in any path in a DAG, the expected squared difference in their predictions is equal to their difference in MSE.

\begin{lemma}[MSE Improvement and Predictor Closeness on a Path]
\label{lem:mse_improvement_closeness_dag}
For any path $P = (A_{p_1}, \ldots, A_{p_k})$, the squared L2 distance between the predictions of consecutive agents on the path, $A_{p_i}$ and $A_{p_{i-1}}$, is equal to the improvement in Mean Squared Error (MSE) between them.
\[ \mathbb{E}[(\hat{y}_{p_i} - \hat{y}_{p_{i-1}})^2] = \mathbb{E}[(\hat{y}_{p_{i-1}} - y)^2] - \mathbb{E}[(\hat{y}_{p_i} - y)^2]. \]
\end{lemma}

\begin{proof}
Each agent is solving a squared error linear regression problem, and so the predictor of agent $A_{p_i}$ is self-orthogonal by Corollary \ref{cor:orthogonal}. By Lemma \ref{lem:per_step_multiaccuracy_dag}, agent $A_{p_i}$'s predictor is also multiaccurate with respect to the predictor of $A_{p_{i-1}}$, since this is one of its parents. Hence the result follows by an application of Lemma \ref{lem:stability}. 
\end{proof}

We are now ready to put the pieces together to formulate a statement of our main upper bound: Along any sufficiently long path in a DAG, the learners must aggregate information available to the union of the agents along some ``long'' subsequence of that path. The analysis will proceed as follows: As established in Lemma \ref{lem:mse_non_increasing_dag}, the MSE of agent predictors is monotonically decreasing along any path in the DAG. Since MSE is bounded from below, this necessarily implies the existence of long contiguous subsequences of the path on which MSE changes very little. By Lemma \ref{lem:mse_improvement_closeness_dag}, we know that along these subsequences, the actual predictors of the agents themselves change very little! Each of the agents' predictors is (perfectly) self-orthogonal and multi-accurate with respect to linear functions of the features that they themselves observe. But because each agents' predictor is close to the predictor of every other agent along this ``stable'' subsequence, each of these agents predictors must also be \emph{approximately} multiaccurate with respect to linear functions of the features observed by the other agents along the stable subsequence. Hence by Lemma \ref{lem:multiaccurate_error_bound}, they must have MSE that is competitive with the best linear function defined on the union of all of these features --- i.e they must have successfully aggregated information. Formalizing this intuition requires keeping track of the error propagation along these stable subsequences, which we do below. Because the multi-accuracy error will be non-zero for features held by \emph{other} agents in the stable subsequence, we will need to assume that two quantities are bounded: the norm of the comparison predictor and second moments of the features. These are both mild assumptions.

\begin{theorem}[Competing with Bounded-Norm Predictors on a Path]
\label{thm:small_improvement_path}
Consider any path $P = (A_{p_1}, \ldots, A_{p_k})$ in the DAG, and a subsequence from index $i$ to $j$ (where $1 \le i \le j \le k$).
Let $S_{p_{i..j}} = \bigcup_{m=i}^{j} S_{p_m}$ be the union of features available to agents on this subsequence.
Let $g(x) = \sum_{l \in S_{p_{i..j}}} \alpha_l x_l$ be any linear predictor using these features.
Suppose the total MSE improvement across this subsequence is $\epsilon_{path}$, defined as:
\[ \MSE{\hat{y}_{p_{i-1}}} - \MSE{\hat{y}_{p_j}} = \epsilon_{path}, \quad \text{where } \epsilon_{path} \ge 0. \]
Let the following quantities be bounded:
\begin{itemize}
    \item The L1 norm of the predictor's coefficients: $\sum_{l \in S_{p_{i..j}}} |\alpha_l| \le A_g$.
    \item The second moments of the features: $\mathbb{E}[x_l^2] \le M_X^2$ for all $l \in S_{p_{i..j}}$.
\end{itemize}
Then, the MSE of the final predictor on the subsequence, $\hat{y}_{p_j}$, is bounded with respect to $\MSE{g}$:
\[ \MSE{\hat{y}_{p_j}} \le \MSE{g} + C \sqrt{\epsilon_{path}}, \]
where $N_{path} = j-i+1$ is the number of agents on the subsequence, and $C$ is defined as:
\[ C = 2A_g M_X \sqrt{N_{path}}. \]
\end{theorem}

\begin{proof}
    By Lemma~\ref{lem:mse_decomposition}, the MSE of $\hat{y}_{p_j}$ can be bounded relative to any other predictor $g$:
    \[ \MSE{\hat{y}_{p_j}} \le \MSE{g} + 2|\mathbb{E}[g(x)(\hat{y}_{p_j} - y)]| + 2|\mathbb{E}[\hat{y}_{p_j}(\hat{y}_{p_j} - y)]|. \]
    Let $\epsilon_{g} = \mathbb{E}[g(x)(\hat{y}_{p_j} - y)]$ and $\delta_{p_j} = \mathbb{E}[\hat{y}_{p_j}(\hat{y}_{p_j} - y)]$. By Lemma \ref{lem:per_step_multiaccuracy_dag} we have that $\delta_{p_j} = 0$. Hence it remains to bound $\epsilon_{g}$.

    We expand $\epsilon_{g}$ using the definition $g(x) = \sum_{l \in S_{p_{i..j}}} \alpha_l x_l$:
    \[ |\epsilon_{g}| \le \sum_{l \in S_{p_{i..j}}} |\alpha_l| |\mathbb{E}[x_l (\hat{y}_{p_j} - y)]|. \]
    For any feature $x_l$ with $l \in S_{p_m}$ for some $m \in \{i, \ldots, j\}$, Lemma~\ref{lem:per_step_multiaccuracy_dag} gives $\mathbb{E}[x_l (\hat{y}_{p_m} - y)] = 0$. Thus, we can write:
    \[ \mathbb{E}[x_l (\hat{y}_{p_j} - y)] = \mathbb{E}[x_l (\hat{y}_{p_j} - \hat{y}_{p_m})]. \]
    By Cauchy-Schwarz, $|\mathbb{E}[x_l (\hat{y}_{p_j} - \hat{y}_{p_m})]| \le \sqrt{\mathbb{E}[x_l^2]} \cdot \sqrt{\mathbb{E}[(\hat{y}_{p_j} - \hat{y}_{p_m})^2]}$.
    To bound $\sqrt{\mathbb{E}[(\hat{y}_{p_j} - \hat{y}_{p_m})^2]}$, we use a telescoping sum and the triangle inequality, finally applying  Lemma~\ref{lem:mse_improvement_closeness_dag}:
    \begin{align*}
    \left\| \hat{y}_{p_j} - \hat{y}_{p_m} \right\|_{L_2} &= \left\| \sum_{t=m+1}^{j} (\hat{y}_{p_t} - \hat{y}_{p_{t-1}}) \right\|_{L_2} \le \sum_{t=m+1}^{j} \left\| \hat{y}_{p_t} - \hat{y}_{p_{t-1}} \right\|_{L_2} \\
    &= \sum_{t=m+1}^{j} \sqrt{\Delta_{p_t}} 
    \end{align*}
where $\Delta_{p_t} = \MSE{\hat y_{p_{t}}}-\MSE{\hat y_{p_{t-1}}}$. 
    Applying Cauchy-Schwarz to the sum of square roots:
    \[ \left(\sum_{t=m+1}^{j} \sqrt{\Delta_{p_t}}\right)^2 \le (j-m) \sum_{t=m+1}^{j} \Delta_{p_t} \le N_{path} \epsilon_{path}. \]
    So, $\sqrt{\mathbb{E}[(\hat{y}_{p_j} - \hat{y}_{p_m})^2]} \le \sqrt{N_{path} \epsilon_{path}}$.
    Using the bounds $\sqrt{\mathbb{E}[x_l^2]} \le M_X$ and $\sum |\alpha_l| \le A_g$:
    \[ |\epsilon_{g}| \le \sum_{l \in S_{p_{i..j}}} |\alpha_l| \cdot (M_X \sqrt{N_{path} \epsilon_{path}}) = A_g M_X \sqrt{N_{path}} \sqrt{\epsilon_{path}}. \]

    This completes the proof.
    \end{proof}

Theorem \ref{thm:small_improvement_path} tells us that whenever we have a subsequence in which the cumulative MSE decrease is ``small'', then agents in the path have predictors that are competitive with the best (bounded norm) predictor that could have been defined on the union of the features available to the agents along the path, where ``competitive'' depends on the total MSE increase as well as the \emph{square root} of the length of the subsequence. How should we interpret this guarantee? The following corollary gives one way of viewing it: suppose features are distributed in the graph so that for every contiguous subsequence of length $M$, each of the $d$ features is available to \emph{some} agent in the subsequence. In this case Theorem \ref{thm:small_improvement_path} implies that if the DAG has a path of length $D$, then the agent at the end of this path must have excess error (compared to the optimal bounded norm predictor on \emph{all} features) bounded by $O\left(\frac{M}{\sqrt{D}} \right)$. 

\begin{corollary}[Overall Guarantee in Terms of Depth and Coverage]
\label{cor:best_linear_error}
Let $g^*(x) = (\beta^*)^T x$ be the optimal linear predictor over all $d$ features with bounded L1 norm, $\|\beta^*\|_1 \le A_{g^*}$.

Let $G$ be any DAG containing a path of length $D$. Consider any assignment of features to agents such that every contiguous interval of $M$ agents along this path cover all $d$ features. Then the final agent in the length-$D$ path learns a model that is competitive with $g^*$, up to error $$\eta = 2A_{g^*} M_X \sqrt{M} \sqrt{\frac{2M \cdot \MSE{\hat{y}_0}}{D}} = O\left(\frac{M}{\sqrt{D}} \right).$$
\end{corollary}

\begin{proof}
The proof uses a pigeonhole argument on MSE improvement along a path with the guarantee of feature coverage on its subsequences.

Consider any path $P$ of length $D$. We partition this path into $K = \lfloor D/M \rfloor$ disjoint subsequences, each of length $M$. Let $\epsilon_k$ be the total MSE improvement across subsequence $P_k$.

The sum of improvements along the path is bounded by the initial MSE: $\sum_{k=1}^{K} \epsilon_k \le \MSE{\hat{y}_0}$. By the pigeonhole principle, there must exist at least one subsequence, say $P_{k^*}$, with an improvement $\epsilon_{k^*} \le \MSE{\hat{y}_0}/K = \MSE{\hat{y}_0}/\lfloor D/M \rfloor$. Let this subsequence end at agent $A_{j^*}$.

We apply Theorem~\ref{thm:small_improvement_path} to this subsequence $P_{k^*}$. The theorem bounds the MSE of its final predictor, $\hat{y}_{j^*}$, relative to any linear predictor $g$ on that subsequence's features. In particular, this includes the global optimal linear predictor $g^*$, so the error is bounded by $C \sqrt{\epsilon_{path}}$, where $C = 2A_{g^*} M_X \sqrt{M}$ and $\epsilon_{path} = \epsilon_{k^*}$. That is,
\begin{align*}
    \MSE{\hat{y}_{j^*}} &\le \MSE{g^*} + 2A_{g^*} M_X \sqrt{M} \sqrt{\frac{\MSE{\hat{y}_0}}{\lfloor D/M \rfloor}} \\
    &\le \MSE{g^*} + 2A_{g^*} M_X \sqrt{M} \sqrt{\frac{2M \cdot \MSE{\hat{y}_0}}{D}} \\
    &\le \MSE{g^*} + \eta.
\end{align*}

Finally, by Lemma~\ref{lem:mse_non_increasing_dag} (Monotonically Non-Increasing Error), this guarantee extends to all subsequent agents on the same path $P$.
\end{proof}





\subsection{Performance Guarantee under a Random Allocation Model}
\label{sec:random-linear}
So far we have proven guarantees for arbitrary DAGs in terms of their depth (longest path length), but the guarantees were necessarily abstract because we made no assumption about how variables were distributed amongst agents in the DAG. 
To make the guarantee of Theorem \ref{thm:small_improvement_path} and Corollary \ref{cor:best_linear_error} more concrete, we now analyze error under a random feature allocation model, which allows us to give theorems on the error of agent predictors as a function of the depth of the graph and parameters of the random feature allocation model. The random allocation model is also used in the experimental results of Section~\ref{sec:experiments}.

\begin{definition}[Random Feature Allocation Model]
Let the total set of features be $[d] = \{1, 2, \ldots, d\}$. We now assume that for each agent $A_i$ in the DAG, its feature set $S_i$ is generated randomly. Specifically, for each feature $k \in [d]$, we have $\mathbb{P}[k \in S_i] = p$ independently for all $i$ and $k$.
\end{definition}

The following proposition establishes that for any path in the DAG, all of its sufficiently long subsequences are likely to have full feature coverage.

\begin{proposition}[Full Feature Coverage on Subsequences]
\label{prop:feature_coverage}
Consider the random feature allocation model. Let $P$ be any path in the DAG of length $N_{path}$. For any desired failure probability $\delta \in (0, 1)$, if the subsequence length $M$ satisfies
\[ M \ge \frac{\ln(N_{path} d / M) + \ln(1/\delta)}{-\ln(1-p)}, \]
then with probability at least $1-\delta$, every consecutive subsequence of $P$ of length $M$ observes the full set of features.
\end{proposition}

\begin{proof}
First, we calculate the probability that a single feature is missed by a single subsequence. Then, we bound the probability that any feature is missed by any subsequence along the main path $P$ by a union bound.

Let's fix a single subsequence of length $M$ along $P$. Let's also fix a single feature $k \in [d]$. The probability that a single agent does \emph{not} observe feature $k$ is $1-p$. Since feature allocations are independent across agents, the probability that \emph{none} of the $M$ agents in the subsequence observe feature $k$ is $(1-p)^M$.

By a union bound over all $d$ features, the probability that this single subsequence misses at least one feature is at most $d(1-p)^M$.

Now, we take a union bound over all consecutive subsequences of length $L$ along the main path $P$. There are at most $N_{path}/M$ such disjoint subsequences. Let $E$ be the event that at least one of these subsequences fails to cover all features.
\[ \mathbb{P}[E] \le \sum_{i=1}^{\lfloor N_{path}/M \rfloor} \mathbb{P}[\text{subsequence } i \text{ fails}] \le \frac{N_{path}}{M} d(1-p)^M. \]

We want this failure probability to be at most $\delta$:
\[ \frac{N_{path} d}{L}(1-p)^M \le \delta. \]
Taking the natural logarithm and solving for $M$ (noting that $\ln(1-p)$ is negative) gives the desired bound:
\[ \ln(N_{path}d/M) + M\ln(1-p) \le \ln(\delta) \]
\[ M \ge \frac{\ln(N_{path}d/(M\delta))}{-\ln(1-p)}. \]
If $M$ satisfies this condition, the probability that all subsequences of length $M$ along path $P$ have full feature coverage is at least $1-\delta$.
\end{proof}

Finally we can combine the high probability ``Full Coverage'' bound that we just proved in Proposition \ref{prop:feature_coverage} with Corollary \ref{cor:best_linear_error} to obtain a concrete theorem in the random allocation model.

\begin{theorem}[Overall Guarantee in Random Model for DAGs]
\label{thm:overall_guarantee_random_model}
Consider the random feature allocation model with feature probability $p$. Let $g^*(x) = (\beta^*)^T x$ be the optimal linear predictor over all $d$ features with bounded L1 norm, $\|\beta^*\|_1 \le A_{g^*}$.
For any target error $\eta > 0$ and confidence level $\delta_L \in (0,1)$, there exists a required path length $N_0$ such that for any path $P$ in the DAG with length at least $N_0$, the following holds with probability at least $1-\delta_L$:
There exists an agent $A_{j^*}$ on the path $P$ (with $j^*$ being an index along the path, $j^* \le N_0$) such that for all subsequent agents $A_k$ on that same path,
\[ \MSE{\hat{y}_k} \le \MSE{g^*} + \eta. \]
The required path length $N_0$ is given by:
\[ N_0 \geq \left\lceil \frac{8M^2 \cdot \MSE{\hat{y}_0} (A_{g^*} M_X)^2}{\eta^2} \right\rceil, \quad \text{where} \quad M = \left\lceil \frac{\ln(N_0 d/M) + \ln(1/\delta_L)}{-\ln(1-p)} \right\rceil. \]
Solving this gives:
$$N_0 \geq O \left( \frac{(\log d + \log(1/p\eta))^2}{p^2 \eta^2}\right)$$
\end{theorem}

\begin{proof}
This follows from Corollary \ref{cor:best_linear_error} together with our high probability bound from Proposition \ref{prop:feature_coverage} establishing the necessary ``full feature coverage'' length.
\end{proof}

The theorem states that if the DAG is ``deep'' enough (i.e., contains a path of length at least $N_0$), then the agents at the end of that path will be nearly optimal.

\section{Extension to General Function Classes}
\label{sec:general_classes}

In this section, we extend our analysis beyond linear predictors. We show that the core results---monotonic error reduction and the ability to compete with a strong benchmark---can be preserved under this more flexible model. This allows the framework to be applied to a wider range of learning algorithms, such as decision trees or neural networks, operating on local feature subsets. Much of the analysis proceeds identically to Section \ref{sec:linear}, and where this is the case, we provide explicit references back to the linear analogues and defer full proofs to Appendix \ref{sec:nonlinear_proofs}. As in Section~\ref{sec:linear}, we work in the distributional regime; see Appendix~\ref{sec:nonlinear_generalization} for finite-sample guarantees.

\subsection{Model Definition with General Predictors}

We modify the learning process to allow for general function classes. Each agent $A_k$ in the DAG is associated with a local hypothesis class $\mathcal{H}_k$, which contains functions that map the agent's local features $x_{S_k}$ to a real value.
\begin{definition}[Hypothesis Class for Agent $k$]
    The \textbf{Hypothesis Class $\mathcal{H}_k$ for Agent $k$} is a set of functions $h: \mathbb{R}^{|S_k|} \to \mathbb{R}$. For example, $\mathcal{H}_k$ could be the set of all decision trees of a certain depth on the features in $S_k$, or a class of neural networks.
\end{definition}

The prediction for any agent $A_k$ is constructed based on the information available to it: its local functions from $\mathcal{H}_k$ and the predictions $\{\hat{y}_p\}$ from its parent agents $p \in \text{Pa}(k)$. The core idea is to build a \textbf{pool of features} $\mathcal{F}_k$ for each agent, composed of functions from its own class $\mathcal{H}_k$ and the predictors from its parents. The agent's final predictor, $\hat{y}_k$, is then defined as the orthogonal projection of the outcome $Y$ onto the linear span of this feature pool, $\text{span}(\mathcal{F}_k)$. The pool of features is built up in a greedy manner to guarantee that the solution that each agent $k$ finds satisfies approximate multiaccuracy with respect to $\mathcal{H}_k$.

\begin{remark}[Implementation via Least Squares]
Algorithmically, finding the orthogonal projection of $Y$ onto the span of a set of features $\mathcal{F}_k$ is equivalent to solving a standard linear least squares regression problem. The goal is to find coefficients $\beta$ that minimize the squared error $\|Y - \sum_{h \in \mathcal{F}_k} \beta_h h\|_2^2$. The resulting predictor, $\hat{y}_k = \sum_{h \in \mathcal{F}_k} \beta_h h$, is the orthogonal projection.
\end{remark}

The fact that each agent continues to solve a least squares optimization problem over some class of features guarantees that each agent's predictor satisfies self-orthogonality, and lets us apply the same analytic structure as in Section \ref{sec:linear}.

\subsection{The Greedy Orthogonal Regression Algorithm}
\label{sec:greedy_orthogonal_regression}

We here define the algorithm that each agent will use, which we call the ``Greedy Orthogonal Regression Algorithm''. Each agent $A_k$ will greedily build up a feature set $\mathcal{F}_k$, initially consisting of only the predictions of their parents. At each iteration, the tentative hypothesis maintained by $A_k$ will be the least-squares regression predictor using $\mathcal{F}_k$ as features. We assume that each agent has the ability to optimize over their own hypothesis class $\mathcal{H}_k$. In particular, at each iteration, they find the $h \in \mathcal{H}_k$ that is maximally correlated with the \emph{residuals} of their tentative hypothesis (this can be implemented using a squared error regression algorithm). If this correlation is smaller than a given threshold $\Delta$, they use their current hypothesis; otherwise they add $h$ to their feature representation $\mathcal{F}_k$ and resolve for a new hypothesis, by solving least squares regression over this larger feature space. Our algorithm resembles existing algorithms for finding multi-accurate predictors \cite{multicalibration,multiaccuracy}, but distinguishes itself by resolving for the optimal linear regression function over the span of its feature representation $\mathcal{F}_k$ at every round (rather than simply accumulating additive updates), which will be important for our application. 

As we show, this algorithm has several important properties. First, it is guaranteed to quickly find a hypothesis that is $\Delta$-approximately multi-accurate with respect to $\mathcal{H}_k$, in a number of rounds (and therefore over a number of features) that scales as $O(1/\Delta^2)$. Moreover, because our predictor is always the optimal linear predictor over its feature set $\mathcal{F}_k$, it is always exactly self-orthogonal and exactly multi-accurate with respect to the predictors of its parents. These properties will drive our analysis. 

\begin{definition}[Greedy Orthogonal Regression Algorithm]
\label{prop:greedy_orthogonal_regression}
The predictor $\hat{y}_k$ for agent $k$ is constructed by the following iterative procedure:
\begin{enumerate}
    \item \textbf{Initialization:} Start with the feature pool containing only parent predictors, $\mathcal{F}_k = \{\hat{y}_p\}_{p \in \text{Pa}(k)}$. The initial predictor is the projection onto this pool: $\hat{y}_k = \text{proj}_{\text{span}(\mathcal{F}_k)}(Y)$.
    
    \item \textbf{Greedy Selection Loop:} The algorithm repeatedly enhances the feature pool $\mathcal{F}_k$.
    \begin{enumerate}
        \item Let the current residual be $R = Y - \hat{y}_k$.
        \item Find the function $h_{\text{next}} \in \mathcal{H}_k$ that is most correlated with the current residual:
        \[ h_{\text{next}} = \arg\max_{h \in \mathcal{H}_k} |\mathbb{E}[h(x_{S_k}) \cdot R]| \]
        \item \textbf{Termination Condition:} If $|\mathbb{E}[h_{\text{next}} \cdot R]| < \Delta$, the loop terminates and the current predictor $\hat{y}_k$ is final.
        \item \textbf{Update:} Otherwise, add the selected function to the pool, $\mathcal{F}_k \leftarrow \mathcal{F}_k \cup \{h_{\text{next}}\}$, and update the predictor by re-projecting onto the expanded pool: $\hat{y}_k \leftarrow \text{proj}_{\text{span}(\mathcal{F}_k)}(Y)$. Then, repeat the loop.
    \end{enumerate}

\end{enumerate}
By construction, the final residual $Y - \hat{y}_k$ is orthogonal to every function in the final feature pool $\mathcal{F}_k$. This orthogonality is key to the analysis that follows.
\end{definition}

First, we establish that the algorithm terminates quickly --- and in particular does not build up a very large feature set $\mathcal{F}_k$.
\begin{proposition}[Convergence Bound]
\label{prop:convergence_bound}
Assume all functions $h \in \mathcal{H}_k$ are normalized such that $\mathbb{E}[h(x_{S_k})^2] \le 1$. The number of iterations, $T_k$, performed by agent $k$'s instantiation of the Greedy Orthogonal Regression Algorithm before termination is bounded by:
\[ T_k \le \frac{\mathbb{E}[(Y - \hat{y}_{k,0})^2]}{\Delta^2}, \]
where $\hat{y}_{k,0}$ is the initial predictor based only on parent predictions.
\end{proposition}
\begin{proof}
Let $\hat{y}_{s-1}$ be the predictor at the start of an iteration and $\hat{y}_s$ be the predictor after adding a new function $h_s$. The new predictor $\hat{y}_s$ is the projection of $Y$ onto $\text{span}(\mathcal{F}_{s-1} \cup \{h_s\})$. By the Pythagorean theorem for projections, the reduction in MSE is:
\[ \MSE{\hat{y}_{s-1}} - \MSE{\hat{y}_s} = \mathbb{E}[(\hat{y}_s - \hat{y}_{s-1})^2]. \]
The update term, $\hat{y}_s - \hat{y}_{s-1}$, is the projection of the residual $R_{s-1} = Y - \hat{y}_{s-1}$ onto the component of $h_s$ that is orthogonal to the previous subspace, $\text{span}(\mathcal{F}_{s-1})$. Let this component be $h_s^{\perp}$. The squared magnitude of this projection is $(\mathbb{E}[R_{s-1} \cdot h_s^{\perp}])^2 / \mathbb{E}[(h_s^{\perp})^2]$.
Since $R_{s-1}$ is orthogonal to $\text{span}(\mathcal{F}_{s-1})$, we have $\mathbb{E}[R_{s-1} \cdot h_s] = \mathbb{E}[R_{s-1} \cdot h_s^{\perp}]$. As long as the algorithm has not terminated, $|\mathbb{E}[R_{s-1} \cdot h_s]| \ge \Delta$. With the normalization assumption $\mathbb{E}[(h_s^{\perp})^2] \le \mathbb{E}[h_s^2] \le 1$, the MSE reduction at each step is at least $\Delta^2$.
\[ \MSE{\hat{y}_{s-1}} - \MSE{\hat{y}_s} \ge \Delta^2. \]
The total MSE reduction after $T_k$ iterations is at least $T_k \Delta^2$. Since the total reduction cannot exceed the initial MSE, $\mathbb{E}[(Y - \hat{y}_{k,0})^2]$, we have $T_k \Delta^2 \le \mathbb{E}[(Y - \hat{y}_{k,0})^2]$, which gives the desired bound.
\end{proof}

Morever, the halting condition guarantees that at termination, the final predictor is approximately multiaccurate with respect to $\mathcal{H}_k$ (as $h_{next}$ is the maximally-correlated function with the final residual $Y-\hat{y}_k$, the inequality must hold for all functions $h\in \mathcal{H}_k$):

\begin{restatable}[$\Delta$-Multiaccuracy]{lemma}{deltamultiacc}\label{lem:achieved_multiaccuracy}
    The final predictor $\hat{y}_k$ produced by the Greedy Orthogonal Regression Algorithm (Definition~\ref{prop:greedy_orthogonal_regression}) is $\Delta$-multiaccurate with respect to the function class $\mathcal{H}_k$. That is, for all functions $h \in \mathcal{H}_k$,
\[ |\mathbb{E}[h(x_{S_k}) \cdot (Y - \hat{y}_k)]| \le \Delta. \]
\end{restatable}


We now give a decomposition generalizing Lemma \ref{lem:mse_decomposition} that relates the MSE of agent $k$'s prediction to the MSE of an arbitrary benchmark function, using the fact that the residual $(y-\hat{y}_k)$ is orthogonal to the subspace $\text{span}(\mathcal{F}_k)$.

\begin{restatable}[General MSE Decomposition]{lemma}{genmsedecomp} \label{lem:gen_mse_decomposition}
    Let $\hat{y}_k$ be the predictor for agent $k$, which is the orthogonal projection of $Y$ onto a subspace $\text{span}(\mathcal{F}_k)$. Let $g$ be any benchmark predictor, and decompose it as $g = g_{\text{pool}} + g_{\text{rest}}$, where $g_{\text{pool}} = \text{proj}_{\text{span}(\mathcal{F}_k)}(g)$ is the projection of $g$ onto the feature subspace, and $g_{\text{rest}}$ is the orthogonal component. Then the MSE of $\hat{y}_k$ is related to the MSE of $g$ by:
\[ \MSE{\hat{y}_k} = \MSE{g} - \mathbb{E}[(\hat{y}_k - g)^2] + 2\mathbb{E}[(Y - \hat{y}_k)g_{\text{rest}}]. \]
\end{restatable}

We can again relate the expected squared difference in agent $k$'s predictions with respect to any of its parents by the decrease in MSE that agent $k$ realizes on top of that of its parent, as in Lemma \ref{lem:mse_improvement_closeness_dag}. This result also directly shows that the MSE is guaranteed to be non-increasing, as in Lemma \ref{lem:mse_non_increasing_dag}.  

\begin{restatable}[MSE Improvement in Greedy Orthogonal Regression]{lemma}{genmseimprovement} \label{thm:gen_small_improvement_path_approx}
    Let $A_k$ be an agent in the DAG, and let $\hat{y}_k$ be its predictor. Let $p \in \text{Pa}(k)$ be any parent of $k$. The MSE of $\hat{y}_k$ is exactly related to the MSE of its parent's predictor $\hat{y}_p$ by:
\[ \MSE{\hat{y}_k} = \MSE{\hat{y}_p} - \mathbb{E}[(\hat{y}_k - \hat{y}_p)^2]. \]
\end{restatable}


\subsection{Overall Guarantee for Non-Linear Models}

In this section, we establish overall guarantees for the accuracy of agents in our distributed learning process, analogous to the ones we proved in Section \ref{sec:random-linear} for the case of linear functions. Just as in Sections \ref{sec:linear-path} and \ref{sec:random-linear}, the core of the argument is to show that within any DAG, a sufficiently long path of agents, by combining their diverse hypothesis classes, can obtain diminishing excess error with respect to some benchmark class. What differs is what the benchmark class is. We now assume a distribution over hypothesis classes $\mathcal{H}_k$, akin to the random feature model we assumed in Section \ref{sec:random-linear}.  For example, each $\mathcal{H}_k$ might be defined as the set of predictors in some common model class (decision trees, neural networks) defined over different subsets of the variables, and $\mathcal{D}_H$ might be implicitly defined by a distribution over subsets of variables as in the linear regression section --- but this is not necessary for our analysis. Our benchmark will now be parameterized by $M$, and will correspond to the expected MSE for the best predictor in the span of $M$ randomly sampled hypothesis classes from this hypothesis class distribution. This benchmark consists of functions that are additively separable across agents, but may exhibit arbitrary non-linear relationships (parameterized by functions in $\mathcal{H}_k$) between any subset of variables that are all observed by a single agent $k$. Note that this is a much stronger benchmark than simply comparing to linear combinations of the models \emph{actually deployed} by a set of agents --- it allows arbitrary linear combinations of models \emph{available} to sets of agents, even those that they choose not to use. 

\subsubsection{A Norm-Constrained Benchmark}

To formalize this, we first define the performance of an arbitrary collection of hypothesis classes subject to norm constraints, and then establish a benchmark based on the expected performance of a randomly drawn committee.

\begin{definition}[Norm-Constrained MSE of a Collection of Hypothesis Classes]
For any collection of hypothesis classes $\mathcal{S}_H$ and any norm bounds $L_{1,g} > 0$ and $B_g > 0$, we define the norm-constrained MSE as the minimum MSE achievable by a predictor in the span of $\mathcal{S}_H$ that respects these bounds:
\begin{align*}
    \text{MSE}(\mathcal{S}_H, L_{1,g}, B_g) = \min_{g} \quad & \mathbb{E}[(g(X) - Y)^2] \\
    \text{s.t.} \quad & g \in \text{span}(\bigcup_{\mathcal{H} \in \mathcal{S}_H} \mathcal{H}) \\
    & g(x) = \sum_i \beta_i h_i(x) \text{ with } \|\beta\|_1 \le L_{1,g} \text{ and } \mathbb{E}[h_i(x)^2] \le B_g^2.
\end{align*}
Let $g_{\text{opt}}(\mathcal{S}_H, L_{1,g}, B_g)$ denote a predictor achieving this minimum.
\end{definition}

\begin{definition}[M-Class Norm-Constrained Benchmark MSE]
For any integer $M > 0$ and norm bounds $L_{1,g}, B_g > 0$, the benchmark MSE is the expected minimum MSE achievable by a norm-constrained predictor from the span of $M$ i.i.d. hypothesis classes drawn from a distribution over hypothesis classes $\mathcal{D}_H$:
\[ \text{MSE}_{\text{bench}}(M, L_{1,g}, B_g) = \mathbb{E}_{\mathcal{S}_M \sim (\mathcal{D}_H)^M} \left[ \text{MSE}(\mathcal{S}_M, L_{1,g}, B_g) \right]. \]
\end{definition}

We will now assume that each agent $A_k$, embedded in an arbitrary DAG, has a hypothesis class $\mathcal{H}_k$ sampled independently from some some arbitrary distribution $\mathcal{D}_H$ over hypothesis classes, over which the benchmark is defined. The following proposition shows that the union of the hypothesis classes along a sufficiently long path of agents is expressive enough, with high probability, to achieve performance close to this benchmark. Note that this does not yet imply that any agent in this path actually makes predictions that are competitive with our benchmark --- just that the optimal predictor in the span of the union of their hypothesis classes would. 

We assume there is a universal upper bound $C_{\text{mse}}$ on the MSE of the best predictor in each hypothesis class (e.g., $C_{\text{mse}} = \text{Var}(Y)$ if constant functions are always in the span of any hypothesis class). The argument is similar to the full feature coverage argument for subsequences in Proposition \ref{prop:feature_coverage}, so here we provide a proof sketch and defer full calculations to Appendix \ref{sec:nonlinear_proofs}.

\begin{restatable}{proposition}{pathbenchmark} \label{prop:path_achieves_benchmark}
    For any benchmark size $M > 0$, norm bounds $L_{1,g}, B_g > 0$, approximation error $\gamma > 0$, and probability $\delta > 0$, if we consider a path $P$ of length $|P| = L \ge \frac{M \ln(1/\delta)}{\ln(1 + \gamma / C_{\text{mse}})}$ whose agents are assigned hypothesis classes i.i.d. from $\mathcal{D}_H$, then with probability at least $1-\delta$:
\[ \text{MSE}(\mathcal{H}_P, L_{1,g}, B_g) \le \text{MSE}_{\text{bench}}(M, L_{1,g}, B_g) + \gamma, \]
where $\mathcal{H}_P = \{\mathcal{H}_k\}_{k \in P}$ and $C_{\text{mse}}$ is a universal upper bound on the MSE of any single hypothesis class.
\end{restatable}

\begin{proof}[Proof sketch]
The proof proceeds via a blocking argument. Consider a ``block'' of $M$ contiguous agents along any path with i.i.d. hypothesis classes $\mathcal{S}_M$. Let $X = \text{MSE}(\mathcal{S}_M, L_{1,g}, B_g)$ be the random variable for the block's MSE, with expectation $\mathbb{E}[X] = \text{MSE}_{\text{bench}}(M, L_{1,g}, B_g)$.
A block is ``bad'' if $X > \mathbb{E}[X] + \gamma$, which occurs with probability 
\[ p_{\text{bad}} = \mathbb{P}(X \ge \mathbb{E}[X] + \gamma) \le \frac{\mathbb{E}[X]}{\mathbb{E}[X] + \gamma} = \frac{\text{MSE}_{\text{bench}}(M)}{\text{MSE}_{\text{bench}}(M) + \gamma} \le \frac{C_{\text{mse}}}{C_{\text{mse}} + \gamma}. \]

We partition a path of length $L$ into $k = \lfloor L/M \rfloor$ disjoint blocks. The proposition fails only if all $k$ blocks are ``bad'', which occurs with probability $(p_{\text{bad}})^k \le \left( \frac{C_{\text{mse}}}{C_{\text{mse}} + \gamma} \right)^k. $
We want this failure probability to be at most $\delta$. Setting
$\left( \frac{C_{\text{mse}}}{C_{\text{mse}} + \gamma} \right)^k \le \delta$ and solving for $k$ gives:
\[ k \ge \frac{\ln(1/\delta)}{\ln\left(1 + \frac{\gamma}{C_{\text{mse}}}\right)} \implies L \ge \frac{M \ln(1/\delta)}{\ln\left(1 + \frac{\gamma}{C_{\text{mse}}}\right)}.  \]
\end{proof}

\subsubsection{Final Guarantee}

Combining these pieces, we arrive at the main theorem, bounding the error of an agent in an arbitrary DAG containing a long path. The final error decomposes into three components: the benchmark's intrinsic error, a propagation error from information loss, and an approximation error. The approximation error itself has two sources: the statistical error from sampling a path of classes, and the algorithmic error from the greedy selection at each agent. We can control these by setting the termination threshold $\Delta$ of the local greedy algorithm to be proportional to the desired statistical precision $\gamma$.

\begin{theorem}[Overall Guarantee for Greedy Orthogonal Regression]
\label{thm:gen_overall_guarantee_strong_benchmark}
Fix an arbitrary DAG in which agent's hypothesis classes are sampled independently from $\mathcal{D}_H$. Fix any desired benchmark size $M > 0$, norm bounds $L_{1,g}, B_g > 0$, approximation error terms $\gamma, \eta > 0$, and probability $\delta > 0$. Let the required path length be $L = M \cdot \max\left(1, \frac{\ln(1/\delta)}{\ln(1 + \gamma / C_{\text{mse}})}\right)$. Let each agent run the greedy orthogonal regression algorithm with termination threshold $\Delta = \gamma$.

If the DAG contains a path of length at least $N_p \ge \frac{L \cdot C_{\text{mse}}}{\eta}$, then with probability at least $1-\delta$, there exists a subsequence $P^*$ of length $L$ ending at agent $p_m$ such that its predictor $\hat{y}_{p_m}$ satisfies:
\[ \text{MSE}(\hat{y}_{p_m}) \le \text{MSE}_{\text{bench}}(M, L_{1,g}, B_g) + \underbrace{\gamma(1 + 2L_{1,g} B_g)}_{\text{Approximation Error}} + \underbrace{2\sqrt{L\eta} L_{1,g} B_g}_{\text{Propagation Error}}. \]
\end{theorem}
\begin{proof}
    The proof combines the results from the preceding lemmas.

    Let $P_{long}$ be a path of length $N_p \ge \frac{L \cdot C_{\text{mse}}}{\eta}$. By the pigeonhole principle, we can find a subsequence $P^*$ of length $L$ with a total MSE improvement of at most $\eta$. Let $p_0$ be the first agent on this path and $p_m$ be the final agent. The total MSE reduction along this path is $\mathbb{E}[(\hat{y}_{p_0} - \hat{y}_{p_m})^2] \le \eta$.

    By Proposition~\ref{prop:path_achieves_benchmark}, with probability at least $1-\delta$, the hypothesis classes along path $P^*$ are such that the optimal benchmark predictor for this set, $g_{\text{bench}}$, satisfies $\text{MSE}(g_{\text{bench}}) \le \text{MSE}_{\text{bench}}(M, L_{1,g}, B_g) + \gamma$.
    Applying Lemma~\ref{lem:gen_mse_decomposition} to the final agent $p_m$ with benchmark $g = g_{\text{bench}}$ gives:
    \[ \text{MSE}(\hat{y}_{p_m}) \le \text{MSE}(g_{\text{bench}}) + 2|\mathbb{E}[(Y - \hat{y}_{p_m})g_{\text{rest}}]|, \]
    where $g_{\text{rest}} = g_{\text{bench}} - \text{proj}_{\text{span}(\mathcal{F}_{p_m})}(g_{\text{bench}})$.

    Since the residual $(Y - \hat{y}_{p_m})$ is orthogonal to the subspace $\text{span}(\mathcal{F}_{p_m})$, the cross-term simplifies to $\mathbb{E}[(Y - \hat{y}_{p_m})g_{\text{bench}}]$. The benchmark predictor $g_{\text{bench}}$ is a linear combination of functions from the classes in $\mathcal{H}_{P^*}$, i.e., $g_{\text{bench}} = \sum_i \beta_i h_i$ where each $h_i \in \mathcal{H}_k$ for some $k \in P^*$, with $\|\beta\|_1 \le L_{1,g}$ and $\mathbb{E}[h_i^2] \le B_g^2$.
    
    We bound the absolute value of the cross-term:
    \[ |\mathbb{E}[(Y - \hat{y}_{p_m})g_{\text{bench}}]| = \left|\sum_i \beta_i \mathbb{E}[(Y - \hat{y}_{p_m})h_i]\right| \le \sum_i |\beta_i| \cdot |\mathbb{E}[(Y - \hat{y}_{p_m})h_i]|. \]
    For each term $|\mathbb{E}[(Y - \hat{y}_{p_m})h_i]|$, where $h_i \in \mathcal{H}_k$ for some agent $k \in P^*$, we decompose the residual:
    \[ Y - \hat{y}_{p_m} = (Y - \hat{y}_k) + (\hat{y}_k - \hat{y}_{p_m}). \]
    This gives two components to bound:
    \begin{itemize}
        \item \textbf{Algorithmic Error at agent $k$:} By Lemma~\ref{lem:achieved_multiaccuracy}, the predictor $\hat{y}_k$ is $\Delta$-multiaccurate with respect to $\mathcal{H}_k$. Since $h_i \in \mathcal{H}_k$, we have $|\mathbb{E}[(Y - \hat{y}_k)h_i]| \le \Delta \cdot \sqrt{\mathbb{E}[h_i^2]} \le \Delta \cdot B_g$.
        \item \textbf{Propagation Error from $k$ to $m$:} We first bound the predictor difference term $\mathbb{E}[(\hat{y}_k - \hat{y}_{p_m})^2]$. Let the path from $k$ to $p_m$ be denoted by agents $j=k, \dots, p_m$. The difference can be written as a telescoping sum: $\hat{y}_{p_m} - \hat{y}_k = \sum_{j=k}^{p_m-1} (\hat{y}_{j+1} - \hat{y}_j)$. By the triangle inequality, $\sqrt{\mathbb{E}[(\hat{y}_{p_m} - \hat{y}_k)^2]} \le \sum_{j=k}^{p_m-1} \sqrt{\mathbb{E}[(\hat{y}_{j+1} - \hat{y}_j)^2]}$. By Lemma~\ref{thm:gen_small_improvement_path_approx}, for any adjacent pair $(j, j+1)$ on the path, the MSE reduction is given by $\mathbb{E}[(\hat{y}_{j+1} - \hat{y}_j)^2] = \text{MSE}(\hat{y}_j) - \text{MSE}(\hat{y}_{j+1})$. Let $\Delta_j = \text{MSE}(\hat{y}_j) - \text{MSE}(\hat{y}_{j+1})$. Our sum is $\sum_{j=k}^{p_m-1} \sqrt{\Delta_j}$. By Cauchy-Schwarz, $(\sum \sqrt{\Delta_j})^2 \le (\sum 1) (\sum \Delta_j) = (p_m-k) (\text{MSE}(\hat{y}_k) - \text{MSE}(\hat{y}_{p_m}))$. The path length $p_m-k$ is at most $L$, and the total MSE drop is at most $\eta$. Thus, $\sqrt{\mathbb{E}[(\hat{y}_{p_m} - \hat{y}_k)^2]} \le \sqrt{L\eta}$. Finally, applying Cauchy-Schwarz to the original term, $|\mathbb{E}[(\hat{y}_k - \hat{y}_{p_m})h_i]| \le \sqrt{\mathbb{E}[(\hat{y}_k - \hat{y}_{p_m})^2]} \cdot \sqrt{\mathbb{E}[h_i^2]} \le \sqrt{L\eta} \cdot B_g$.
    \end{itemize}
    Combining these, $|\mathbb{E}[(Y - \hat{y}_{p_m})h_i]| \le (\Delta + \sqrt{L\eta})B_g$.
    
    Substituting this back into the sum:
    \[ |\mathbb{E}[(Y - \hat{y}_{p_m})g_{\text{bench}}]| \le \sum_i |\beta_i| (\Delta + \sqrt{L\eta})B_g = (\|\beta\|_1) (\Delta + \sqrt{L\eta})B_g \le L_{1,g} B_g (\Delta + \sqrt{L\eta}). \]
    Plugging this into the MSE decomposition, and using the theorem's condition that we set the termination threshold $\Delta = \gamma$:
    \begin{align*}
    \text{MSE}(\hat{y}_{p_m}) &\le \text{MSE}(g_{\text{bench}}) + 2 L_{1,g} B_g (\gamma + \sqrt{L\eta}) \\
    &\le (\text{MSE}_{\text{bench}}(M, L_{1,g}, B_g) + \gamma) + 2\gamma L_{1,g} B_g + 2\sqrt{L\eta} L_{1,g} B_g \\
    &= \text{MSE}_{\text{bench}}(M, L_{1,g}, B_g) + \gamma(1 + 2L_{1,g} B_g) + 2\sqrt{L\eta} L_{1,g} B_g.
    \end{align*}
    This completes the proof.
    \end{proof}

\section{Lower Bounds: Depth is Necessary}
\label{sec:lower-bounds}
In this section, we prove two lower bounds, both based on the following distribution:

\begin{definition}[Lower Bound Distribution]
\label{def:lb}
Let $z_1, z_2, \ldots, z_k$ be independent standard Gaussian random variables with  $\mathbb{E}[z_i] = 0$ and $\mathbb{E}[z_i^2] = 1$. The target variable is $Y = z_k$.

The instance has $d = k$ features, defined as follows:
\begin{align*}
x_1 &= z_1 \\
x_i &= z_i - z_{i-1} \quad \text{for } i=2, \ldots, k.
\end{align*}
All features have mean zero. Their variances are $\mathbb{E}[x_1^2] = 1$ and $\mathbb{E}[x_i^2] = 2$ for $i \ge 2$.
\end{definition}

This construction has the properties that:
\begin{enumerate}
\item It is consistent with a perfect linear predictor on the whole feature space:
$$Y = \sum_{i=1}^k x_i$$
and so there is a linear predictor $f$ that achieves $\MSE{f}= 0$, but also
\item Variables $x_1,\ldots,x_{k-1}$ are all independent of $Y = z_k$, and moreover for every $j$, $x_1,\ldots,x_j$ are independent of $x_{j+2},\ldots,x_k$.
\end{enumerate}
This second condition intuitively creates a chain of dependencies.  In order for an Agent with access to just $x_i$ to add useful information, they must already have information (through the predictions of their parents) about $\{x_{i+1},\ldots,x_k\}$. This induces a sequential dependency that we rely on to prove our lower bounds.

We prove two lower bounds: First, in Section \ref{sec:lower_bound_path}, we construct a lower bound instance corresponding to a single long path, in which the variables from our lower bound instance appear in \emph{backwards} order over and over again. For efficient information propagation, we would want the first Agent to observe $x_k$, the second to observe $x_{k-1}$, etc --- but instead in our construction the first agent observes $x_1$, the second observes $x_2$, etc. The result is that information is propogated only ``one step'' every time the \emph{entire sequence} of variables is observed, and the result is a lower bound showing that in order to obtain excess error $\eta$ compared to the best linear predictor, we might need a long enough path so that we see \emph{every variable} $1/\eta$ many times. 

Then, in Section \ref{sec:lower_bound_general} we use the same lower bound distribution to prove a very general lower bound for DAGs of arbitrary topology. We show that for every DAG of depth $D$, there is a distribution that admits a perfect linear predictor, such that even for the \emph{best case} assignment of single variables to Agents in the DAG, no agent is able to make predictions with MSE better than $\frac{1}{D+1}$. The distribution is once again simply the distribution from Definition \ref{def:lb}, where we choose $k \geq D+1$. This shows  that our upper bounds, which give excess error guarantees that diminish with the depth of the graph, have no analogue in which depth is replaced by any other property of the graph (that can be scaled independently of depth) --- depth is the fundamental parameter that controls worst case excess error. 

\subsection{Learning Requires Seeing Each Variable Many Times in Sequence}
\label{sec:lower_bound_path}

Our first lower bound constructs a specific graph (a simple path) and assignment of variables to the Agents in order to show that for accurate learning it is necessary not just to observe each variable in sequence, but to observe each variable in sequence many times.

 Let $k \ge 2$ be an integer. We construct a lower bound instance as a path of length $N$.

\begin{definition}[Information Propagation Path Construction]  Let $x$ and $Y$ be jointly distributed as in Definition \ref{def:lb}. 
We define a path graph $A_1 \to A_2 \to \dots \to A_N$ for some large $N$. The features are assigned cyclically to the agents along this path. Specifically, agent $A_j$ is assigned the single feature $x_i$, where $i = ((j-1) \pmod k) + 1$. This structure models a process where information must be passed repeatedly through the same sequence of feature types. We refer to a full cycle of $k$ agents (e.g., $A_1, \dots, A_k$ or $A_{k+1}, \dots, A_{2k}$) as a ``pass''. 
\end{definition}

Using this lower bound instance we prove the following theorem in this section:
\begin{theorem}[Lower Bound on Convergence Rate]
\label{thm:main_lower_bound_propagation}
For any pass $p \in \{1, \ldots, k-1\}$, the mean squared error of the predictor from the agent observing feature $x_k$ at the end of pass $p$, denoted $\hat{y}_{p,k}$, is bounded below by:
$$\mathbb{E}[(Y - \hat{y}_{p,k})^2] \ge \frac{1}{p+1}.$$
\end{theorem}

\begin{remark}
Recall that Corollary \ref{cor:best_linear_error} showed that in any graph containing a path of length $D$ for which each variable appeared in every contiguous sequence of length $M$, the excess error $\eta$ of the last Agent in the path was at most $O\left( \frac{M}{\sqrt{D}}\right)$. The lower bound in Theorem \ref{thm:main_lower_bound_propagation} above uses an instance in which $M = k$ and $D = p\cdot k$ and so gives a lower bound of $\Omega\left(\frac{M}{D}\right)$, showing that the dependence on $D$ in our upper bound cannot be improved by more than a quadratic factor.
\end{remark}

\subsubsection{Analysis of the Predictor Structure}

The key to analyzing this construction is to understand the structure of the sequential predictor $\hat{y}_t$. Because the features $x_i$ are constructed from the latent variables $z_j$, the predictor $\hat{y}_t$ will always be a linear combination of the $z_j$. We can characterize exactly which $z_j$ contribute to the predictor after each pass.

\begin{lemma}[Predictor Structure and Information Propagation]
\label{lem:predictor_structure_propagation}
Let $\hat{y}_{p,i}$ denote the predictor from the agent that observes feature $x_i$ during pass $p$. Let $S_{p,i} = \{ j \mid \mathbb{E}[\hat{y}_{p,i} z_j] \neq 0 \}$ be the set of latent variables correlated with the predictor. Let $\hat{y}_{p,0} = \hat{y}_{p-1,k}$ be the prediction from the end of the previous pass (with $\hat{y}_{0,k}=0$).

For any pass $p \ge 1$, the following hold:
\begin{enumerate}
    \item The predictor does not change during the initial phase of the pass. For $i=1, \ldots, k-p$, we have $\hat{y}_{p,i} = \hat{y}_{p,i-1}$.
    \item The first non-trivial update in pass $p$ occurs at the agent observing feature $x_{k-p+1}$.
    \item At the end of pass $p$, the predictor $\hat{y}_{p,k}$ is a linear combination of the variables in the set $S_{p,k} = \{z_k, z_{k-1}, \ldots, z_{k-p}\}$.
\end{enumerate}
\end{lemma}

\begin{proof}
The proof is by induction on the pass number $p$.

\textbf{Base Case (p=1):} The initial prediction is $\hat{y}_{1,0} = \hat{y}_{0,k} = 0$. For agents observing features $x_i$ with $i=1, \ldots, k-1$, the local feature $x_i$ is uncorrelated with the target $Y=z_k$. Since the incoming prediction $\hat{y}_{1,i-1}$ is also 0, the optimal new prediction is $\hat{y}_{1,i}=0$. This proves (1) and (2) for the first pass. 

At the agent observing feature $x_k$, the learner sees feature $x_k = z_k - z_{k-1}$ and receives the prediction $\hat{y}_{1,k-1}=0$. The optimal predictor is $\hat{y}_{1,k} = \frac{\mathbb{E}[Y x_k]}{\mathbb{E}[x_k^2]} x_k = \frac{1}{2}(z_k - z_{k-1})$. The predictor is a linear combination of $z_k$ and $z_{k-1}$, so its support is $S_{1,k} = \{z_k, z_{k-1}\}$. This matches the set described in (3) for $p=1$ (since $z_{k-p} = z_{k-1}$).

\textbf{Inductive Step:} Assume the claims hold for pass $p-1$. Thus, the predictor at the end of pass $p-1$, $\hat{y}_{p-1,k}$, is a linear combination of variables in $S_{p-1,k} = \{z_k, \ldots, z_{k-(p-1)}\}$.

Now consider pass $p$. The initial prediction for this pass is $\hat{y}_{p,0} = \hat{y}_{p-1,k}$.
For an agent observing feature $x_i$ with $i \le k-p$, the local feature is $x_i = z_i - z_{i-1}$. The indices $i$ and $i-1$ are both smaller than $k-p+1$. Therefore, $x_i$ is uncorrelated with every $z_j \in S_{p-1,k}$. The incoming predictor $\hat{y}_{p,i-1}$ (which equals $\hat{y}_{p-1,k}$ for $i=1$) has support $S_{p-1,k}$, so it is also uncorrelated with $x_i$. The update is trivial: $\hat{y}_{p,i} = \hat{y}_{p,i-1}$. This proves (1) for pass $p$.

The first agent for which $x_i$ is correlated with the existing predictor is the one observing feature $x_{k-p+1}$. Its feature $x_{k-p+1} = z_{k-p+1} - z_{k-p}$ contains the variable $z_{k-p+1}$, which is in the support $S_{p-1,k}$ of the incoming predictor $\hat{y}_{p,k-p} = \hat{y}_{p-1,k}$. This allows a non-trivial update, proving (2).

When the agent observing feature $x_{k-p+1}$ is activated, its new predictor $\hat{y}_{p, k-p+1}$ will be a linear combination of the incoming predictor $\hat{y}_{p-1,k}$ and its local feature $x_{k-p+1}$. The support of the new predictor will be the union of the supports of these two terms. The support of $\hat{y}_{p-1,k}$ is $S_{p-1,k} = \{z_k, \ldots, z_{k-p+1}\}$ and the support of $x_{k-p+1}$ is $\{z_{k-p+1}, z_{k-p}\}$.

Thus, the support of the new predictor is $S_{p, k-p+1} = S_{p-1,k} \cup \{z_{k-p+1}, z_{k-p}\} = \{z_k, \ldots, z_{k-p}\}$. Subsequent updates within pass $p$ (for agents observing features $x_i$ with $i > k-p+1$) will only involve features $x_i$ constructed from latent variables already in this new support set. Therefore, no new variables are introduced, and the support set remains unchanged for the rest of the pass. At the end of the pass, the support is $S_{p,k} = \{z_k, \ldots, z_{k-p}\}$. This proves (3) for pass $p$.
\end{proof}

\subsubsection{Lower Bound on Mean Squared Error}

Lemma~\ref{lem:predictor_structure_propagation} establishes which features contribute to the predictor after a given number of passes. We now calculate the minimum possible error for a predictor that is restricted to a suffix of the features.

\begin{lemma}[MSE for Suffix Predictor]
\label{lem:mse_lower_bound_suffix}
For any $j \in \{2, \ldots, k\}$, the minimum mean squared error for predicting $Y=z_k$ using a linear combination of the features $\{x_j, x_{j+1}, \ldots, x_k\}$ is given by:
$$\min_{f \in \text{span}(x_j, \ldots, x_k)} \mathbb{E}[(Y - f(x))^2] = \frac{1}{k-j+2}$$
\end{lemma}

\begin{proof}
The optimal predictor is the conditional expectation $\hat{Y}_j = \mathbb{E}[Y | x_j, \ldots, x_k]$. The minimum MSE is the variance of the residual, $\mathbb{E}[(Y - \hat{Y}_j)^2] = \text{Var}(Y | x_j, \ldots, x_k)$.

We can express the target variable $Y$ in terms of the features and a latent variable:
$$ Y = z_k = (z_k - z_{k-1}) + (z_{k-1} - z_{k-2}) + \dots + (z_j - z_{j-1}) + z_{j-1} = \sum_{i=j}^k x_i + z_{j-1} $$
Since the features $x_j, \ldots, x_k$ are given, the conditional variance is:
$$ \text{Var}(Y | x_j, \ldots, x_k) = \text{Var}\left(\sum_{i=j}^k x_i + z_{j-1} \Big| x_j, \ldots, x_k\right) = \text{Var}(z_{j-1} | x_j, \ldots, x_k) $$
Since all variables are jointly Gaussian, we can use the formula for conditional variance:
$$ \text{Var}(z_{j-1} | x_j, \ldots, x_k) = \text{Var}(z_{j-1}) - \text{Cov}(z_{j-1}, X_{j:k}) \text{Var}(X_{j:k})^{-1} \text{Cov}(X_{j:k}, z_{j-1}) $$
where $X_{j:k} = (x_j, \ldots, x_k)^T$ is the vector of features. We have $\text{Var}(z_{j-1}) = 1$. The covariance vector is $\text{Cov}(z_{j-1}, X_{j:k}) = (\mathbb{E}[z_{j-1}x_j], \ldots, \mathbb{E}[z_{j-1}x_k])$.
For $i>j$, $\mathbb{E}[z_{j-1}x_i] = \mathbb{E}[z_{j-1}(z_i - z_{i-1})] = 0$. For $i=j$, $\mathbb{E}[z_{j-1}x_j] = \mathbb{E}[z_{j-1}(z_j - z_{j-1})] = -1$. So, $\text{Cov}(z_{j-1}, X_{j:k}) = (-1, 0, \ldots, 0)$.

Let $C = \text{Var}(X_{j:k})$. The MSE is $1 - (-1, 0, \ldots, 0) C^{-1} (-1, 0, \ldots, 0)^T = 1 - (C^{-1})_{11}$.

The covariance matrix $C$ is a $(k-j+1) \times (k-j+1)$ matrix. For $i, l \ge j > 1$, the entries are $\mathbb{E}[x_i x_l]$. The diagonal entries are $\mathbb{E}[x_i^2]=2$. The off-diagonal entries are $\mathbb{E}[x_i x_{i+1}] = -1$. Thus, $C$ is the symmetric tridiagonal matrix with 2s on the diagonal and -1s on the first off-diagonal.

Let $m=k-j+1$. The inverse of this $m \times m$ matrix has  entries $(C^{-1})_{ab} = \frac{1}{m+1}\min(a,b)(m+1 - \max(a,b))$.
We need the $(1,1)$ entry:
$$ (C^{-1})_{11} = \frac{1}{m+1} \cdot 1 \cdot (m+1-1) = \frac{m}{m+1} $$
Substituting back, the MSE is $1 - \frac{m}{m+1} = \frac{1}{m+1} = \frac{1}{k-j+2}$.
\end{proof}

Putting these two lemmas together lets us prove the lower bound:

\begin{proof}[Proof of Theorem \ref{thm:main_lower_bound_propagation}]
From Lemma~\ref{lem:predictor_structure_propagation}, we know that the predictor from the agent observing feature $x_k$ at the end of pass $p$, $\hat{y}_{p,k}$, is a linear combination of the latent variables in the set $S_{p,k} = \{z_k, z_{k-1}, \ldots, z_{k-p}\}$.

Crucially, this means that $\hat{y}_{p,k}$ is independent of the latent variables $\{z_1, \ldots, z_{k-p-1}\}$.

The features $\{x_1, \ldots, x_{k-p}\}$ are constructed exclusively from these latent variables. Therefore, $\hat{y}_{p,k}$ is independent of the features $\{x_1, \ldots, x_{k-p}\}$.

This implies that the predictor $\hat{y}_{p,k}$ can be expressed as a function of only the features $\{x_{k-p+1}, \ldots, x_k\}$.

The minimum possible MSE for any such predictor is given by Lemma~\ref{lem:mse_lower_bound_suffix}, with $j = k-p+1$. Substituting this into the formula gives:
$$ \mathbb{E}[(Y - \hat{y}_{p,k})^2] \ge \min_{f \in \text{span}(x_{k-p+1}, \ldots, x_k)} \mathbb{E}[(Y - f(x))^2] = \frac{1}{k - (k-p+1) + 2} = \frac{1}{p+1} $$
This completes the proof.
\end{proof}

\subsection{Learning Requires Depth in Arbitrary DAGs}
\label{sec:lower_bound_general}

The previous section proved a lower bound specifically for path graphs (in which features were arranged in a worst-case manner) but we now prove a more general result. We demonstrate a fundamental separation between what a single, global learner can achieve and what is possible for distributed learning in any DAG. For any target depth $D$, we construct a specific problem instance where a global learner achieves zero error, yet the performance of distributed learning on \emph{any} DAG of depth $D$ with \emph{any} allocation of (single) variables to nodes in the DAG is strictly limited by its depth. This shows that \emph{depth} is the fundamental bottleneck to distributed learning in our model, and our upper bounds (which depend only on the depth of the DAG) cannot be modified to depend exclusively on any other graph property that can be scaled independently of depth.

We use the hard problem distribution from Definition~\ref{def:lb}, where $Y=z_k$ and features are defined as $x_j = z_j - z_{j-1}$. As we will show, solving this problem requires sequentially uncovering the latent variables $z_k, z_{k-1}, \ldots, z_1$, a task whose difficulty is dictated by the DAG's structure.

\begin{definition}[Depth of a Node and a DAG]
The \textbf{depth} of a node $v$ in a DAG, denoted $\text{depth}(v)$, is the length of the longest path from any root node to $v$. The \textbf{depth} of a DAG, $D$, is the maximum depth of any node in the graph.
\end{definition}

The core of our argument is to show that information is fundamentally bottlenecked by the DAG's depth. To make this argument as strong as possible, we establish a lower bound on error that holds even under the most favorable conditions for the learner. Specifically, we analyze a \emph{best-case feature allocation}, where features are assigned to nodes in the most helpful way possible to reveal the target $Y$. If we can show that error is high even in this best case, it must also be high for any other (less favorable) feature allocation.

\begin{lemma}[Information Propagation on a General DAG]
\label{lem:info_prop_dag}
Consider a node $v$ in a DAG with $\text{depth}(v) = p$. Even under an optimal feature allocation designed to reveal $Y$ as quickly as possible along a path to $v$, the predictor $\hat{y}_v$ can be a function of at most the latent variables in the set $\{z_k, z_{k-1}, \ldots, z_{k-p}\}$.
\end{lemma}

\begin{proof}
The proof is by induction on the depth $p$ of the node $v$. Let $\text{Info}(f)$ denote the set of latent variables $\{z_i\}$ that a function $f$ is linearly dependent on. The claim is that for any node $v$ with $\text{depth}(v)=p$, we have $\text{Info}(\hat{y}_v) \subseteq \{z_k, z_{k-1}, \ldots, z_{k-p}\}$.

\textbf{Base Case ($p=1$):} A node $v$ with $\text{depth}(v)=1$ is a root node and receives no parent predictions. Its predictor $\hat{y}_v$ is the orthogonal projection of $Y$ onto the span of its local features, $\text{span}(\{x_j\}_{j \in S_v})$. Since $\hat{y}_v \in \text{span}(\{x_j\}_{j \in S_v})$, its information content is bounded by the union of the information content of its features: $\text{Info}(\hat{y}_v) \subseteq \bigcup_{j \in S_v} \text{Info}(x_j)$.
To learn as much as possible about $Y=z_k$, the best-case allocation would assign features correlated with $z_k$. The most informative single feature is $x_k = z_k - z_{k-1}$. If $v$ is assigned this feature, then $\text{Info}(\hat{y}_v) \subseteq \text{Info}(x_k) = \{z_k, z_{k-1}\}$. This satisfies the claim for $p=1$.

\textbf{Inductive Step:} Assume the claim holds for all nodes with depth up to $p-1$. Let $v$ be a node with $\text{depth}(v) = p$. All parents $u \in \text{Pa}(v)$ must have $\text{depth}(u) \le p-1$.
By the inductive hypothesis, for any parent $u$, $\text{Info}(\hat{y}_u) \subseteq \{z_k, \ldots, z_{k-(p-1)}\}$.

The predictor $\hat{y}_v$ is the projection of $Y$ onto the subspace $\mathcal{S}_v = \text{span}(\{\hat{y}_u\}_{u \in \text{Pa}(v)} \cup \{x_j\}_{j \in S_v})$. Since $\hat{y}_v \in \mathcal{S}_v$, its information content is bounded by the union of the information from its basis vectors:
\[ \text{Info}(\hat{y}_v) \subseteq \left( \bigcup_{u \in \text{Pa}(v)} \text{Info}(\hat{y}_u) \right) \cup \left( \bigcup_{j \in S_v} \text{Info}(x_j) \right). \]
Substituting the inductive hypothesis, this becomes:
\[ \text{Info}(\hat{y}_v) \subseteq \{z_k, \ldots, z_{k-(p-1)}\} \cup \left( \bigcup_{j \in S_v} \text{Info}(x_j) \right). \]
To introduce a new, "deeper" latent variable $z_{k-p}$ into the predictor, the local features must contain it. The only features constructed from $z_{k-p}$ are $x_{k-p} = z_{k-p} - z_{k-p-1}$ and $x_{k-p+1} = z_{k-p+1} - z_{k-p}$.
Assigning $x_{k-p}$ is not useful, as it is uncorrelated with $Y=z_k$ and with all latent variables in the parents' information set, $\{z_k, \ldots, z_{k-(p-1)}\}$. Thus, it would not receive non-zero weight in the projection.
The only way to productively introduce $z_{k-p}$ is to assign feature $x_{k-p+1}$. This feature links the new variable $z_{k-p}$ to the existing information set via $z_{k-p+1}$.

Under this optimal assignment, the information set of the local features is $\text{Info}(x_{k-p+1}) = \{z_{k-p+1}, z_{k-p}\}$. The total information available to node $v$ is therefore:
\[ \text{Info}(\hat{y}_v) \subseteq \{z_k, \ldots, z_{k-(p-1)}\} \cup \{z_{k-p+1}, z_{k-p}\} = \{z_k, \ldots, z_{k-p}\}. \]
No other feature assignment can introduce information about a deeper latent variable (e.g., $z_{k-p-1}$) because doing so would require prior information about $z_{k-p}$, which is not available at depth $p$. This completes the induction.
\end{proof}

This lemma establishes a fundamental limit on information propagation based on graph depth. We can now combine this with our earlier MSE calculation.

\begin{theorem}[Depth as a Fundamental Barrier]
\label{thm:general_lower_bound}
For any depth $D \ge 1$, there exists a joint distribution over $(X,Y)$  where the optimal linear predictor has an MSE of 0, but for any DAG with depth $D$ and for any allocation of single features to its agents, the MSE of the predictor at any node $v$ is lower bounded by:
\[ \text{MSE}(\hat{y}_v) \ge \frac{1}{D+1} \]
\end{theorem}

\begin{proof}
For a given depth $D \ge 1$, we use the label/feature distribution from Definition~\ref{def:lb} with a number of features $k$ such that $k > D$. (For instance, we can set $k=D+1$, the minimal number required for the argument). In this instance, the target variable is $Y = z_k$, and the features are $x_1=z_1$ and $x_i = z_i - z_{i-1}$ for $i>1$. The target can be written as a linear combination of the full feature set: $Y = \sum_{i=1}^k x_i$. Therefore, the optimal linear predictor using all features $\{x_1, \ldots, x_k\}$ is $f(X) = \sum_{i=1}^k x_i$, which achieves an MSE of 0.

Now, consider any DAG with depth $D$ and any allocation of single features to its nodes. Let $v$ be any node in the DAG. Its depth, $\text{depth}(v)$, is at most $D$. Lemma~\ref{lem:info_prop_dag} states that even under the best possible feature allocation for the learner, the predictor $\hat{y}_v$ can be a function of at most the latent variables $\{z_k, \ldots, z_{k-\text{depth}(v)}\}$. 

The minimum possible MSE for any such predictor is given by Lemma~\ref{lem:mse_lower_bound_suffix}, with $j = k-\text{depth}(v)+1$. Substituting this into the formula gives a lower bound on the MSE:
$$ \mathbb{E}[(Y - \hat{y}_v)^2] \ge \min_{f \in \text{span}(x_{k-\text{depth}(v)+1}, \ldots, x_k)} \mathbb{E}[(Y - f(x))^2] = \frac{1}{k - (k-\text{depth}(v)+1) + 2} = \frac{1}{\text{depth}(v)+1} $$
Since this holds for any node $v$, and $\text{depth}(v) \le D$, the error for any node is bounded by:
\[ \text{MSE}(\hat{y}_v) \ge \frac{1}{\text{depth}(v)+1} \ge \frac{1}{D+1} \]
This completes the proof.
\end{proof}

Focusing solely on the depth $p$ of a vertex in our lower bound construction, it is worth noting
that there is a gap between our general upper bounds on excess error, which scale as $1/\sqrt{p}$, and
the lower bound, which scales as $1/p$. It is an interesting theoretical problem to resolve what the
``right'' dependence on $p$ is, but in Figure~\ref{fig:lowerbound} we show simulation results for
the lower bound construction. The x-axis measures the depth of position of a vertex in the chain,
and the y-axis measures error. The red line shows train and test error at the vertex (which are
nearly identical and thus occluded due to the large sample size), and in the early stages of learning
the errors reduce in phases corresponding to the sweeps through subsequent blocks of the lower bound
construction. We also overlay on this simulation curve the best fits of the form
$\alpha/p$ and $\beta/\sqrt{p}$ for parameters $\alpha, \beta$. Numerically, the $\beta/\sqrt{p}$
family is the better fit, and this holds at longer chain lengths for the same lower bound construction,
suggesting empirically that our upper bounds may be close to tight, and that this might be demonstrated by a better analysis of the same lower bound instance.

\begin{figure}[H]
    \centering
    \includegraphics[width=0.5\textwidth,keepaspectratio]{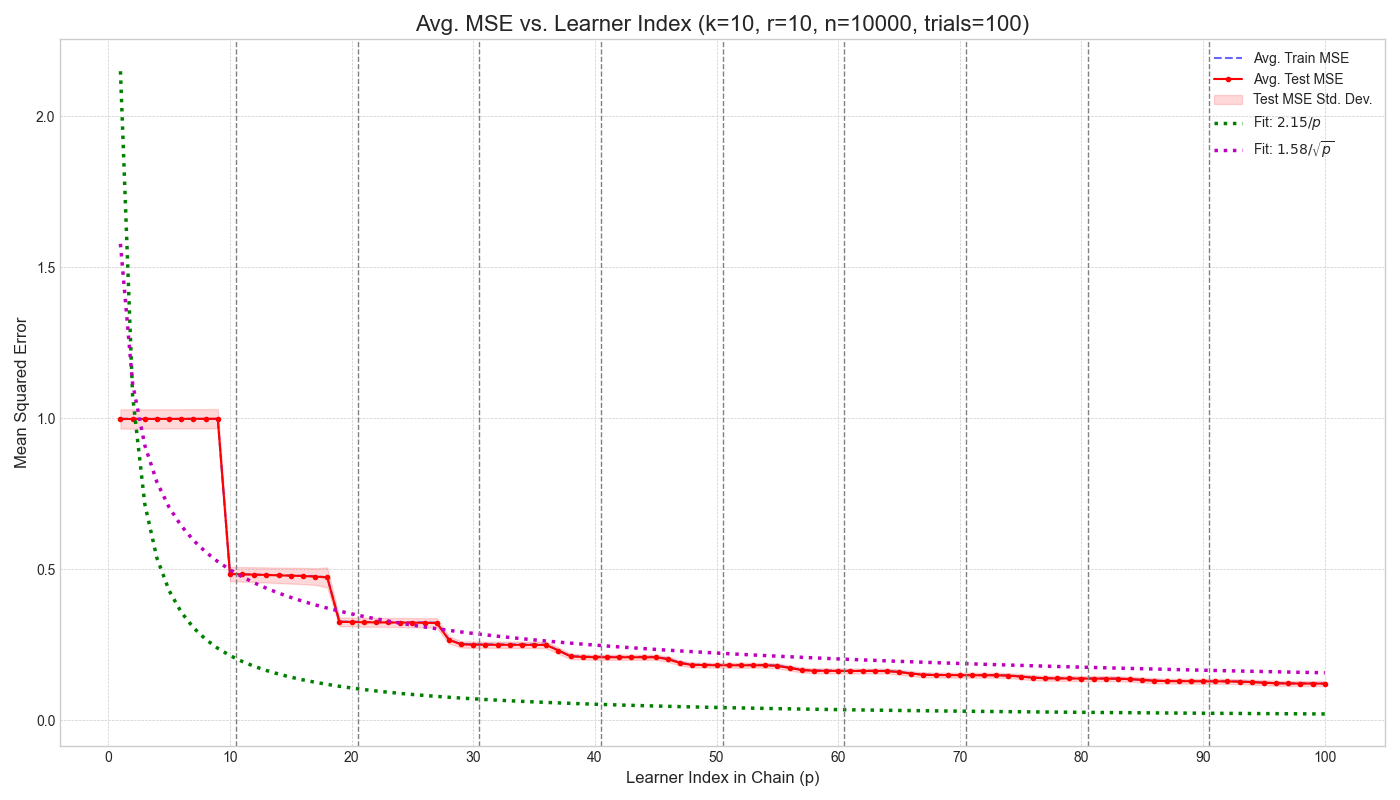}
    \caption{Simulation results for lower bound construction. See text for discussion.}
    \label{fig:lowerbound}
\end{figure}

\section{Experimental Results}
\label{sec:experiments}

We conclude with an experimental investigation of the framework and theoretical
results we have presented. For this we used two tabular datasets for regression, one in which the
goal is to predict wine quality from chemical and sensory properties (11 features,
4898 instances; see \url{https://archive.ics.uci.edu/dataset/186/wine+quality}), and the other in which the goal is to predict household appliance usage
from temperature, humidity, and other environmental variables (26 features, 19,735 instances;
see \url{https://archive.ics.uci.edu/dataset/374/appliances+energy+prediction}). In all
experiments, standard train/test methodology was employed over many trials, as detailed below.

We begin with Figure~\ref{fig:linear}, which shows experimental results for the case in which all
learners use linear models. The left column of the figure corresponds to the wine quality dataset,
while the right corresponds to the appliance energy usage dataset.
The first row of figures shows the basic setting in which 50 learners are arranged
in a left-to-right chain. As per the theoretical framework, each learner sees a random 
subset of the features and performs least squares regression, and then passes its predictions
to the next learner in the chain as an additional feature. We vary the fraction of features seen
by each learner from 0.1 to 0.5 in increments of 0.1, and show both train and test errors as well
as test error error bars; the fraction of features is color coded as described in the plot legends.
The plots average over many runs in which both the specific random subsets of features, and the
train/test splits, are varied.
The x-axis is sorted by position in the chain, with the first learner being leftmost and the last
being rightmost. We also show the natural baseline of the test error of a single linear model trained
using all the features as a dashed red horizontal line; this is the best we could hope to approach as
information is aggregated along the chain of learners.

Beginning with the results for the wine quality dataset on the left, we see results strongly corroborating
the theory. Convergence to the globally optimal linear model is rapid for all values of the feature fraction,
with steady improvement in the convergence rate as we move from learners seeing 0.1 to 0.5 of the feature
space. Results are qualitatively similar for the appliance energy usage dataset on the right, except that
for the smallest feature fraction of 0.1 we have not converged to global optimal after 50 learners, presumably
due to the larger number of features for this dataset.

In the remaining rows of Figure~\ref{fig:linear}, we explore more complex graph topologies than a simple
chain, instead considering tree DAGs of 50 vertices or learners in which the direction of edges is either top-down from the root, or bottom-up from the leaves. Figure~\ref{fig:tree} shows an example of a top-down tree; for
bottom-up topologies, the direction of all edges would be reversed.

The second row of Figure~\ref{fig:linear} is similar to the first row, except for top-down trees rather
than a simple chain. Thus the order of the x-axis is by topological sort --- in other words, the leftmost
point corresponds to the root of the tree, followed by its children in a breadth-first search, followed by
their children, and so on. Performance on both datasets is similar to the chain case of the first row,
with a couple of natural but noteworthy exceptions. First, convergence is no longer monotonic in vertex
ordering, simply because learners at the same depth have aggregated incomparable information along
their different paths from the root. Second, unlike the chain case, even the rightmost learners in
the ordering have not converged to the global optimal performance for several of the smaller values
of the feature fraction. This is because less aggregation occurs in the leaves of a tree than in
a chain of the same size.

The observation about non-monotonicity in top-down ordering suggests that this might not be
the natural ordering to measure information aggregation --- for instance, there is no reason
to believe that the learners appearing, say, 13th in a top-down ordering averaged over many random
trees will have the same depth. To highlight this, in the third row of Figure~\ref{fig:linear} we
instead order the x-axis by learner depth in the tree, again averaged over many runs/trees. For
both datasets, these plots show that depth is the right correlate of performance in a top-down
tree --- convergence is now monotone in depth, and the aforementioned variability of depth at a
given breadth-first order position is now eradicating, resulting in dramatically reduced error bars.

In the fourth row of Figure~\ref{fig:linear}, we again consider trees, but now with bottom-up
information aggregation. Thus the leaves are the initial learners, who in turn pass their predictions
to their parents, and so on up the tree. The ordering of learners on the x-axis is again via
topological sort, averaged over runs --- thus the (say) 13th position corresponds to all
the learners that we processed 13th in a bottom-up sweep over many trees. As the plot clearly
shows, topological sorting is even less correlated with performance than it was for top-down trees,
with the error plots not just non-monotonic but almost random. Again this is because bottom-up
visitation order is simply not the right measure as it does not correlate well with the amount
of information aggregated. In a bottom-up ordering, a more natural quantity would be the size of
the subtree rooted at a vertex/learner, since that is a measure of the total number of learners
feeding into a given learner. In the fifth and final row of Figure~\ref{fig:linear} we reorder
the x-axis by subtree size, again recovering monotonicity and tight error bars.

While our theoretical results and lower bounds for general DAGs emphasize the importance of depth --- the
length of the longest path leading to a given learner --- the experimental results for chains, top-down
trees, and bottom-up trees show that empirically, the right correlate of learner performance may be
more nuanced beyond the worst case. 
Quantities such as network centrality, degree distribution, clustering coefficient and
other topological measures might be deserving of further empirical and theoretical study.

Finally, Figure~\ref{fig:nn} has the exact same form and narrative as Figure~\ref{fig:linear},
except now for the case in which all learners fit neural networks with 3 hidden units at a learning
rate of 0.01 and 100 epochs per trial. Since we no longer have the worst-case optimization guarantees of
convex models, we cannot expect monotonicity even in the case of a chain, and indeed
in the first row we see that for both datasets, error sometimes increase on average as we move
further along the chain. However, convergence remains strong and is now able to beat the 
globally optimal linear model (dashed red line) since learners have greater representational power.
The remaining rows of Figure~\ref{fig:nn} show that the same qualitative phenomena regarding
tree position in the top-down and bottom-up cases from the linear case still hold for neural 
networks, except for occasional non-monotonicities resulting from our inability to guarantee
optimal fits to the training data.

\begin{figure}[h]
    \centering
    \includegraphics[width=0.48\textwidth, height=0.195\textheight]{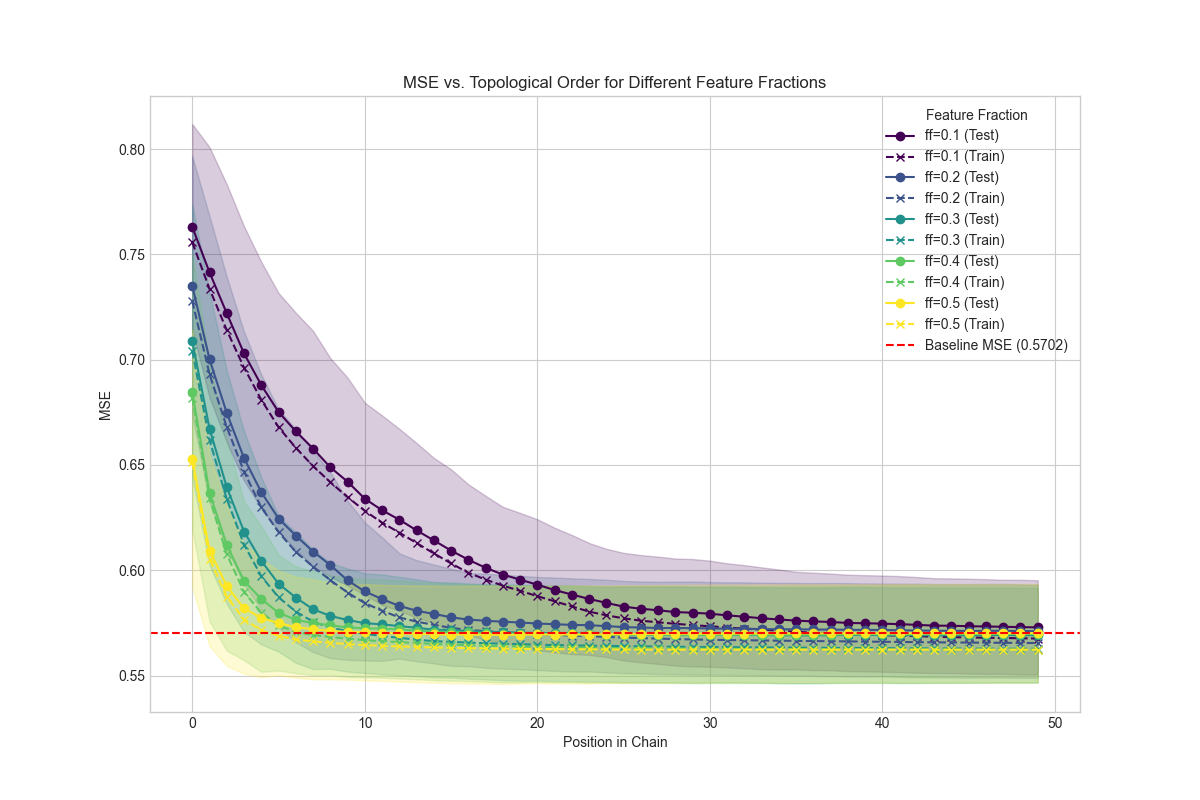}\hfill%
    \includegraphics[width=0.48\textwidth, height=0.195\textheight]{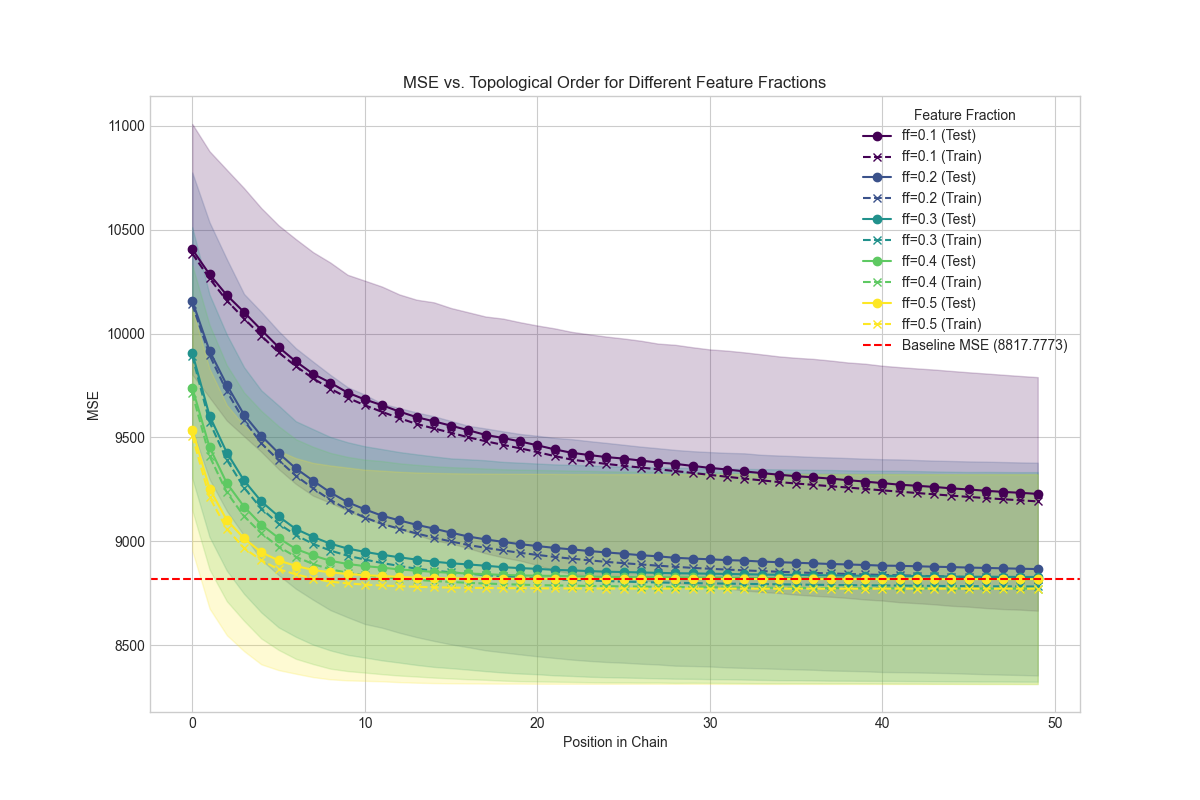}
    \vfill
    \includegraphics[width=0.48\textwidth, height=0.195\textheight]{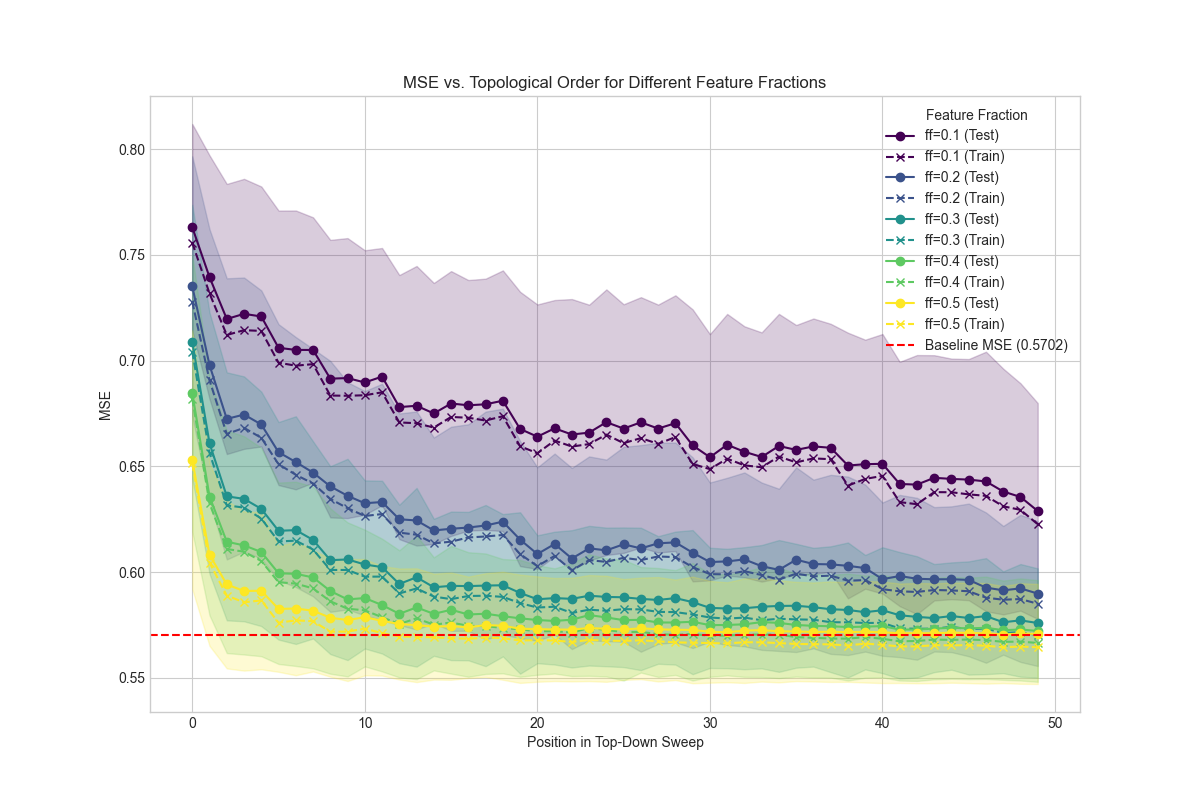}\hfill%
    \includegraphics[width=0.48\textwidth, height=0.195\textheight]{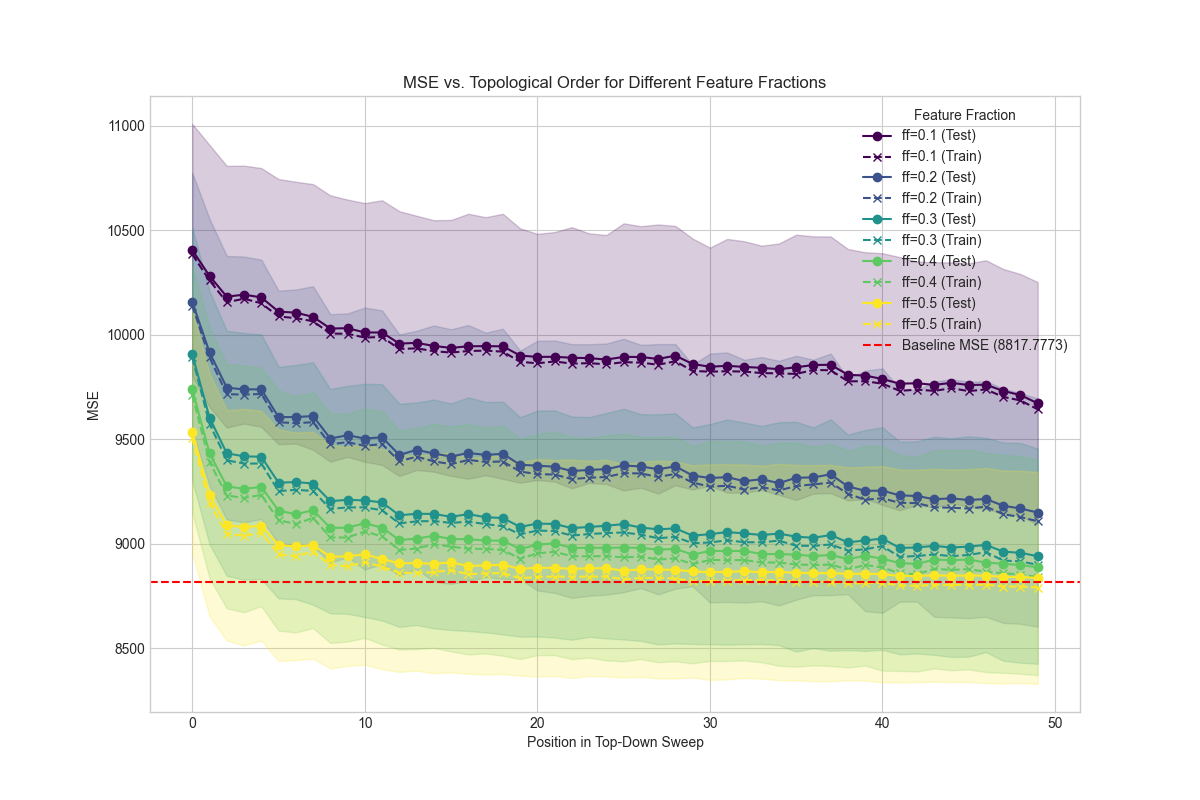}
    \vfill
    \includegraphics[width=0.48\textwidth, height=0.195\textheight]{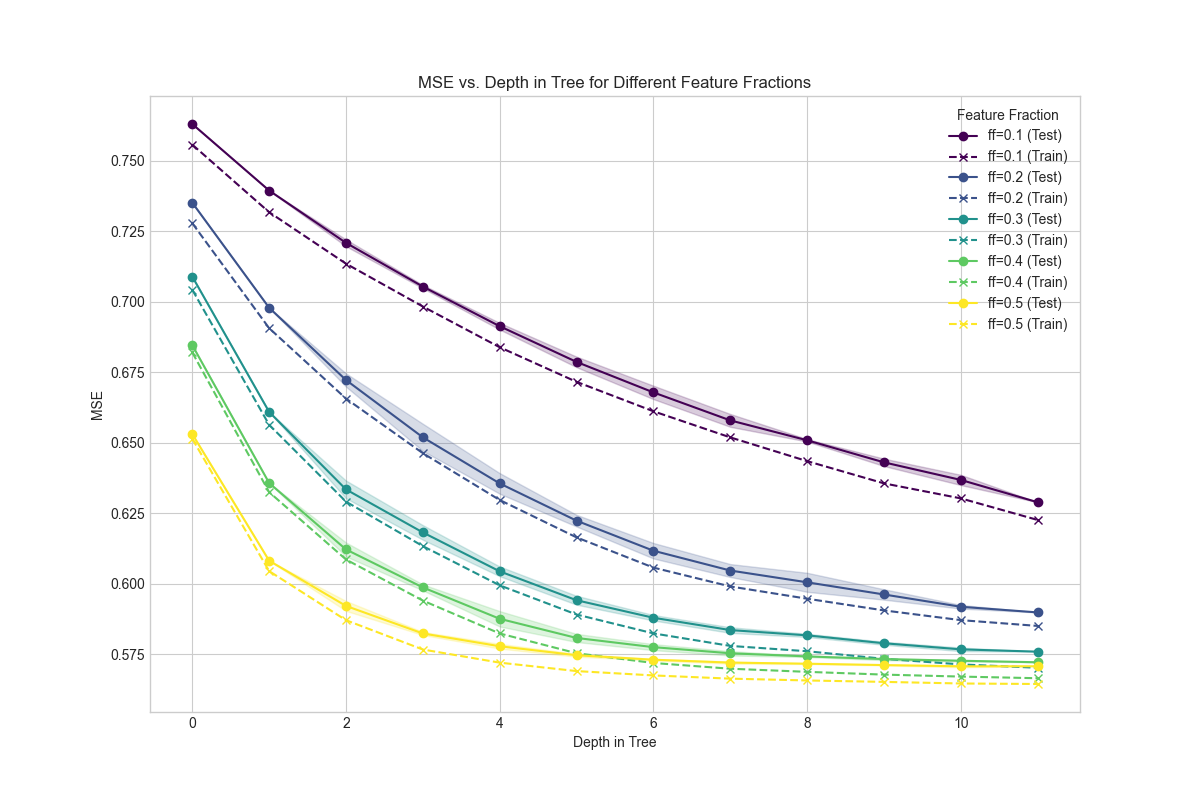}\hfill%
    \includegraphics[width=0.48\textwidth, height=0.195\textheight]{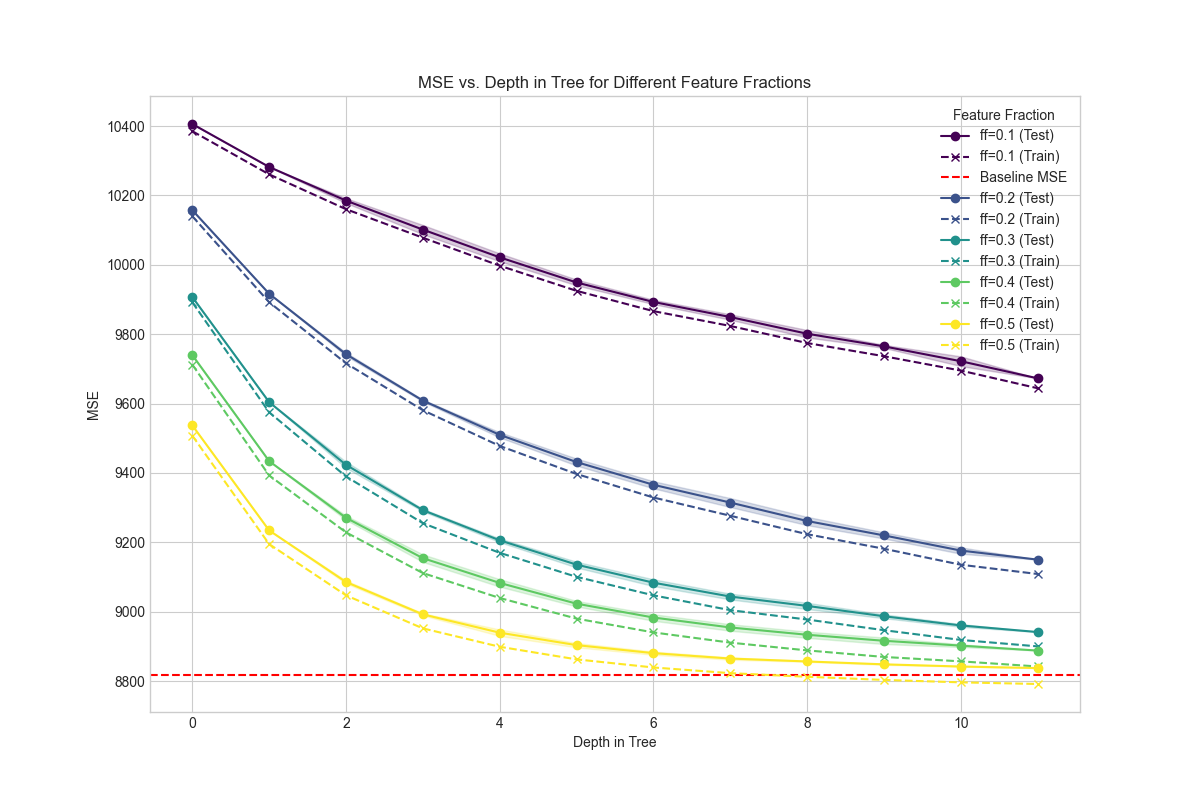}
    \vfill
    \includegraphics[width=0.48\textwidth, height=0.195\textheight]{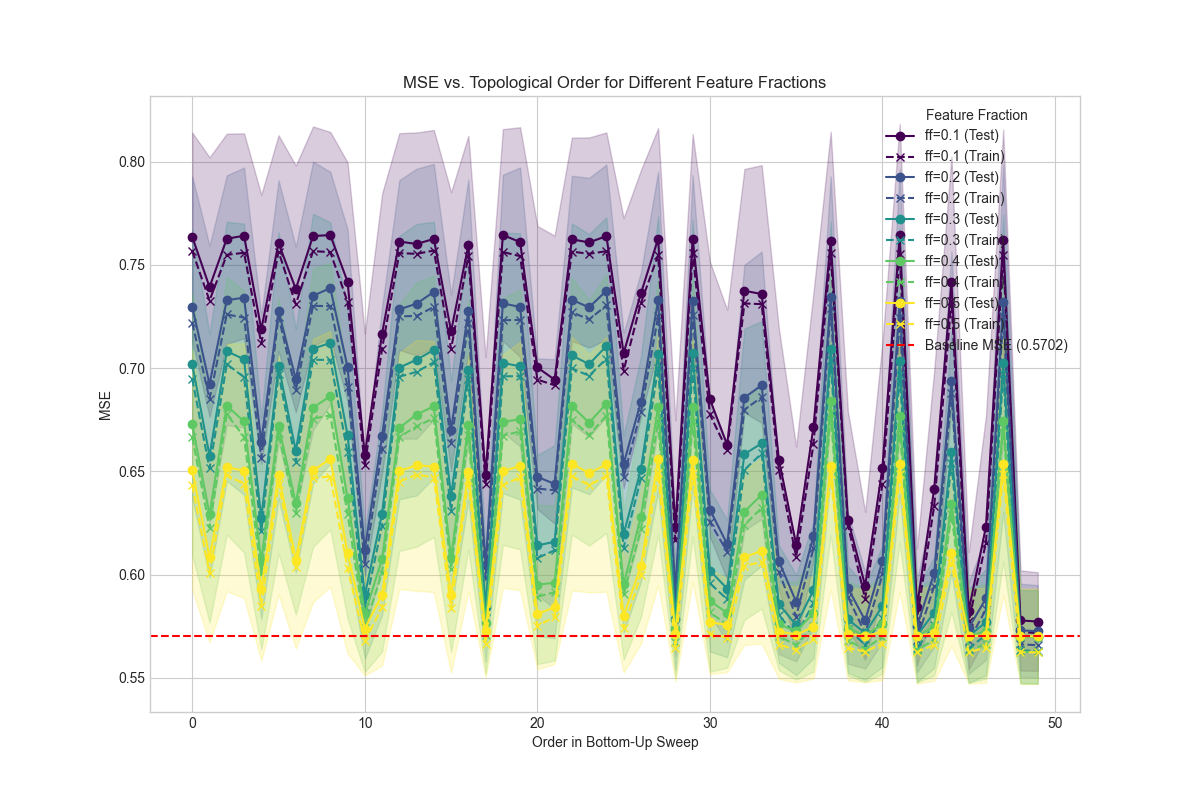}\hfill%
    \includegraphics[width=0.48\textwidth, height=0.195\textheight]{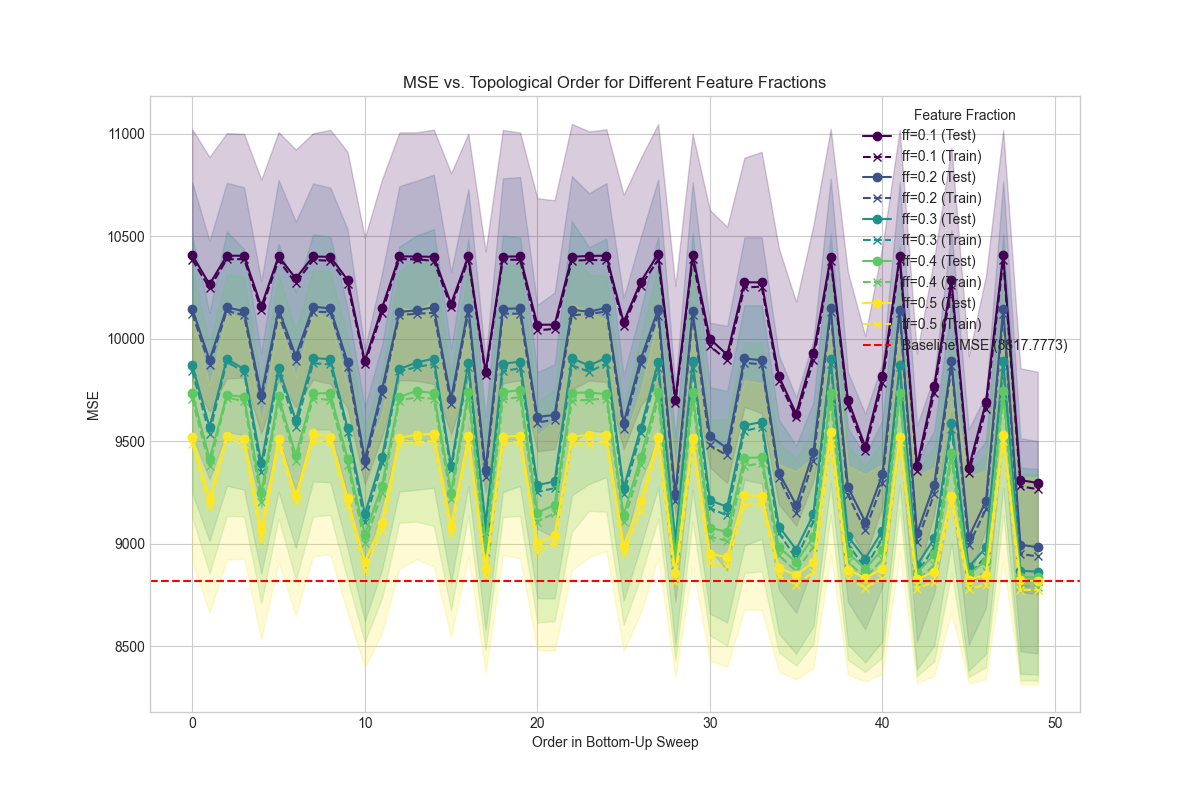}
    \vfill
    \includegraphics[width=0.48\textwidth, height=0.195\textheight]{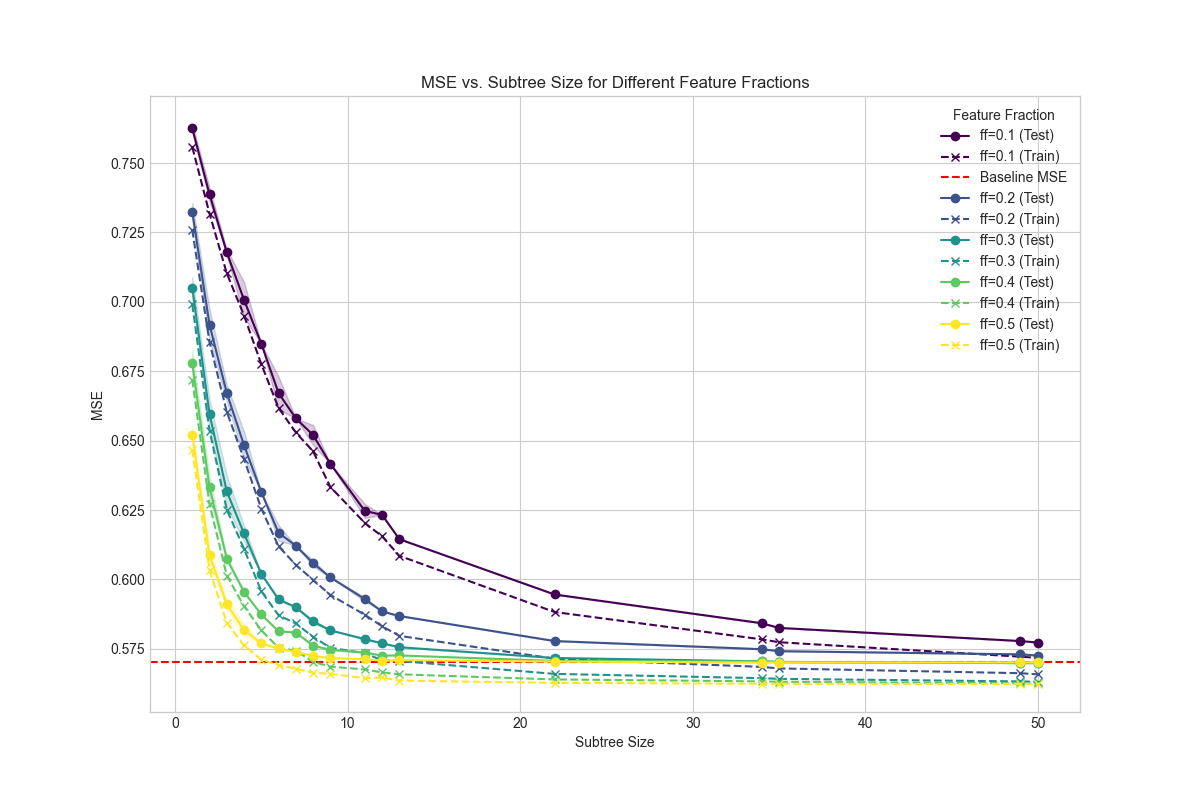}\hfill%
    \includegraphics[width=0.48\textwidth, height=0.195\textheight]{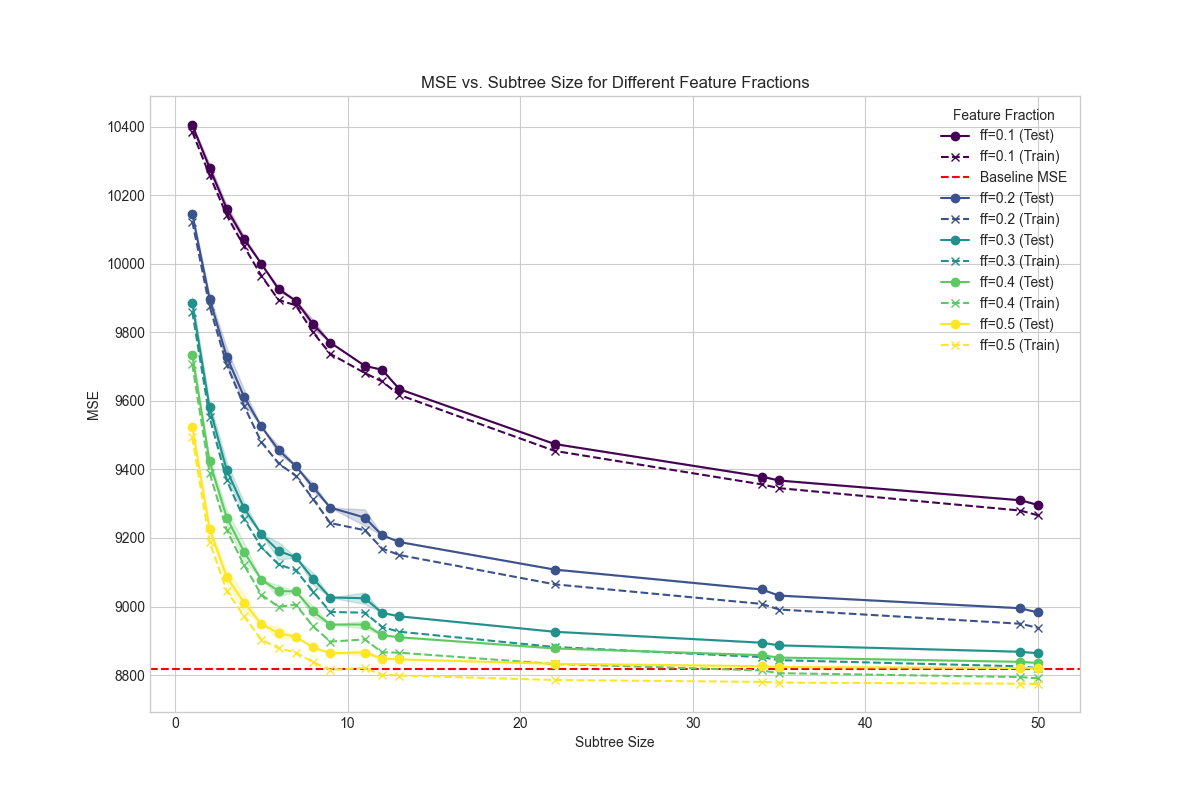}
    \caption{Experimental results for linear models. Left column corresponds to wine quality dataset, right
        column to appliance energy usage dataset. See text for details.}
    \label{fig:linear}
\end{figure}

\begin{figure}[h]
    \centering
    \includegraphics[width=\textwidth,height=\textheight,keepaspectratio]{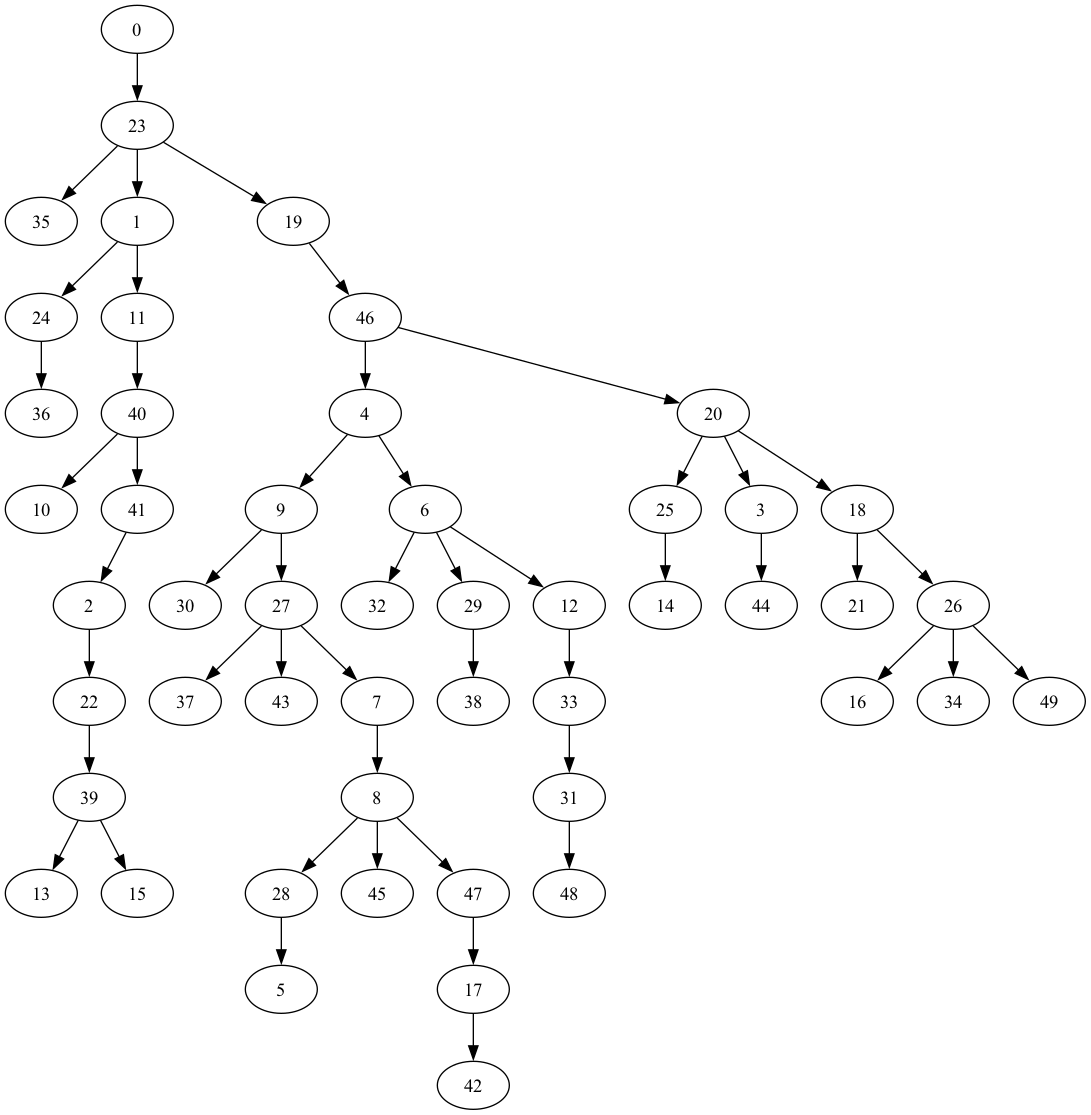}
    \caption{Sample top-down tree of 50 learners.}
    \label{fig:tree}
\end{figure}

\begin{figure}[h]
    \centering
    \includegraphics[height=0.198\textheight]{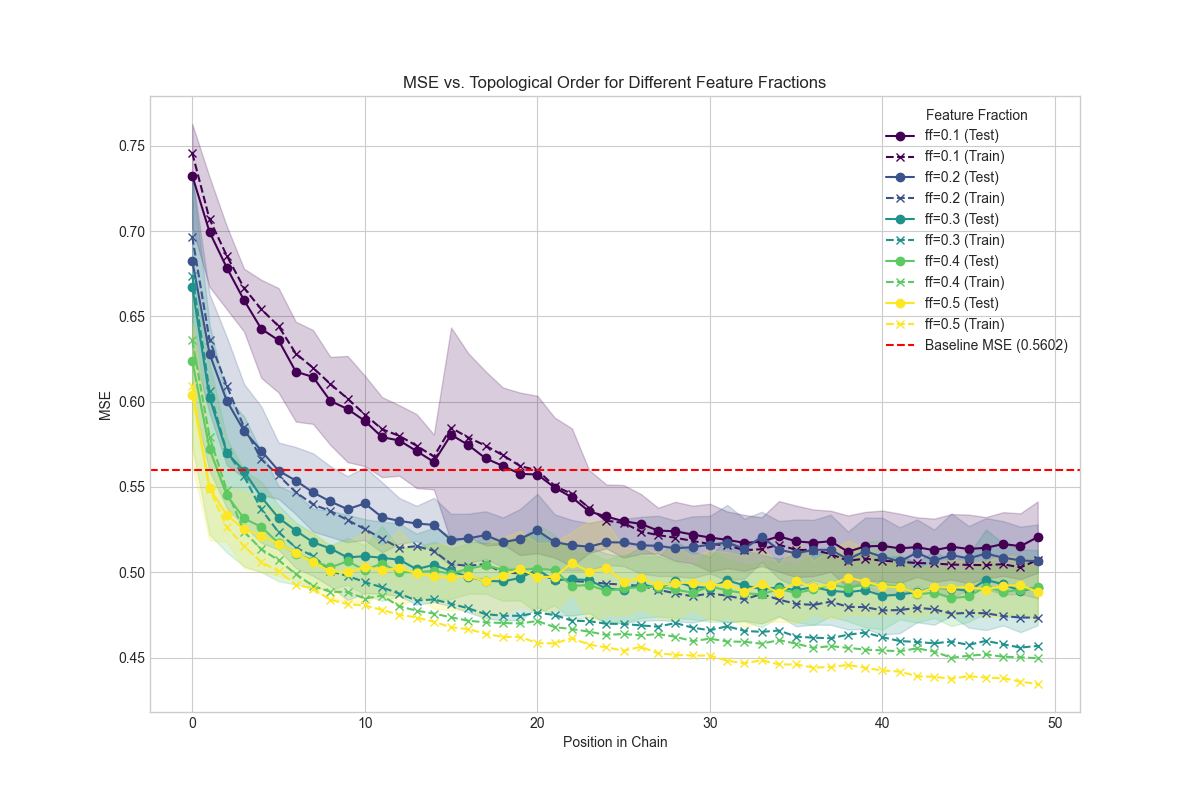}\hfill%
    \includegraphics[height=0.198\textheight]{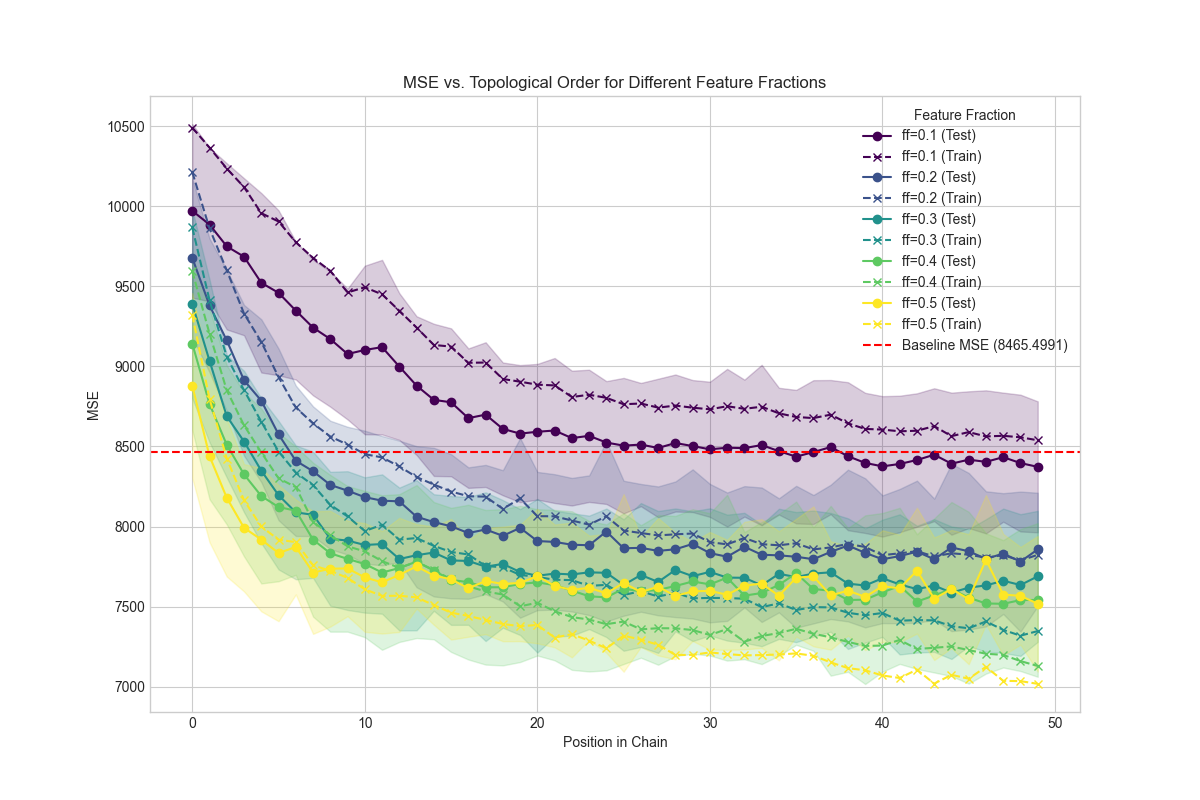}
    \vfill
    \includegraphics[height=0.198\textheight]{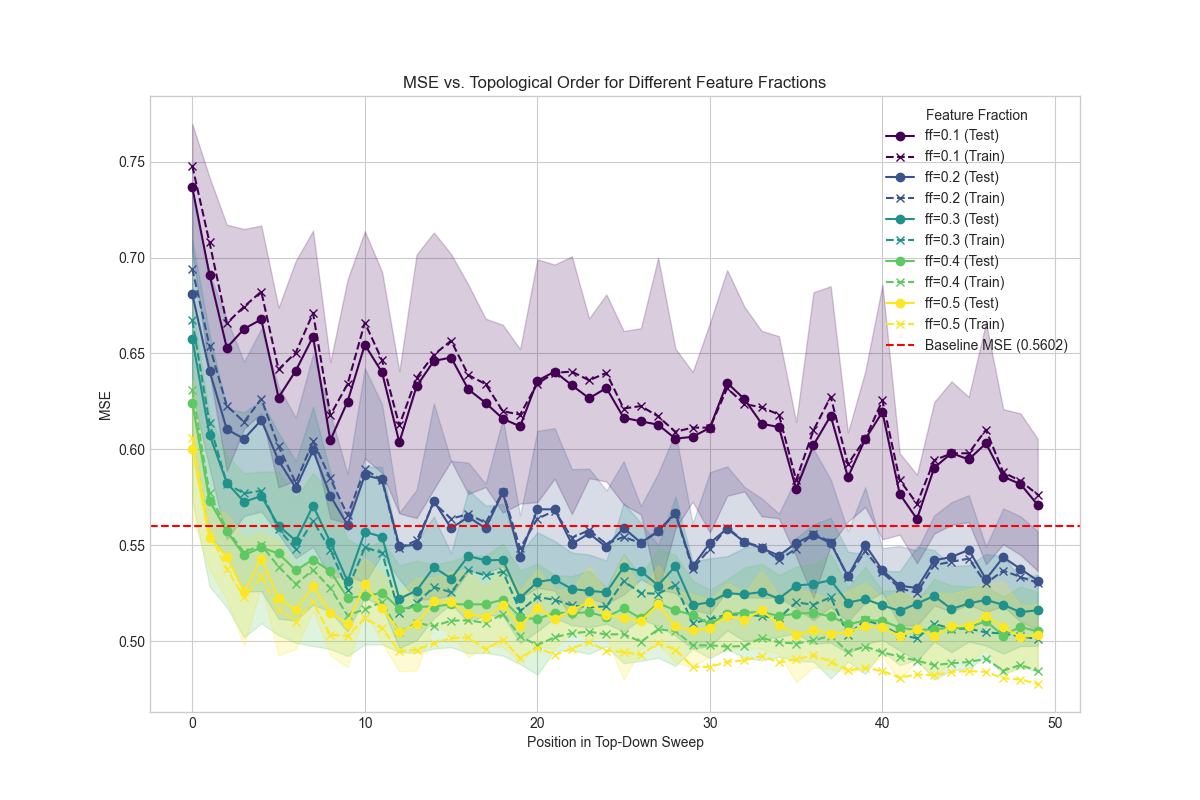}\hfill%
    \includegraphics[height=0.198\textheight]{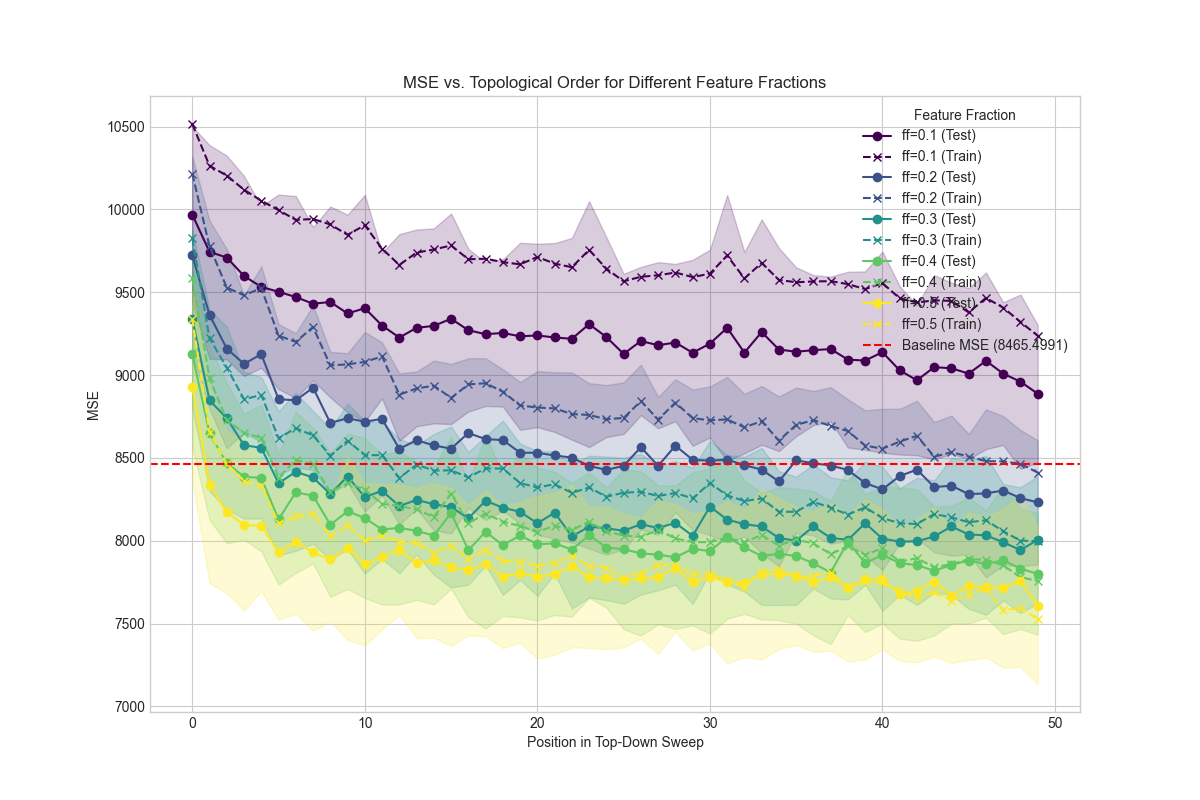}
    \vfill
    \includegraphics[height=0.198\textheight]{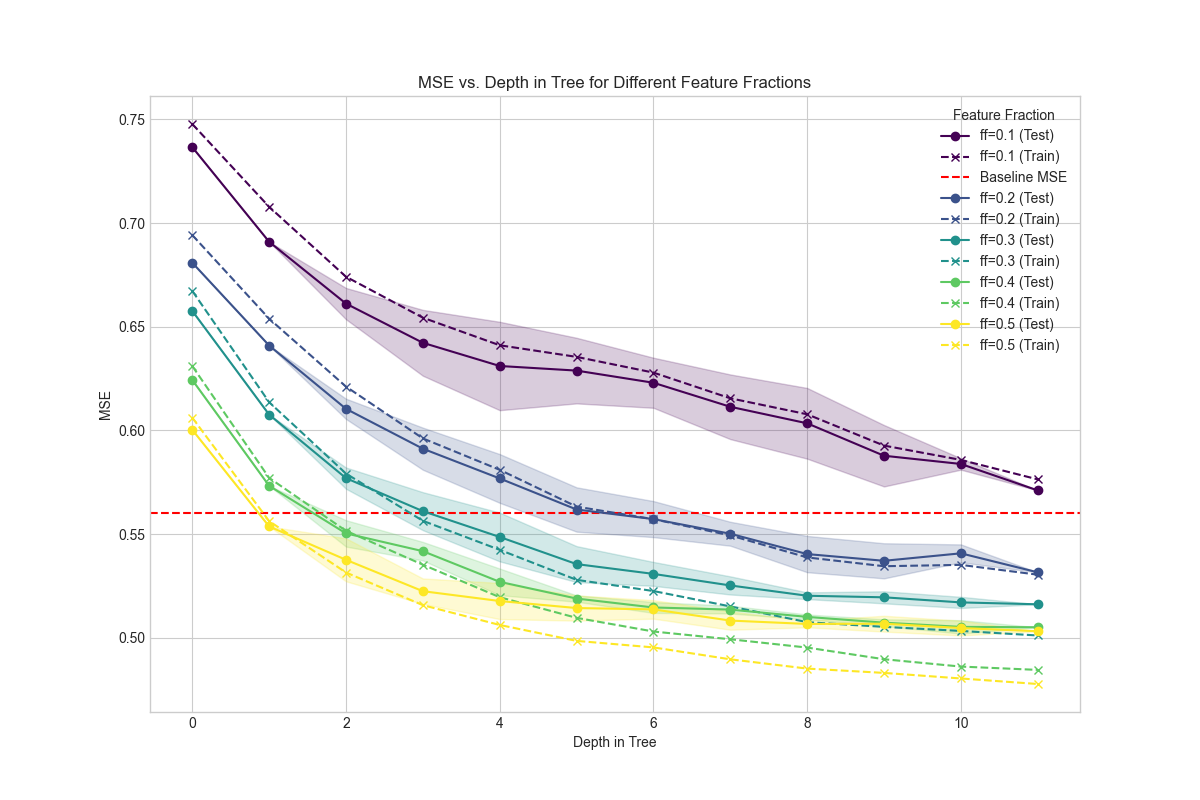}\hfill%
    \includegraphics[height=0.198\textheight]{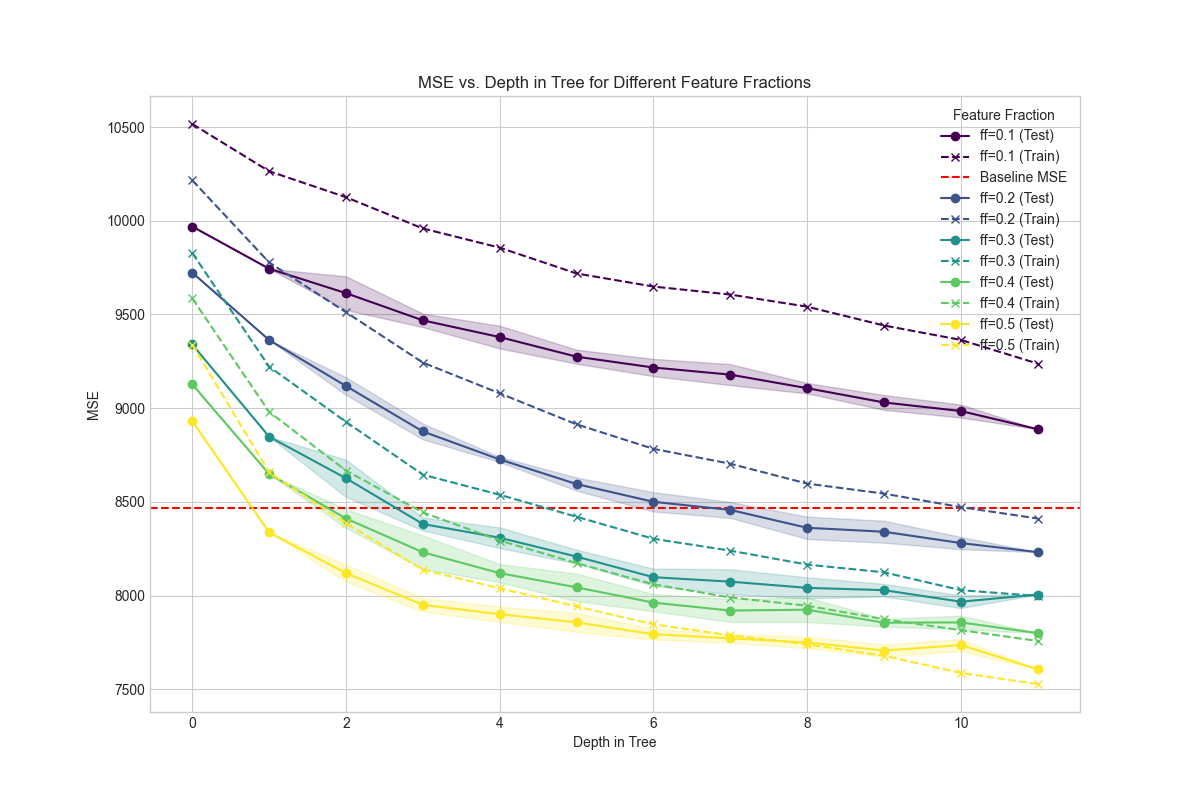}
    \vfill
    \includegraphics[height=0.198\textheight]{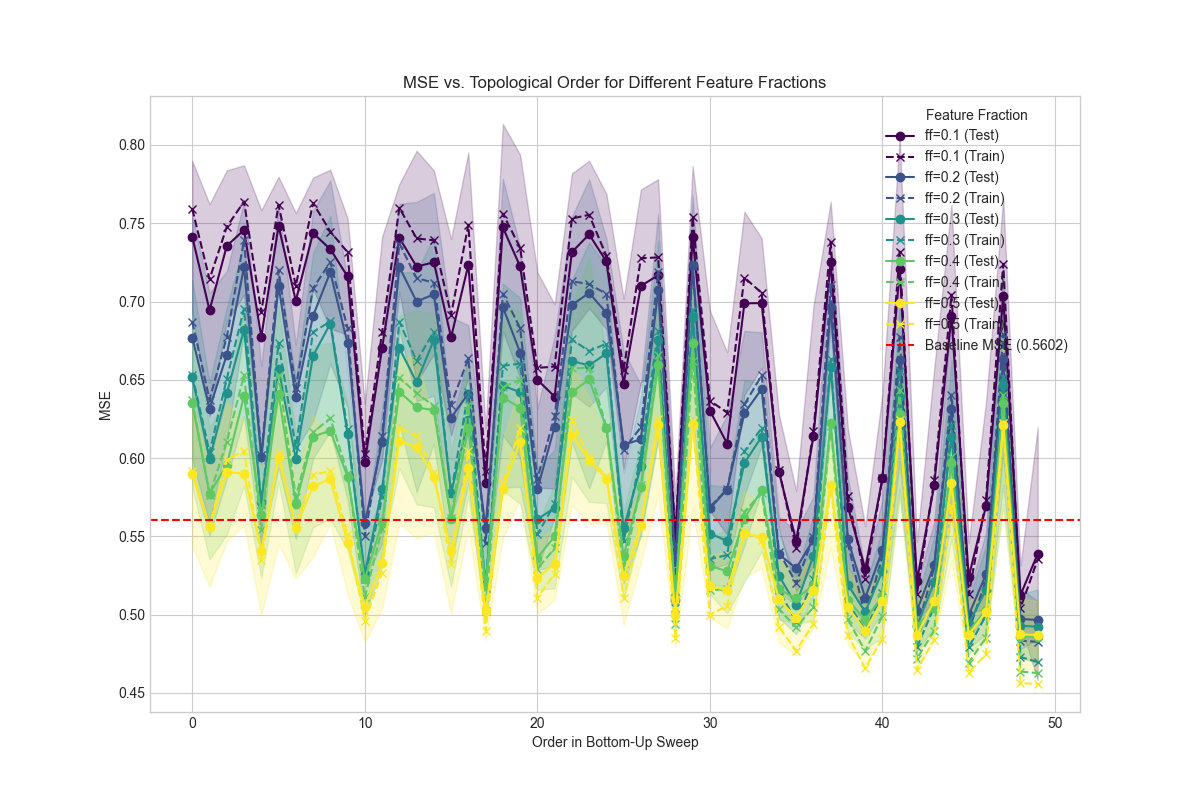}\hfill%
    \includegraphics[height=0.198\textheight]{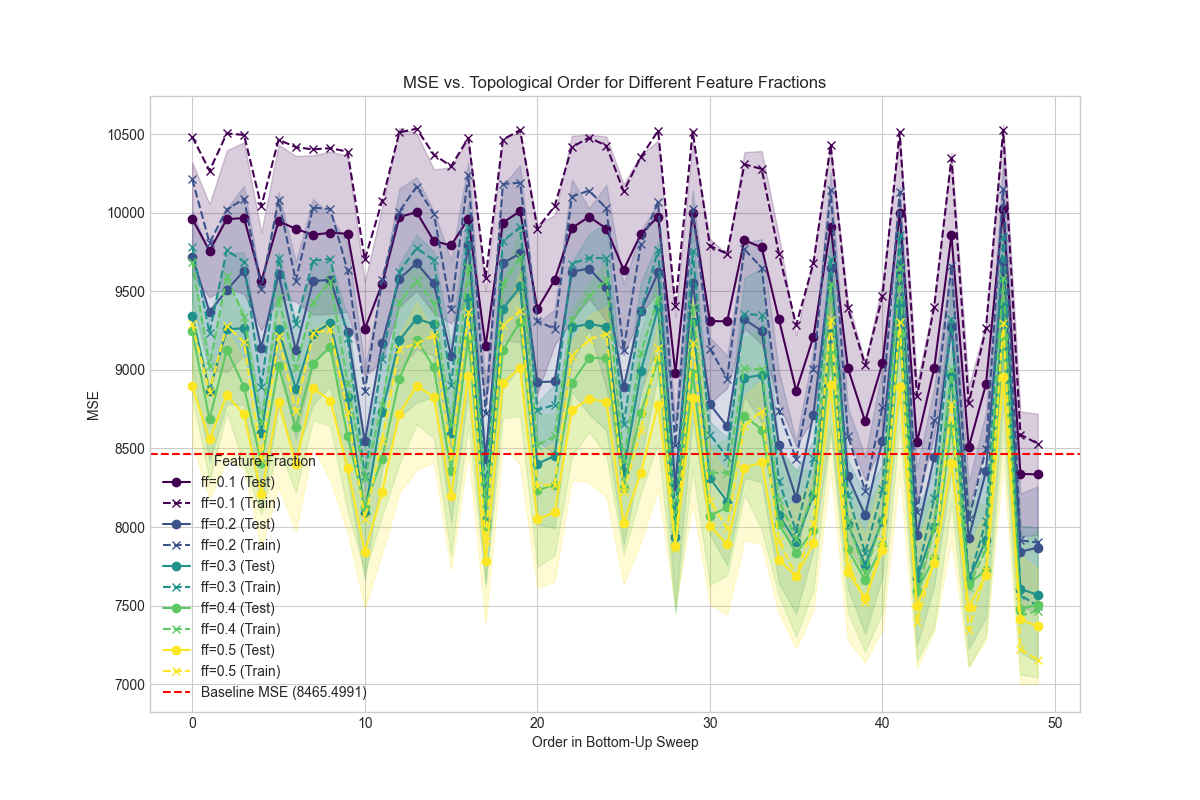}
    \vfill
    \includegraphics[height=0.198\textheight]{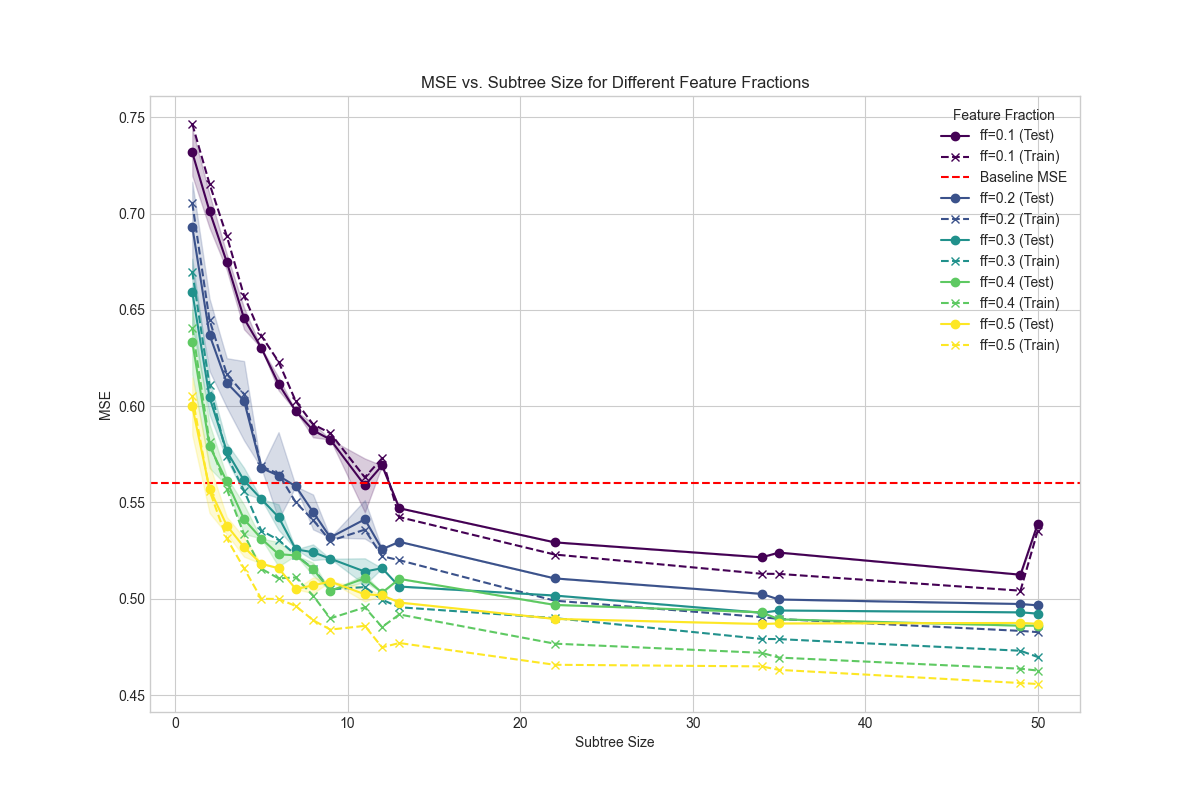}\hfill%
    \includegraphics[height=0.198\textheight]{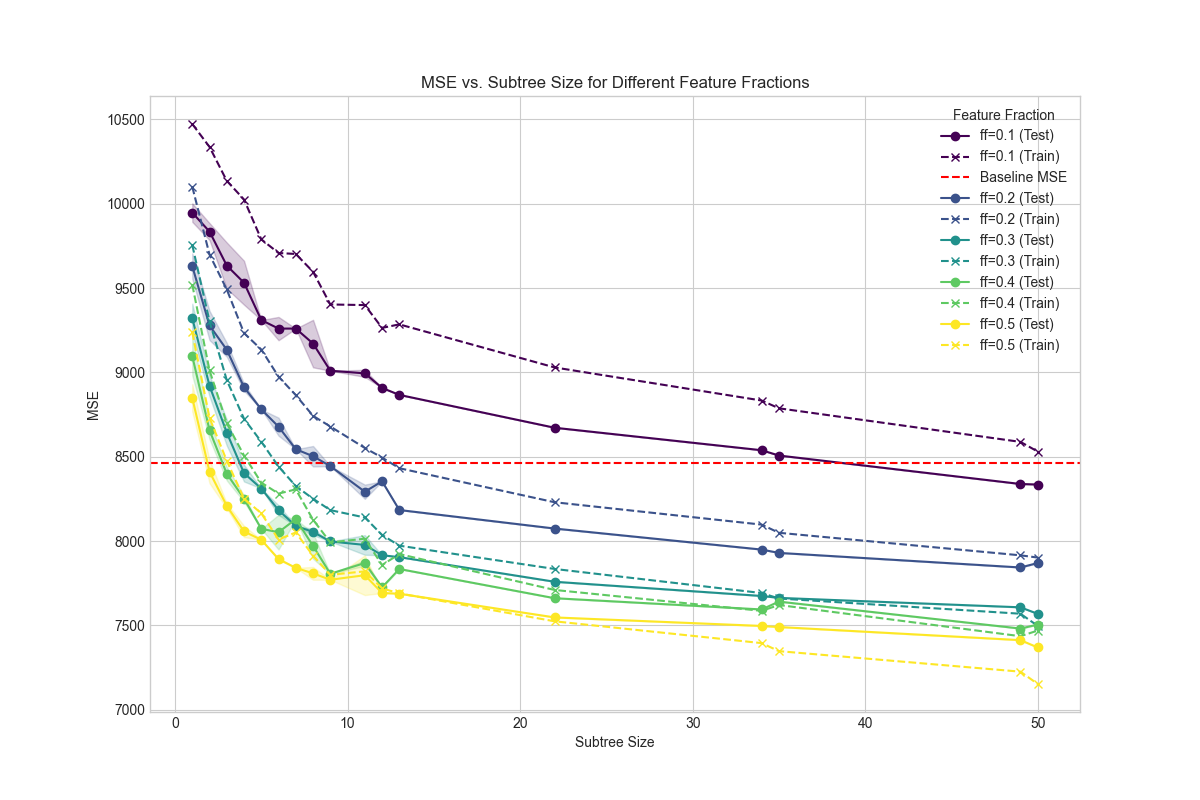}
    \caption{Experimental results for neural networks. Left column corresponds to wine quality dataset, right
        column to appliance energy usage dataset. See text for details.}
    \label{fig:nn}
\end{figure}

\clearpage

\subsection*{Acknowledgments}
The authors gratefully acknowledge support from NSF grant CCF-2217062, the NSF Graduate Research Fellowship (grant DGE-2139899), and a grant from the Simons Foundation. We used Gemini 2.5 Pro within the Windsurf Cascade environment as an aid in writing code and proving lemmas, with detailed instructions from the authors. AI tools were not used for exposition or discussion of related work. 

\bibliographystyle{plainnat}
\bibliography{bib}
\newpage

\appendix
\section{Deferred Proofs from Section \ref{sec:general_classes}} \label{sec:nonlinear_proofs}

Here, we provide complete proofs deferred from Section \ref{sec:general_classes} in the main text.

\deltamultiacc*
\begin{proof}
Let $R = Y - \hat{y}_k$ be the final residual. The algorithm's termination condition is that for $h_{\text{next}} = \arg\max_{h \in \mathcal{H}_k} |\mathbb{E}[h \cdot R]|$, we have $|\mathbb{E}[h_{\text{next}} \cdot R]| < \Delta$. Since $h_{\text{next}}$ is the function with the maximum possible correlation, this inequality must hold for all other functions $h \in \mathcal{H}_k$ as well. Therefore, $|\mathbb{E}[h \cdot R]| \le |\mathbb{E}[h_{\text{next}} \cdot R]| < \Delta$.
\end{proof}

\genmsedecomp*
\begin{proof}
We start with the identity:
\[ \mathbb{E}[(y - g)^2] = \mathbb{E}[(y - \hat{y}_k + \hat{y}_k - g)^2] = \mathbb{E}[(y - \hat{y}_k)^2] + \mathbb{E}[(\hat{y}_k - g)^2] + 2\mathbb{E}[(y - \hat{y}_k)(\hat{y}_k - g)]. \]
Rearranging gives the MSE of $\hat{y}_k$:
\[ \MSE{\hat{y}_k} = \MSE{g} - \mathbb{E}[(\hat{y}_k - g)^2] - 2\mathbb{E}[(y - \hat{y}_k)(\hat{y}_k - g)]. \]
We focus on the cross-term. We substitute the decomposition $g = g_{\text{pool}} + g_{\text{rest}}$ into the term $(\hat{y}_k - g)$:
\[ \hat{y}_k - g = \hat{y}_k - (g_{\text{pool}} + g_{\text{rest}}) = (\hat{y}_k - g_{\text{pool}}) - g_{\text{rest}}. \]
By definition, both $\hat{y}_k$ and $g_{\text{pool}}$ lie in the subspace $\text{span}(\mathcal{F}_k)$. Therefore, their difference $(\hat{y}_k - g_{\text{pool}})$ also lies in this subspace. By the properties of orthogonal projection, the residual $(y - \hat{y}_k)$ is orthogonal to every function in $\text{span}(\mathcal{F}_k)$. Thus, the inner product with the first part of the difference is zero:
\[ \mathbb{E}[(y - \hat{y}_k)(\hat{y}_k - g_{\text{pool}})] = 0. \]
The cross-term simplifies to:
\[ \mathbb{E}[(y - \hat{y}_k)(\hat{y}_k - g)] = \mathbb{E}[(y - \hat{y}_k)((\hat{y}_k - g_{\text{pool}}) - g_{\text{rest}})] = - \mathbb{E}[(y - \hat{y}_k)g_{\text{rest}}]. \]
Substituting this back into the rearranged MSE equation gives:
\[ \MSE{\hat{y}_k} = \MSE{g} - \mathbb{E}[(\hat{y}_k - g)^2] - 2(-\mathbb{E}[(y - \hat{y}_k)g_{\text{rest}}]) = \MSE{g} - \mathbb{E}[(\hat{y}_k - g)^2] + 2\mathbb{E}[(y - \hat{y}_k)g_{\text{rest}}]. \]
This completes the proof.
\end{proof}

\genmseimprovement*
\begin{proof}
The proof is a direct corollary of Lemma~\ref{lem:gen_mse_decomposition}. We apply the lemma with the benchmark predictor $g = \hat{y}_p$. By the construction of the algorithm for agent $A_k$, the parent predictor $\hat{y}_p$ is always included in the initial feature pool $\mathcal{F}_{k,0}$, and thus is in the final pool $\mathcal{F}_k$.
Therefore, $\hat{y}_p$ lies entirely within the subspace $\text{span}(\mathcal{F}_k)$. When we decompose $g=\hat{y}_p$ into $g_{\text{pool}} + g_{\text{rest}}$, we get $g_{\text{pool}} = \hat{y}_p$ and $g_{\text{rest}} = 0$.
Substituting $g=\hat{y}_p$ and $g_{\text{rest}}=0$ into the result of Lemma~\ref{lem:gen_mse_decomposition}:
\begin{align*}
\MSE{\hat{y}_k} &= \MSE{\hat{y}_p} - \mathbb{E}[(\hat{y}_k - \hat{y}_p)^2] + 2\mathbb{E}[(y - \hat{y}_k) \cdot 0] \\
&= \MSE{\hat{y}_p} - \mathbb{E}[(\hat{y}_k - \hat{y}_p)^2].
\end{align*}
This holds for any parent $p \in \text{Pa}(k)$. This result shows that the MSE is guaranteed to be non-increasing, and the reduction in error at agent $k$ relative to a parent $p$ is precisely the squared $L_2$ distance between their respective predictors.
\end{proof}

\pathbenchmark*
\begin{proof}
The proof proceeds via a blocking argument. Consider a ``block'' of $M$ contiguous agents along any path with i.i.d. hypothesis classes $\mathcal{S}_M$. Let $X = \text{MSE}(\mathcal{S}_M, L_{1,g}, B_g)$ be the random variable for the block's MSE, with expectation $\mathbb{E}[X] = \text{MSE}_{\text{bench}}(M, L_{1,g}, B_g)$.
A block is ``bad'' if $X > \mathbb{E}[X] + \gamma$. Let $p_{\text{bad}}$ be the probability of this event. Since $X$ is a non-negative random variable, we can apply Markov's inequality:
\[ p_{\text{bad}} = \mathbb{P}(X \ge \mathbb{E}[X] + \gamma) \le \frac{\mathbb{E}[X]}{\mathbb{E}[X] + \gamma}. \]
Since $\mathbb{E}[X] = \text{MSE}_{\text{bench}}(M) \le C_{\text{mse}}$ and the function $f(z)=z/(z+\gamma)$ is increasing for $z>0$, we can establish an upper bound:
\[ p_{\text{bad}} \le \frac{C_{\text{mse}}}{C_{\text{mse}} + \gamma}. \]
Let $p_{\text{upper}} = \frac{C_{\text{mse}}}{C_{\text{mse}} + \gamma}$. We have $p_{\text{bad}} \le p_{\text{upper}} < 1$.

We partition a path of length $L$ into $k = \lfloor L/M \rfloor$ disjoint blocks. The proposition fails only if all $k$ blocks are ``bad''. The probability of this is $(p_{\text{bad}})^k$.
\[ \mathbb{P}(\text{failure}) = (p_{\text{bad}})^k \le (p_{\text{upper}})^k = \left( \frac{C_{\text{mse}}}{C_{\text{mse}} + \gamma} \right)^k. \]
We want this failure probability to be at most $\delta$. Assuming $C_{\text{mse}} > 0$:
\[ \left( \frac{C_{\text{mse}}}{C_{\text{mse}} + \gamma} \right)^k \le \delta. \]
Solving for $k$ by taking logarithms:
\[ k \ln\left( \frac{C_{\text{mse}}}{C_{\text{mse}} + \gamma} \right) \le \ln(\delta) \implies k \ge \frac{\ln(1/\delta)}{\ln\left(1 + \frac{\gamma}{C_{\text{mse}}}\right)}. \]
Since we need $L \ge k \cdot M$, we require:
\[ L \ge \frac{M \ln(1/\delta)}{\ln\left(1 + \frac{\gamma}{C_{\text{mse}}}\right)}. \]
Choosing $L$ according to the proposition statement ensures the failure probability is at most $\delta$.
\end{proof}

\section{Generalization Bound Preliminaries}
\label{sec:generalization_prelim}

This section introduces the foundational concepts and key theoretical results that will be used throughout our analysis of generalization error for both linear and non-linear models.

\subsection{In-Sample and Out-of-Sample Error}

Let $\mathcal{D}$ be the underlying true data distribution over $(x,y)$ pairs, where $x \in \R^d$ and $y \in \R$. For any predictor $f: \R^d \to \R$, its true (out-of-sample) squared error, or risk, is:
\[ R(f) = \E_{(x,y) \sim \mathcal{D}}[(f(x) - y)^2] \]
In practice, models are trained on a finite dataset. Let $D_m = \{(x^{(j)}, y^{(j)})\}_{j=1}^m$ be a training set of $m$ i.i.d. samples drawn from $\mathcal{D}$. The empirical (in-sample) risk for a predictor $f$ is:
\[ \hat{R}_{D_m}(f) = \frac{1}{m} \sum_{j=1}^m (f(x^{(j)}) - y^{(j)})^2 \]
The central goal of generalization theory is to bound the difference $|R(f) - \hat{R}_{D_m}(f)|$, known as the generalization gap.

\subsection{Complexity Measures and Concentration}

The generalization gap is typically controlled by the complexity of the class of functions from which the predictor is chosen. Rademacher complexity is a standard measure of this complexity.

\begin{definition}[Rademacher Complexity]
Let $\mathcal{G}$ be a class of functions mapping from $\mathcal{Z}$ to $\R$, and let $Z = (z_1, \ldots, z_m)$ be a fixed set of $m$ points in $\mathcal{Z}$. The empirical Rademacher complexity of $\mathcal{G}$ with respect to the set $Z$ is:
\[ \hat{\mathcal{R}}_Z(\mathcal{G}) = \E_{\sigma_1, \ldots, \sigma_m} \left[ \sup_{g \in \mathcal{G}} \frac{1}{m} \sum_{i=1}^m \sigma_i g(z_i) \right], \]
where $\sigma_i$ are independent random variables drawn from the Rademacher distribution ($\mathbb{P}(\sigma_i=1) = \mathbb{P}(\sigma_i=-1) = 1/2$). The (population) Rademacher complexity is the expectation over all datasets of size $m$: $\mathcal{R}_m(\mathcal{G}) = \E_{Z \sim \mathcal{D}^m}[\hat{\mathcal{R}}_Z(\mathcal{G})]$.
\end{definition}

The following lemma shows how the Rademacher complexity of a function class changes when its outputs are passed through a Lipschitz function.

\begin{lemma}[Talagrand's Contraction Principle]
\label{lem:talagrand_contraction}
Let $\mathcal{G}$ be a class of functions mapping from $\mathcal{Z}$ to $\mathbb{R}$. Let $\phi: \mathbb{R} \to \mathbb{R}$ be a function that is $L_\phi$-Lipschitz, i.e., $|\phi(a) - \phi(b)| \le L_\phi |a - b|$ for all $a, b \in \mathbb{R}$. Let $\phi \circ \mathcal{G} = \{ z \mapsto \phi(g(z)) \mid g \in \mathcal{G} \}$ be the class of functions formed by composing $\phi$ with each function in $\mathcal{G}$. Then the Rademacher complexity of the composed class is bounded by:
\[ \mathcal{R}_m(\phi \circ \mathcal{G}) \le L_\phi \mathcal{R}_m(\mathcal{G}). \]
\end{lemma}

Additionally, our analysis relies on concentration inequalities, which bound the deviation of sample averages from their expectations. McDiarmid's inequality is a general-purpose tool for functions of independent variables that have bounded differences.

\begin{lemma}[McDiarmid's Inequality]
\label{lem:mcdiarmid}
Let $Z_1, \ldots, Z_m$ be independent random variables, and let $g: \mathcal{Z}^m \to \mathbb{R}$ be a function of these variables. Suppose that for each $i \in \{1, \ldots, m\}$, the function satisfies the bounded differences property:
\[ \sup_{z_1, \ldots, z_m, z'_i} |g(z_1, \ldots, z_i, \ldots, z_m) - g(z_1, \ldots, z'_i, \ldots, z_m)| \le c_i. \]
Then for any $t > 0$, the following holds:
\[ \mathbb{P}\left( g(Z_1, \ldots, Z_m) - \mathbb{E}[g(Z_1, \ldots, Z_m)] \ge t \right) \le \exp\left( \frac{-2t^2}{\sum_{i=1}^m c_i^2} \right). \]
\end{lemma}

The final tool is a matrix-valued concentration inequality, which is crucial for analyzing sample covariance matrices.

\begin{lemma}[Matrix Bernstein Inequality, \cite{tropp2015introductionmatrixconcentrationinequalities}, Theorem 1.6.2]
\label{lem:matrix_bernstein}
Let $Y_1, \ldots, Y_m$ be independent, zero-mean, $d \times d$ random matrices. Assume there exists a constant $L > 0$ such that $\|Y_i\|_{op} \le L$ almost surely for all $i$. Let $\sigma^2$ be the norm of the total variance, $\sigma^2 = \left\| \sum_{i=1}^m \mathbb{E}[Y_i^2] \right\|_{op}$. Then, for any $t \ge 0$:
\[ \mathbb{P}\left( \left\| \sum_{i=1}^m Y_i \right\|_{op} \ge t \right) \le 2d \cdot \exp\left( \frac{-t^2/2}{\sigma^2 + Lt/3} \right). \]
\end{lemma}

\section{Generalization Bounds for Linear Predictors}
\label{sec:generalization-linear}

\subsection{Linearity of Sequential Predictors}

We begin by establishing that each predictor $\hat{y}_k$ generated by the sequential learning process is a linear function of the original $d$-dimensional feature vector $x$. This will then allow us to apply standard generalization bounds for linear models.

\begin{lemma}[Linearity of Predictors in a DAG]
\label{lem:linearity_of_predictors}
For each agent $A_k$ in the DAG, its prediction $\hat{y}_k$ can be expressed as a linear function of the full feature vector $x \in \mathbb{R}^d$. That is, for each $k$, there exists a vector $\beta_k \in \mathbb{R}^d$ such that $\hat{y}_k = \beta_k^T x$.
\end{lemma}

\begin{proof}
We prove this by induction over the agents, following a topological sort of the DAG.

\textbf{Base Case (Root Agents):}
Let $A_k$ be a root agent (an agent with no parents). Its predictor is $f_k(x_{S_k}) = w_k^T x_{S_k}$. The prediction is $\hat{y}_k = w_k^T x_{S_k}$.
We can define a vector $\beta_k \in \mathbb{R}^d$ such that its components corresponding to features in $S_k$ are given by $w_k$, and all other components are zero. Then, $\hat{y}_k = \beta_k^T x$. The statement holds for all root agents.

\textbf{Inductive Step:}
Assume that for all parents $A_p$ of agent $A_k$ (i.e., for all $p \in \text{Pa}(k)$), the prediction $\hat{y}_p$ is a linear function of $x$. That is, for each $p \in \text{Pa}(k)$, there exists a vector $\beta_p \in \mathbb{R}^d$ such that $\hat{y}_p = \beta_p^T x$.

Now consider agent $A_k$. Its predictor is $f_k(x_{S_k}, \{\hat{y}_p\}_{p \in \text{Pa}(k)}) = w_k^T x_{S_k} + \sum_{p \in \text{Pa}(k)} v_{k,p} \hat{y}_p$.
The prediction is $\hat{y}_k = w_k^T x_{S_k} + \sum_{p \in \text{Pa}(k)} v_{k,p} \hat{y}_p$.
Substituting the inductive hypothesis for each $\hat{y}_p$:
\[ \hat{y}_k = w_k^T x_{S_k} + \sum_{p \in \text{Pa}(k)} v_{k,p} (\beta_p^T x) \]
Let $w'_k \in \mathbb{R}^d$ be a vector whose components corresponding to features in $S_k$ are given by $w_k$, and other components are zero. So, $w_k^T x_{S_k} = (w'_k)^T x$.
Then,
\[ \hat{y}_k = (w'_k)^T x + \left( \sum_{p \in \text{Pa}(k)} v_{k,p} \beta_p \right)^T x = \left(w'_k + \sum_{p \in \text{Pa}(k)} v_{k,p} \beta_p \right)^T x \]
Let $\beta_k = w'_k + \sum_{p \in \text{Pa}(k)} v_{k,p} \beta_p$. Since $w'_k$ and all $\beta_p$ are vectors in $\mathbb{R}^d$, and all $v_{k,p}$ are scalars, $\beta_k$ is also a vector in $\mathbb{R}^d$.
Thus, $\hat{y}_k = \beta_k^T x$, which shows that $\hat{y}_k$ is a linear function of $x$.

By induction, the statement holds for all agents in the DAG.
\end{proof}

\subsection{Standard Generalization Bound for Linear Predictors}

To apply generalization bounds, we typically make some assumptions on the data distribution and the complexity of the learned models.
Assume:
\begin{enumerate}
    \item The features are bounded: $\|x\|_{\infty} \le X_{\max}$ for all $x \sim \mathcal{D}_x$.
    \item The labels are bounded: $|y| \le Y_{\max}$ for all $y \sim \mathcal{D}_y$.
    \item The linear predictors $\hat{y}_k = \beta_k^T x$ considered are such that their coefficient vectors have a bounded $L_1$ norm: $\|\beta_k\|_1 \le \Lambda_1$. This aligns with the conditions in Theorem~\ref{thm:small_improvement_path} and Theorem~\ref{thm:overall_guarantee_random_model} which compare against predictors with bounded $L_1$ norms.
\end{enumerate}

Under these assumptions, for any predictor $f(x) = \beta^T x$ with $\|\beta\|_1 \le \Lambda_1$, we have $|f(x)| \le \|\beta\|_1 \|x\|_{\infty} \le \Lambda_1 X_{\max}$.
The squared loss $(f(x)-y)^2$ is then bounded by $M_L^2$, where $M_L = \Lambda_1 X_{\max} + Y_{\max}$.

The following is a standard generalization bound for linear predictors with squared loss, derived from uniform convergence principles:

\begin{theorem}[Generalization Bound for $L_1$-Bounded Linear Regression]
\label{thm:gen_bound_linear}
Let $\mathcal{H}_{\Lambda_1} = \{ x \mapsto \beta^T x \mid \|\beta\|_1 \le \Lambda_1 \}$ be the class of linear predictors with $L_1$-norm bounded by $\Lambda_1$.
Assume $\|x\|_{\infty} \le X_{\max}$ and $|y| \le Y_{\max}$. Let $D_m$ be a training set of $m$ i.i.d. samples from $\mathcal{D}$.
Then, for any $\delta \in (0,1)$, with probability at least $1-\delta$ over the random draw of $D_m$, for all $f \in \mathcal{H}_{\Lambda_1}$ (and thus for each learned predictor $\hat{y}_k$ provided $\|\beta_k\|_1 \le \Lambda_1$):
\[ |R(f) - \hat{R}_{D_m}(f)| \le C \cdot \Lambda_1 X_{\max} \sqrt{\frac{\log d}{m}} + (\Lambda_1 X_{\max} + Y_{\max})^2 \sqrt{\frac{\log(C'/\delta)}{2m}} \]
where $C, C'$ are universal constants. The term $\log d$ arises from the complexity of $L_1$-constrained linear functions in $d$ dimensions.
\end{theorem}

\begin{proof}
The (standard) proof proceeds in three steps. First, we use a standard symmetrization argument to relate the generalization error to the Rademacher complexity of the associated loss class. Second, we use Talagrand's contraction principle to bound the complexity of the loss class in terms of the complexity of the hypothesis class $\mathcal{H}_{\Lambda_1}$. Third, we bound the complexity of $\mathcal{H}_{\Lambda_1}$.

\textbf{Step 1: Symmetrization and High-Probability Bound.}
Let $\mathcal{L} = \{ (x,y) \mapsto (f(x)-y)^2 \mid f \in \mathcal{H}_{\Lambda_1} \}$ be the loss class. The proof first relates the expected generalization error to the Rademacher complexity of this loss class, and then uses a concentration inequality to obtain a high-probability bound.

A standard result based on symmetrization (introducing a ``ghost sample'') states that the expected supremum of the generalization error is bounded by twice the Rademacher complexity of the loss class:
\[ \E_{D_m} \left[ \sup_{f \in \mathcal{H}_{\Lambda_1}} |R(f) - \hat{R}_{D_m}(f)| \right] \le 2 \mathcal{R}_m(\mathcal{L}). \]
To obtain a high-probability bound, we use McDiarmid's inequality. For any $f \in \mathcal{H}_{\Lambda_1}$, the loss is bounded by $(|f(x)| + |y|)^2 \le (\Lambda_1 X_{\max} + Y_{\max})^2$. Let $M_L = \Lambda_1 X_{\max} + Y_{\max}$. The maximum value of the loss is $M_L^2$. Changing a single data point in the training set $D_m$ can change the value of $\sup_{f} |R(f) - \hat{R}_{D_m}(f)|$ by at most $M_L^2/m$.

Applying McDiarmid's inequality to the function $\Phi(D_m) = \sup_{f} |R(f) - \hat{R}_{D_m}(f)|$, and combining with the expectation bound from symmetrization, we find that for any $\delta \in (0,1)$, with probability at least $1-\delta$:
\[ \sup_{f \in \mathcal{H}_{\Lambda_1}} |R(f) - \hat{R}_{D_m}(f)| \le 2 \mathcal{R}_m(\mathcal{L}) + M_L^2 \sqrt{\frac{\ln(2/\delta)}{2m}}. \]

\textbf{Step 2: Contraction Principle.}
We bound the complexity of the loss class $\mathcal{R}_m(\mathcal{L})$ using the complexity of the hypothesis class $\mathcal{H}_{\Lambda_1}$. The loss function can be seen as a composition $\ell(f(x), y) = \phi(f(x)-y)$, where $\phi(z) = z^2$. The function $\phi$ is Lipschitz on the interval $[-M_L, M_L]$, with a Lipschitz constant $L_\phi = \sup_{z \in [-M_L, M_L]} |\phi'(z)| = \sup_{z \in [-M_L, M_L]} |2z| = 2M_L$.

By Talagrand's contraction principle, the Rademacher complexity of the composed class is bounded by the Lipschitz constant times the complexity of the pre-composed class:
\[ \mathcal{R}_m(\mathcal{L}) \le L_\phi \mathcal{R}_m(\{ (x,y) \mapsto f(x)-y \mid f \in \mathcal{H}_{\Lambda_1} \}) = 2M_L \mathcal{R}_m(\mathcal{H}_{\Lambda_1}). \]
Note that subtracting a fixed value $y_i$ for each sample does not change the Rademacher complexity.

\textbf{Step 3: Bounding the Hypothesis Class Complexity.}
The final step is to bound the Rademacher complexity of the hypothesis class $\mathcal{H}_{\Lambda_1}$. By definition and the property of dual norms $(\norm{\cdot}_1, \norm{\cdot}_\infty)$:
\begin{align*}
\mathcal{R}_m(\mathcal{H}_{\Lambda_1}) &= \E_{D_m, \sigma} \left[ \sup_{\norm{\beta}_1 \le \Lambda_1} \frac{1}{m} \sum_{i=1}^m \sigma_i \beta^T x_i \right] \\
&= \E_{D_m, \sigma} \left[ \sup_{\norm{\beta}_1 \le \Lambda_1} \beta^T \left(\frac{1}{m}\sum_{i=1}^m \sigma_i x_i\right) \right] \\
&= \Lambda_1 \E_{D_m, \sigma} \left[ \norm{\frac{1}{m}\sum_{i=1}^m \sigma_i x_i}_{\infty} \right].
\end{align*}
Let $v = \frac{1}{m}\sum_{i=1}^m \sigma_i x_i$. For each coordinate $j \in \{1, \dots, d\}$, $v_j = \frac{1}{m}\sum_{i=1}^m \sigma_i x_{i,j}$ is an average of independent random variables with mean 0, bounded in $[-X_{\max}, X_{\max}]$. By Hoeffding's inequality, $\mathbb{P}(|v_j| \ge t) \le 2\exp(-m t^2 / (2X_{\max}^2))$. Using a union bound over all $d$ dimensions and a standard bounding argument for the expectation, we get:
\[ \mathcal{R}_m(\mathcal{H}_{\Lambda_1}) \le \Lambda_1 X_{\max} \sqrt{\frac{2\ln(2d)}{m}}. \]

\textbf{Combining the Bounds.}
Putting the steps together, we have that with probability at least $1-\delta$:
\begin{align*}
\sup_{f \in \mathcal{H}_{\Lambda_1}} |R(f) - \hat{R}_{D_m}(f)| &\le 2 \mathcal{R}_m(\mathcal{L}) + M_L^2 \sqrt{\frac{\ln(2/\delta)}{2m}} \\
&\le 2 (2M_L \mathcal{R}_m(\mathcal{H}_{\Lambda_1})) + M_L^2 \sqrt{\frac{\ln(2/\delta)}{2m}} \\
&\le 4M_L \left(\Lambda_1 X_{\max} \sqrt{\frac{2\ln(2d)}{m}}\right) + M_L^2 \sqrt{\frac{\ln(2/\delta)}{2m}} \\
&= 4(\Lambda_1 X_{\max} + Y_{\max}) \Lambda_1 X_{\max} \sqrt{\frac{2\ln(2d)}{m}} + (\Lambda_1 X_{\max} + Y_{\max})^2 \sqrt{\frac{\ln(2/\delta)}{2m}}.
\end{align*}
The second term matches the theorem by setting $C' = 2$. For the first term, we can bound the logarithmic factor for $d \ge 2$:
\[ \sqrt{2\ln(2d)} = \sqrt{2(\ln 2 + \ln d)} = \sqrt{2\ln d \left(1 + \frac{\ln 2}{\ln d}\right)} \le \sqrt{2\ln d (1 + 1)} = 2\sqrt{\ln d}. \]
The complexity term is therefore bounded by:
\[ 4(\Lambda_1 X_{\max} + Y_{\max}) \Lambda_1 X_{\max} \frac{2\sqrt{\ln d}}{\sqrt{m}} = 8(\Lambda_1 X_{\max} + Y_{\max}) \Lambda_1 X_{\max} \sqrt{\frac{\ln d}{m}}. \]
This matches the theorem's format by defining $C = 8(\Lambda_1 X_{\max} + Y_{\max})$.

    \end{proof}

This theorem states that the difference between the true error and the empirical error (the generalization gap) for any linear predictor learned by our agents (assuming its $L_1$ norm is appropriately bounded) will be small if the number of training samples $m$ is sufficiently large relative to the dimension $d$ and the desired confidence.

\subsection{Matrix Concentration for Sample Covariance} 

\begin{theorem}[Matrix Concentration for Sample Covariance]
\label{thm:matrix_concentration}
Let $x_1, \ldots, x_m$ be i.i.d. samples from a distribution with $\|x_i\|_2 \le R_X$ almost surely. Let $\Sigma = \mathbb{E}[xx^T]$ be the population covariance matrix and $\hat{\Sigma} = \frac{1}{m}\sum_{i=1}^m x_i x_i^T$ be the empirical covariance matrix. Then for any $\delta \in (0,1)$, with probability at least $1-\delta$:
\[ \|\hat{\Sigma} - \Sigma\|_{op} \le C \cdot R_X^2 \sqrt{\frac{\log d + \log(1/\delta)}{m}} \]
for some universal constant $C > 0$, where $\|\cdot\|_{op}$ denotes the operator (spectral) norm.
\end{theorem}


\begin{proof}
We apply Lemma~\ref{lem:matrix_bernstein} (Matrix Bernstein Inequality), defining the random matrices as follows: 

For $i=1, \ldots, m$, let $Y_i = x_i x_i^T - \Sigma$. The random vectors $x_i$ are i.i.d., so the random matrices $Y_i$ are also independent and identically distributed. The mean of each $Y_i$ is zero:
\[ \mathbb{E}[Y_i] = \mathbb{E}[x_i x_i^T - \Sigma] = \mathbb{E}[x_i x_i^T] - \Sigma = \Sigma - \Sigma = \mathbf{0} \]
The quantity we wish to bound is the operator norm of the difference between the empirical and true covariance matrices, which we can write as:
\[ \|\hat{\Sigma} - \Sigma\|_{op} = \left\| \frac{1}{m} \sum_{i=1}^m x_i x_i^T - \Sigma \right\|_{op} = \left\| \frac{1}{m} \sum_{i=1}^m (x_i x_i^T - \Sigma) \right\|_{op} = \frac{1}{m} \left\| \sum_{i=1}^m Y_i \right\|_{op}. \]

Next, we establish bounds for the parameters $L$ and $\sigma^2$ for our matrices $Y_i$.

To bound the operator norm parameter $L$, note that by the triangle inequality:
\[ \|Y_i\|_{op} = \|x_i x_i^T - \Sigma\|_{op} \le \|x_i x_i^T\|_{op} + \|\Sigma\|_{op} \]
The first term is bounded by our assumption on $x_i$: $\|x_i x_i^T\|_{op} = \|x_i\|_2^2 \le R_X^2$. For the second term, we use Jensen's inequality for the convex operator norm function:
\[ \|\Sigma\|_{op} = \|\mathbb{E}[x x^T]\|_{op} \le \mathbb{E}[\|x x^T\|_{op}] = \mathbb{E}[\|x\|_2^2] \le R_X^2 \]
Therefore, we can set our parameter $L$ as:
\[ L \le R_X^2 + R_X^2 = 2R_X^2. \]

To bound the variance parameter $\sigma^2$, first note that since the $Y_i$ are i.i.d., $\sigma^2 = \left\| \sum_{i=1}^m \mathbb{E}[Y_i^2] \right\|_{op} = m \left\| \mathbb{E}[Y_1^2] \right\|_{op}$. We analyze $\mathbb{E}[Y_1^2]$:
\begin{align*}
\mathbb{E}[Y_1^2] &= \mathbb{E}[(x_1 x_1^T - \Sigma)^2] \\
&= \mathbb{E}[(x_1 x_1^T)^2] - \mathbb{E}[x_1 x_1^T]\Sigma - \Sigma\mathbb{E}[x_1 x_1^T] + \Sigma^2 \\
&= \mathbb{E}[\|x_1\|_2^2 x_1 x_1^T] - \Sigma^2
\end{align*}
where we used $(x_1 x_1^T)^2 = x_1(x_1^T x_1)x_1^T = \|x_1\|_2^2 (x_1 x_1^T)$. To bound its norm, we have:
\[ \|\mathbb{E}[Y_1^2]\|_{op} = \|\mathbb{E}[\|x_1\|_2^2 x_1 x_1^T] - \Sigma^2\|_{op} \le \|\mathbb{E}[\|x_1\|_2^2 x_1 x_1^T]\|_{op} \]
This last inequality follows because $\mathbb{E}[\|x_1\|_2^2 x_1 x_1^T] \succeq \Sigma^2$, and for positive semidefinite matrices $A \succeq B \succeq 0$, we have $\|A\|_{op} \ge \|B\|_{op}$. Using Jensen's inequality again:
\[ \|\mathbb{E}[\|x_1\|_2^2 x_1 x_1^T]\|_{op} \le \mathbb{E}[\|\|x_1\|_2^2 x_1 x_1^T\|_{op}] = \mathbb{E}[\|x_1\|_2^4] \le R_X^4 \]
Thus, the bound on the total variance parameter is:
\[ \sigma^2 = m \|\mathbb{E}[Y_1^2]\|_{op} \le m R_X^4. \]

Finally, to apply the Matrix Bernstein inequality, let $\epsilon = \|\hat{\Sigma} - \Sigma\|_{op}$. We want the probability of this event being large to be bounded by $\delta$. This holds if
\[ \mathbb{P}\left(\left\|\sum_{i=1}^m Y_i\right\|_{op} \ge m\epsilon \right) \le 2d \cdot \exp\left( \frac{-(m\epsilon)^2/2}{\sigma^2 + L(m\epsilon)/3} \right)= \delta. \]

Solving for $\epsilon$, we first rearrange the equality and then substitute our bounds $L \le 2R_X^2$ and $\sigma^2/m \le R_X^4$:
\begin{align*}
    \frac{2d}{\delta} &= \exp\left( \frac{m^2\epsilon^2/2}{\sigma^2 + Lm\epsilon/3}\right) \\
    \log\left(\frac{2d}{\delta}\right) &= \frac{m^2\epsilon^2/2}{\sigma^2 + Lm\epsilon/3} \\
    \varepsilon^2 &= \frac{2\log(2d/\delta)}{m} \left( \frac{\sigma^2}{m} + \frac{L\epsilon}{3} \right) \\
    &\le \frac{2\log(2d/\delta)}{m} \left( R_X^4 + \frac{2R_X^2\epsilon}{3} \right).
\end{align*}

An inequality of the form $\epsilon^2 \le A + B\epsilon$ for $A, B > 0$ implies that $\epsilon \le \sqrt{A} + B$. Applying this simplification yields:
\[ \epsilon \le \sqrt{\frac{2R_X^4 \log(2d/\delta)}{m}} + \frac{4R_X^2 \log(2d/\delta)}{3m} =  \sqrt{2} R_X^2 \sqrt{\frac{\log(2d/\delta)}{m}} + \frac{4R_X^2}{3} \frac{\log(2d/\delta)}{m}. \]
Let $u = \sqrt{\frac{\log d + \log(2/\delta)}{m}}$. The bound is of the form $c_1 R_X^2 u + c_2 R_X^2 u^2$. To unify these two terms by finding a single constant $C$ such that $c_1 R_X^2 u + c_2 R_X^2 u^2 \le C \cdot R_X^2 u$, we write:  
\[ \epsilon \le \left( \sqrt{2} + \frac{4}{3} \sqrt{\frac{\log(2d/\delta)}{m}} \right) R_X^2 \sqrt{\frac{\log(2d/\delta)}{m}} \]
If $m \ge \log(2d/\delta)$, then the term $\sqrt{\frac{\log(2d/\delta)}{m}} \le 1$, and we can bound the expression by $(\sqrt{2} + 4/3) R_X^2 \sqrt{\frac{\log(2d/\delta)}{m}}$. If $m < \log(2d/\delta)$, the bound is trivially satisfied by choosing a large enough $C$, since we know $\|\hat{\Sigma} - \Sigma\|_{op} \le \| \hat{\Sigma} \|_{op} + \| \Sigma \|_{op} \le 2R_X^2$, and the right-hand side of the theorem statement, $C R_X^2 \sqrt{\frac{\log(2d/\delta)}{m}}$, becomes larger than $2R_X^2$ for $C > 2$.

Therefore, we can always find a universal constant $C$ (e.g., $C=4$ is sufficient) that makes the following inequality hold for all $m, d, \delta$:
\[ \|\hat{\Sigma} - \Sigma\|_{op} \le C \cdot R_X^2 \sqrt{\frac{\log d + \log(2/\delta)}{m}} \]
\end{proof}

\begin{corollary}[Empirical Eigenvalue Bound]
\label{cor:empirical_eigenvalue_bound}
Under the conditions of Theorem~\ref{thm:matrix_concentration}, if the true covariance matrix $\Sigma$ has minimum eigenvalue $\lambda_{\min}(\Sigma) > 0$, and the number of samples $m$ satisfies
\[ m > \frac{4 C^2 R_X^4 (\log d + \log(1/\delta))}{\lambda_{\min}(\Sigma)^2} \]
then with probability at least $1-\delta$, the minimum eigenvalue of the empirical covariance matrix $\hat{\Sigma}$ is bounded as:
\[ \lambda_{\min}(\hat{\Sigma}) > \frac{\lambda_{\min}(\Sigma)}{2}. \]
\end{corollary}
\begin{proof}
By Weyl's inequality, the eigenvalues of two symmetric matrices $A$ and $B$ are related by $|\lambda_i(A) - \lambda_i(B)| \le \|A-B\|_{op}$. Applying this to our case:
\[ |\lambda_{\min}(\hat{\Sigma}) - \lambda_{\min}(\Sigma)| \le \|\hat{\Sigma} - \Sigma\|_{op} \]
This implies $\lambda_{\min}(\hat{\Sigma}) \ge \lambda_{\min}(\Sigma) - \|\hat{\Sigma} - \Sigma\|_{op}$.
From Theorem~\ref{thm:matrix_concentration}, with probability at least $1-\delta$, we have $\|\hat{\Sigma} - \Sigma\|_{op} \le C \cdot R_X^2 \sqrt{\frac{d + \log(1/\delta)}{m}}$.
If we choose $m$ such that the right-hand side is less than $\lambda_{\min}(\Sigma)/2$, the result follows. The condition on $m$ in the corollary statement ensures this.
\[ C \cdot R_X^2 \sqrt{\frac{\log d + \log(1/\delta)}{m}} < \frac{\lambda_{\min}(\Sigma)}{2} \implies m > \frac{4 C^2 R_X^4 (\log d + \log(1/\delta))}{\lambda_{\min}(\Sigma)^2}. \]
\end{proof}

\subsection{Bounding Predictor Norms via Covariance}

A key challenge in applying the generalization bound in Theorem~\ref{thm:gen_bound_linear} is to establish a bound $\Lambda_1$ on the $L_1$ norm of the learned coefficient vectors $\beta_k$. We can derive such a bound by making a standard assumption on the data distribution's covariance matrix, without altering the unregularized least-squares procedure at each step. This preserves the multiaccuracy properties central to our main results.

The assumption that the covariance matrix $\Sigma = \mathbb{E}[xx^T]$ has a positive minimum eigenvalue $\lambda_{\min}(\Sigma) > 0$ is reasonable in many practical settings where features are not perfectly collinear and the data distribution is well-conditioned. This assumption is weaker than requiring feature independence and is commonly made in theoretical analyses of linear regression.

We first establish these bounds in the population setting, where predictors are learned by minimizing the true risk $R(f)$.

\begin{lemma}[Predictor Norm Bound]
\label{lem:predictor_norm_bound}
For any agent $A_k$, its predictor $\hat{y}_k$ is the orthogonal projection of the true label $y$ onto the linear subspace spanned by its input features. As a result, the $L_2$ norm of the predictor is bounded by the $L_2$ norm of the label:
\[ \|\hat{y}_k\|_{L_2} \le \|y\|_{L_2} \]
where the $L_2$ norm is defined as $\|f\|_{L_2} = \sqrt{\mathbb{E}[f(x)^2]}$.
\end{lemma}

\begin{proof}
At step $k$, the predictor $\hat{y}_k$ is the solution to the least-squares problem, minimizing $\mathbb{E}[(f - y)^2]$ over all functions $f$ in the linear span of agent $A_k$'s inputs (i.e., $x_{S_k}$ and $\hat{y}_{k-1}$). The solution to this problem is the orthogonal projection of $y$ onto this subspace. A fundamental property of orthogonal projections in a Hilbert space (like the space of random variables with finite second moments) is that they are non-expansive. Therefore, the norm of the projection $\hat{y}_k$ cannot exceed the norm of the original vector $y$.
\end{proof}

\begin{lemma}[Coefficient Vector Norm Bound]
\label{lem:coeff_norm_bound}
Let $\hat{y}_k = \beta_k^T x$ be the linear representation of the $k$-th predictor. Let $\Sigma = \mathbb{E}[xx^T]$ be the covariance matrix of the full feature vector $x$, and assume its minimum eigenvalue $\lambda_{\min}(\Sigma)$ is strictly positive. Then, the $L_2$ norm of the coefficient vector $\beta_k$ is bounded:
\[ \|\beta_k\|_2 \le \frac{\|y\|_{L_2}}{\sqrt{\lambda_{\min}(\Sigma)}} \]
\end{lemma}

\begin{proof}
The squared $L_2$ norm of the predictor function is:
\[ \|\hat{y}_k\|_{L_2}^2 = \mathbb{E}[(\beta_k^T x)^2] = \mathbb{E}[\beta_k^T x x^T \beta_k] = \beta_k^T \mathbb{E}[x x^T] \beta_k = \beta_k^T \Sigma \beta_k \]
By the definition of the minimum eigenvalue of a symmetric matrix, for any vector $\beta_k$, we have $\beta_k^T \Sigma \beta_k \ge \lambda_{\min}(\Sigma) \|\beta_k\|_2^2$.
Combining this with the result from Lemma~\ref{lem:predictor_norm_bound}:
\[ \lambda_{\min}(\Sigma) \|\beta_k\|_2^2 \le \beta_k^T \Sigma \beta_k = \|\hat{y}_k\|_{L_2}^2 \le \|y\|_{L_2}^2 \]
Rearranging the terms gives the desired result:
\[ \|\beta_k\|_2^2 \le \frac{\|y\|_{L_2}^2}{\lambda_{\min}(\Sigma)} \implies \|\beta_k\|_2 \le \frac{\|y\|_{L_2}}{\sqrt{\lambda_{\min}(\Sigma)}} \]
\end{proof}

This provides a bound on the $L_2$ norm of the coefficients. We can convert this to the needed $L_1$ norm bound using the standard inequality $\|\beta\|_1 \le \sqrt{d} \|\beta\|_2$.

\subsection{Bounding the Norm of Empirically Learned Predictors}

We now combine the population-level norm bounds with the matrix concentration results to bound the norm of the coefficient vectors $\beta_k^{emp}$ that are actually learned from a finite data sample $D_m$.

\begin{lemma}[Bound on Empirically Learned Coefficients]
\label{lem:empirical_coeff_bound}
Let $\hat{y}_k^{emp} = (\beta_k^{emp})^T x$ be the predictor learned by agent $A_k$ from a training set $D_m$ of size $m$. Assume $|y| \le Y_{\max}$ and that the conditions for Corollary~\ref{cor:empirical_eigenvalue_bound} hold. Then, with probability at least $1-\delta$ over the draw of $D_m$, the $L_2$ norm of the learned coefficient vector is bounded:
\[ \|\beta_k^{emp}\|_2 \le \frac{Y_{\max}}{\sqrt{\lambda_{\min}(\Sigma)/2}} \]
Consequently, the $L_1$ norm is bounded by:
\[ \|\beta_k^{emp}\|_1 \le \frac{\sqrt{2d} \cdot Y_{\max}}{\sqrt{\lambda_{\min}(\Sigma)}} \]
\end{lemma}

\begin{proof}
The proof mirrors that of Lemma~\ref{lem:coeff_norm_bound}, but uses the empirical (in-sample) norms and the empirical covariance matrix $\hat{\Sigma}$.
The predictor $\hat{y}_k^{emp}$ is the orthogonal projection of the vector of labels $(y^{(1)}, \ldots, y^{(m)})$ onto the subspace spanned by the corresponding input feature vectors. Therefore, its empirical squared norm is bounded by the empirical squared norm of the labels:
\[ \frac{1}{m}\sum_{j=1}^m (\hat{y}_k^{emp}(x^{(j)}))^2 \le \frac{1}{m}\sum_{j=1}^m (y^{(j)})^2 \]
Given the assumption $|y| \le Y_{\max}$, the right-hand side is bounded by $Y_{\max}^2$.
The empirical squared norm of the predictor function can be written in terms of the empirical covariance matrix:
\[ \frac{1}{m}\sum_{j=1}^m ((\beta_k^{emp})^T x^{(j)})^2 = (\beta_k^{emp})^T \left(\frac{1}{m}\sum_{j=1}^m x^{(j)}(x^{(j)})^T\right) \beta_k^{emp} = (\beta_k^{emp})^T \hat{\Sigma} \beta_k^{emp} \]
From Corollary~\ref{cor:empirical_eigenvalue_bound}, we know that with probability at least $1-\delta$, $\lambda_{\min}(\hat{\Sigma}) > \lambda_{\min}(\Sigma)/2$.
Putting these pieces together:
\[ (\lambda_{\min}(\Sigma)/2) \|\beta_k^{emp}\|_2^2 < \lambda_{\min}(\hat{\Sigma}) \|\beta_k^{emp}\|_2^2 \le (\beta_k^{emp})^T \hat{\Sigma} \beta_k^{emp} \le Y_{\max}^2 \]
Rearranging gives the $L_2$ norm bound:
\[ \|\beta_k^{emp}\|_2 \le \frac{Y_{\max}}{\sqrt{\lambda_{\min}(\Sigma)/2}} = \frac{\sqrt{2} Y_{\max}}{\sqrt{\lambda_{\min}(\Sigma)}} \]
The $L_1$ norm bound follows from the standard inequality $\|\beta\|_1 \le \sqrt{d} \|\beta\|_2$:
\[ \|\beta_k^{emp}\|_1 \le \sqrt{d} \cdot \frac{\sqrt{2} Y_{\max}}{\sqrt{\lambda_{\min}(\Sigma)}} = \frac{\sqrt{2d} \cdot Y_{\max}}{\sqrt{\lambda_{\min}(\Sigma)}}. \]
\end{proof}

This lemma provides the explicit bound $\Lambda_1 = \frac{\sqrt{2d} \cdot Y_{\max}}{\sqrt{\lambda_{\min}(\Sigma)}}$ that we can substitute into Theorem~\ref{thm:gen_bound_linear}.

\subsection{Combined Performance Guarantee with Finite Samples}

We are now ready to state our main result, which combines the multi-agent performance guarantee from Theorem~\ref{thm:overall_guarantee_random_model} with the generalization bounds developed in this section. The following theorem shows that with high probability, the out-of-sample (true) error of the empirically trained predictors is close to the error of the optimal linear predictor, provided the path length in the DAG and the size of the training set are sufficiently large.

\begin{theorem}[Overall Performance Guarantee with Finite Samples in a DAG]
\label{thm:overall_guarantee_finite_sample}
Let all assumptions from Theorem~\ref{thm:overall_guarantee_random_model} hold. Further, assume the data distribution satisfies the assumptions for generalization (bounded features and labels, $\lambda_{\min}(\Sigma) > 0$). Let each agent train on a dataset $D_m$ of size $m$, where $m$ is large enough to satisfy the condition in Corollary~\ref{cor:empirical_eigenvalue_bound} for a given confidence $\delta_m$.

Let $\delta_L$ be the confidence for the random feature model and $\delta_m$ be the confidence for the sampling bounds.

Then for any target error $\eta > 0$, there exists a required path length $N_0$ (as defined in Theorem~\ref{thm:overall_guarantee_random_model}) such that for any path $P$ in the DAG of length at least $N_0$, the following holds with probability at least $1 - \delta_L - 2\delta_m$ (by a union bound). There exists an agent $A_{j^*}$ on the path $P$ (with $j^*$ being an index along the path, $j^* \le N_0$) such that for all subsequent agents $A_k$ on that same path, the true error of the empirically learned predictor $\hat{y}_k^{emp}$ is bounded by:
\[ \MSE{\hat{y}_k^{emp}} \le \MSE{g^*} + \eta + 2 \cdot \epsilon_{gen}(m, d, \delta_m) \]
where $g^*$ is the optimal linear predictor, and $\epsilon_{gen}$ is the one-sided generalization error from Theorem~\ref{thm:gen_bound_linear}, which is on the order of:
\[ \epsilon_{gen}(m, d, \delta_m) \approx O\left( \frac{\sqrt{d}X_{\max}Y_{\max}}{\sqrt{\lambda_{\min}(\Sigma)}} \sqrt{\frac{\log d}{m}} \right). \]
\end{theorem}

\begin{proof}
Let $R(f)$ denote the true risk (MSE) and $\hat{R}(f)$ denote the empirical risk. The proof proceeds by bounding the true risk of the empirically learned predictor, $R(\hat{y}_k^{emp})$, through a series of steps that connect it to the true risk of the optimal linear predictor, $R(g^*)$.

The argument relies on three key probabilistic events holding simultaneously:
\begin{enumerate}
    \item \textbf{Event $\mathcal{E}_L$ (Feature Coverage):} The random feature allocation model succeeds for a given path $P$, as per Theorem~\ref{thm:overall_guarantee_random_model}. This happens with probability at least $1-\delta_L$.
    \item \textbf{Event $\mathcal{E}_C$ (Covariance Concentration):} The empirical covariance matrix is close to the true one, satisfying Corollary~\ref{cor:empirical_eigenvalue_bound}. This ensures that the norm of the empirically learned coefficients $\beta_k^{emp}$ is bounded (Lemma~\ref{lem:empirical_coeff_bound}). This happens with probability at least $1-\delta_m$.
    \item \textbf{Event $\mathcal{E}_G$ (Uniform Generalization):} The generalization bound from Theorem~\ref{thm:gen_bound_linear} holds uniformly for all predictors in the hypothesis class $\mathcal{H}_{\Lambda_1}$, where $\Lambda_1$ is the norm bound from Lemma~\ref{lem:empirical_coeff_bound}. This happens with probability at least $1-\delta_m$.
\end{enumerate}
By a union bound, all three events occur with probability at least $1 - \delta_L - 2\delta_m$. We now proceed assuming all three events hold.

Let $\epsilon_{gen}$ be the one-sided generalization error bound from Theorem~\ref{thm:gen_bound_linear}:
\[ \epsilon_{gen} = C \cdot \Lambda_1 X_{\max} \sqrt{\frac{\log d}{m}} + (\Lambda_1 X_{\max} + Y_{\max})^2 \sqrt{\frac{\log(C'/\delta_m)}{2m}} \]
where $\Lambda_1 = \frac{\sqrt{2d} \cdot Y_{\max}}{\sqrt{\lambda_{\min}(\Sigma)}}$ is the high-probability bound on $\|\beta_k^{emp}\|_1$ established in Lemma~\ref{lem:empirical_coeff_bound} (which holds under event $\mathcal{E}_C$).

We can now bound the true risk of the empirical predictor $R(\hat{y}_k^{emp})$ for an agent $A_k$ on path $P$:
\begin{align*}
    R(\hat{y}_k^{emp}) &\le \hat{R}(\hat{y}_k^{emp}) + \epsilon_{gen} \tag{By uniform generalization, event $\mathcal{E}_G$} \\
    &\le \hat{R}(\hat{y}_k) + \epsilon_{gen} \tag{By definition of $\hat{y}_k^{emp}$ as the empirical risk minimizer} \\
    &\le R(\hat{y}_k) + 2\epsilon_{gen} \tag{By uniform generalization again, event $\mathcal{E}_G$} \\
\end{align*}
This chain of inequalities shows that the true risk of the empirically learned predictor is close to the true risk of its population counterpart, $\hat{y}_k$.

Next, we apply our main in-distribution result. Under event $\mathcal{E}_L$, Theorem~\ref{thm:overall_guarantee_random_model} guarantees that for our path $P$ (with length at least $N_0$), there exists an agent $A_{j^*}$ on the path such that for all subsequent agents $A_k$ on $P$:
\[ R(\hat{y}_k) \le R(g^*) + \eta \]

Combining this with our previous inequality, we get:
\[ R(\hat{y}_k^{emp}) \le (R(g^*) + \eta) + 2\epsilon_{gen} = \MSE{g^*} + \eta + 2\epsilon_{gen} \]
This result holds for any such agent $A_k$ on path $P$ with probability at least $1 - \delta_L - 2\delta_m$, which completes the proof.
\end{proof}

\section{Generalization Bounds for Non-Linear Models}
\label{sec:nonlinear_generalization}

In this section, we develop generalization bounds for the greedy orthogonal regression algorithm introduced in Section~\ref{sec:general_classes}. Our goal is to bridge the gap between the in-population performance guarantees derived previously and the performance of predictors learned from a finite training sample. We will show that, under standard assumptions on the complexity of the base function classes and the data distribution, the true error of an empirically trained predictor is close to its empirical error, allowing us to extend our main theorems to the finite-sample regime.

\subsection{Empirical Learning Process}

We first define the empirical version of the learning process. Let $D_m = \{(x^{(j)}, y^{(j)})\}_{j=1}^m$ be a training set of $m$ i.i.d. samples drawn from the true data distribution $\mathcal{D}$. All expectations $\mathbb{E}[\cdot]$ in the algorithm are replaced with empirical averages over $D_m$, denoted by $\hat{\mathbb{E}}[\cdot] = \frac{1}{m}\sum_{j=1}^m [\cdot]$.

The empirical version of the Greedy Orthogonal Regression algorithm for an agent $A_k$ proceeds as follows:

\begin{enumerate}
    \item \textbf{Initialization:} The initial feature pool $\mathcal{F}_k$ contains the empirically evaluated parent predictors, $\{\hat{y}_p^{emp}\}_{p \in \text{Pa}(k)}$. The initial predictor, $\hat{y}_k^{emp,0}$, is the solution to the empirical least squares problem: $\min_{f \in \text{span}(\mathcal{F}_k)} \hat{\mathbb{E}}[(f - Y)^2]$.
    
    \item \textbf{Greedy Selection Loop:} The algorithm iteratively adds functions from the local class $\mathcal{H}_k$.
    \begin{enumerate}
        \item Let the current empirical residual be $\hat{R} = Y - \hat{y}_k^{emp}$.
        \item Find the function $h_{\text{next}} \in \mathcal{H}_k$ with the highest empirical correlation with the residual:
        \[ h_{\text{next}} = \arg\max_{h \in \mathcal{H}_k} |\hat{\mathbb{E}}[h(x_{S_k}) \cdot \hat{R}]| \]
        \item \textbf{Termination:} If $|\hat{\mathbb{E}}[h_{\text{next}} \cdot \hat{R}]| < \Delta$, terminate.
        \item \textbf{Update:} Add $h_{\text{next}}$ to the pool $\mathcal{F}_k$ and update the predictor $\hat{y}_k^{emp}$ by solving the empirical least squares problem over the new, larger pool.
    \end{enumerate}
\end{enumerate}

Let $\hat{y}_k^{emp}$ denote the final predictor learned by this empirical process. Our main challenge is to bound the complexity of the class of functions to which $\hat{y}_k^{emp}$ belongs, which is necessary to apply uniform convergence bounds.

\subsection{Generalization Bound for Function Compositions}

To build our argument, we first state a standard generalization bound for predictors formed by taking linear combinations of functions from a base hypothesis class. This result connects the generalization error to the complexity of the base class,  measured by its Rademacher complexity.


The following theorem provides a uniform convergence bound for function classes constructed by taking $L_1$-norm-bounded linear combinations of functions from a base class.

\begin{theorem}[Generalization of $L_1$-norm Bounded Function Combinations]
\label{thm:gen_bound_nonlinear}
Let $\mathcal{H}$ be a base class of functions $h: \mathcal{X} \to \mathbb{R}$. Let $\mathcal{G}_{\Lambda_1}(\mathcal{H})$ be the class of predictors formed by taking linear combinations of functions from $\mathcal{H}$ with an $L_1$-norm constraint:
\[ \mathcal{G}_{\Lambda_1}(\mathcal{H}) = \left\{ g(x) = \sum_i \beta_i h_i(x) \mid h_i \in \mathcal{H}, \|\beta\|_1 \le \Lambda_1 \right\}. \]
Assume that for all $h \in \mathcal{H}$, $|h(x)| \le B_h$ almost surely, and for the learning task, $|y| \le Y_{\max}$. Let $D_m$ be a training set of $m$ i.i.d. samples.
Then for any $\delta \in (0,1)$, with probability at least $1-\delta$ over the random draw of $D_m$, for all $g \in \mathcal{G}_{\Lambda_1}(\mathcal{H})$:
\[ |R(g) - \hat{R}_{D_m}(g)| \le 4 (\Lambda_1 B_h + Y_{\max}) \Lambda_1 \mathcal{R}_m(\mathcal{H}) + (\Lambda_1 B_h + Y_{\max})^2 \sqrt{\frac{\ln(2/\delta)}{2m}}. \]
\end{theorem}

\begin{proof}
The proof follows the same three-step structure as in Theorem~\ref{thm:gen_bound_linear}.

\textbf{Step 1: Symmetrization and High-Probability Bound.}
Let $\mathcal{L} = \{ (x,y) \mapsto (g(x)-y)^2 \mid g \in \mathcal{G}_{\Lambda_1}(\mathcal{H}) \}$ be the loss class. For any $g \in \mathcal{G}_{\Lambda_1}(\mathcal{H})$, the function's output is bounded by $|g(x)| = |\sum_h \beta_h h(x)| \le \sum_h |\beta_h| |h(x)| \le B_h \norm{\beta}_1 \le \Lambda_1 B_h$. The loss is therefore bounded by $(|g(x)| + |y|)^2 \le (\Lambda_1 B_h + Y_{\max})^2$. Let $M_L = \Lambda_1 B_h + Y_{\max}$.

A standard result, combining a symmetrization argument with McDiarmid's inequality, shows that for any $\delta \in (0,1)$, with probability at least $1-\delta$:
\[ \sup_{g \in \mathcal{G}_{\Lambda_1}(\mathcal{H})} |R(g) - \hat{R}_{D_m}(g)| \le 2 \mathcal{R}_m(\mathcal{L}) + M_L^2 \sqrt{\frac{\ln(2/\delta)}{2m}}. \]

\textbf{Step 2: Contraction Principle.}
As before, we view the loss as a composition with $\phi(z)=z^2$, which is $2M_L$-Lipschitz on $[-M_L, M_L]$. Talagrand's contraction principle gives:
\[ \mathcal{R}_m(\mathcal{L}) \le 2M_L \mathcal{R}_m(\mathcal{G}_{\Lambda_1}(\mathcal{H})). \]

\textbf{Step 3: Bounding the Hypothesis Class Complexity.}
The final step is to bound the Rademacher complexity of $\mathcal{G}_{\Lambda_1}(\mathcal{H})$. A standard result, sometimes known as the Ledoux-Talagrand contraction lemma, bounds the complexity of an $L_1$-ball of a function class by the complexity of the base class:
\[ \mathcal{R}_m(\mathcal{G}_{\Lambda_1}(\mathcal{H})) \le \Lambda_1 \mathcal{R}_m(\mathcal{H}). \]

\textbf{Combining the Steps.}
Putting everything together, with probability at least $1-\delta$:
\begin{align*}
    \sup_{g \in \mathcal{G}_{\Lambda_1}(\mathcal{H})} |R(g) - \hat{R}_{D_m}(g)| &\le 2 \mathcal{R}_m(\mathcal{L}) + M_L^2 \sqrt{\frac{\ln(2/\delta)}{2m}} \\
    &\le 2 (2M_L \mathcal{R}_m(\mathcal{G}_{\Lambda_1}(\mathcal{H}))) + M_L^2 \sqrt{\frac{\ln(2/\delta)}{2m}} \\
    &\le 4M_L (\Lambda_1 \mathcal{R}_m(\mathcal{H})) + M_L^2 \sqrt{\frac{\ln(2/\delta)}{2m}} \\
    &= 4 (\Lambda_1 B_h + Y_{\max}) \Lambda_1 \mathcal{R}_m(\mathcal{H}) + (\Lambda_1 B_h + Y_{\max})^2 \sqrt{\frac{\ln(2/\delta)}{2m}}.
\end{align*}

\end{proof}

\subsection{Bounding the Norm of Empirically Learned Predictors}

A key challenge in applying Theorem~\ref{thm:gen_bound_nonlinear} is to establish a bound $\Lambda_1$ on the $L_1$ norm of the coefficients for the predictors $\hat{y}_k^{emp}$ learned by our algorithm. Similar to the linear case, we can derive such a bound by making a reasonable assumption on the data distribution and the function classes, which serves as a non-linear analogue to the assumption of a well-conditioned covariance matrix.

It is worth noting that the greedy selection mechanism of the algorithm provides a strong justification for this assumption, particularly for the features chosen iteratively from the local class $\mathcal{H}_k$. A new function $h_{\text{next}}$ is added only if its empirical correlation with the residual $\hat{R}$ is at least $\Delta > 0$. This implies that the component of $h_{\text{next}}$ that is empirically orthogonal to the span of previously selected features must have a norm of at least $\Delta / Y_{\max}$. This provides a quantitative guarantee that each new local feature is substantially linearly independent from those already in the pool.

However, the assumption is still formally required to cover two key aspects. First, and most critically, the guarantee does not apply to the initial set of features given to the agent---namely, the parent predictions $\{\hat{y}_p^{\text{emp}}\}_{p \in \text{Pa}(k)}$. These are fixed inputs that are not filtered by the greedy selection process and could be highly collinear. Second, the step-wise guarantee of independence does not automatically translate into a simple lower bound on the minimum eigenvalue of the final Gram matrix for the full set of selected functions. Therefore, we retain the assumption to ensure that the \emph{entire} basis for any agent, including parent predictions, is well-conditioned.

\begin{assumption}[Empirical Incoherence]
\label{assump:empirical_incoherence}
Let $\mathcal{H}_{univ}$ be the universe of all functions that can be selected by any agent. For any finite subset of functions $\mathcal{F} \subset \mathcal{H}_{univ}$ of size at most $T_{\max}$, let their vector representation on the training data $D_m$ be $h_1, \ldots, h_{|\mathcal{F}|} \in \mathbb{R}^m$. Let $\mathbf{H}$ be the $m \times |\mathcal{F}|$ matrix whose columns are these vectors. We assume that with high probability over the draw of $D_m$, the minimum eigenvalue of the empirical Gram matrix $\frac{1}{m}\mathbf{H}^T\mathbf{H}$ is bounded below by a positive constant $\lambda_{\min, H} > 0$.

This assumption states that any reasonably-sized set of functions that could form a basis for our predictors are not perfectly collinear on the training data.
\end{assumption}

\begin{lemma}[Bound on $L_1$ Norm of Coefficients]
\label{lem:l1_norm_bound}
Under Assumption~\ref{assump:empirical_incoherence}, for any agent $k$, let $\hat{y}_k^{\text{emp}}$ be the predictor learned by the empirical greedy algorithm (described at the start of Section~\ref{sec:nonlinear_generalization}) on dataset $D_m$. Let $\beta_k$ be its vector of coefficients. Then the $L_1$ norm of the coefficients is bounded as:
\[ ||\beta_k||_1 \le \frac{T_k \cdot B_{\text{basis}} \cdot Y_{\max}}{\lambda_{\min, H}} \]
where $T_k = |\text{Pa}(k)| + T_{k, \text{local}}$ is the number of basis functions used by agent $k$, $B_{\text{basis}} = \max(B_h, Y_{\max})$, and $\lambda_{\min, H}$ is the constant from Assumption~\ref{assump:empirical_incoherence}. The number of locally selected features, $T_{k, \text{local}}$, is bounded by $(B_h Y_{\max} / \Delta)^2$.
\end{lemma}

\begin{proof}
The predictor $\hat{y}_k^{\text{emp}}$ is the solution to a standard empirical risk minimization (least squares) problem over the chosen basis functions. Let the basis be $\mathcal{F}_k = \{\hat{y}_p^{\text{emp}}\}_{p \in \text{Pa}(k)} \cup \{h_j\}_{j=1}^{T_{k, \text{local}}}$, and let $\mathbf{H}_k$ be the $m \times T_k$ matrix of these basis functions evaluated on the training data. The coefficient vector $\beta_k$ is given by the normal equations:
\[ \left(\frac{1}{m}\mathbf{H}_k^T\mathbf{H}_k\right) \beta_k = \frac{1}{m}\mathbf{H}_k^T y. \]
We can bound the $L_2$ norm of $\beta_k$ as follows:
\[ ||\beta_k||_2 \le \left\|\left(\frac{1}{m}\mathbf{H}_k^T\mathbf{H}_k\right)^{-1}\right\|_{\text{op}} \left\|\frac{1}{m}\mathbf{H}_k^T y\right\|_2. \]
From Assumption~\ref{assump:empirical_incoherence}, the operator norm of the inverse Gram matrix is bounded by $1/\lambda_{\min, H}$.

For the second term, the $j$-th entry of the vector $\frac{1}{m}\mathbf{H}_k^T y$ is $\frac{1}{m} \sum_{i=1}^m f_j(x_i) y_i$ for some basis function $f_j \in \mathcal{F}_k$. The functions in the basis are bounded by $B_{\text{basis}} = \max(B_h, Y_{\max})$ (since parent predictions are also bounded by $Y_{\max}$ by induction). Thus, the infinity norm of this vector is bounded:
\[ \left\|\frac{1}{m}\mathbf{H}_k^T y\right\|_{\infty} \le B_{\text{basis}} Y_{\max}. \]
This implies an $L_2$ bound of $\|\frac{1}{m}\mathbf{H}_k^T y\|_2 \le \sqrt{T_k} B_{\text{basis}} Y_{\max}$.

Combining these, we get a bound on the $L_2$ norm of the coefficients:
\[ ||\beta_k||_2 \le \frac{\sqrt{T_k} B_{\text{basis}} Y_{\max}}{\lambda_{\min, H}}. \]
Finally, converting to the $L_1$ norm gives:
\[ ||\beta_k||_1 \le \sqrt{T_k} ||\beta_k||_2 \le \frac{T_k B_{\text{basis}} Y_{\max}}{\lambda_{\min, H}}. \]
To bound $T_k$, we note that the number of parent predictions is fixed at $|\text{Pa}(k)|$. The number of locally selected features, $T_{k, \text{local}}$, is bounded because each step of the greedy algorithm reduces the empirical squared error by at least $\Delta^2 / B_h^2$. Since the initial error is at most $Y_{\max}^2$, the number of steps is bounded by $Y_{\max}^2 / (\Delta^2/B_h^2) = (B_h Y_{\max}/\Delta)^2$.
\end{proof}

With this bound on the $L_1$ norm, we can now state the main generalization result for the predictors learned by our algorithm.

\begin{corollary}[Generalization Guarantee for Empirically Learned Predictors]
\label{cor:final_gen_bound}
Under the assumptions of Theorem~\ref{thm:gen_bound_nonlinear} and Assumption~\ref{assump:empirical_incoherence}, let $\hat{y}_k^{\text{emp}}$ be the predictor for agent $k$ learned via the empirical greedy algorithm. Let $\Lambda_{1,k} = \frac{T_k B_{\text{basis}} Y_{\max}}{\lambda_{\min, H}}$. Then with probability at least $1-\delta$ over the draw of the training data $D_m$:
\[ |R(\hat{y}_k^{\text{emp}}) - \hat{R}(\hat{y}_k^{\text{emp}})| \le 4 (\Lambda_{1,k} B_h + Y_{\max}) \Lambda_{1,k} \mathcal{R}_m(\mathcal{H}) + (\Lambda_{1,k} B_h + Y_{\max})^2 \sqrt{\frac{\ln(2/\delta)}{2m}}. \]
\end{corollary}

\begin{proof}
The result follows immediately by applying Theorem~\ref{thm:gen_bound_nonlinear} and substituting the value of $\Lambda_1$ with the high-probability bound $\Lambda_{1,k}$ derived in Lemma~\ref{lem:l1_norm_bound}.
\end{proof}

Under this assumption, we can bound the norm of the coefficients for any predictor produced by the empirical greedy algorithm.

\begin{lemma}[Bound on Empirically Learned Coefficients]
\label{lem:empirical_coeff_bound_nonlinear}
Let $\hat{y}_k^{emp}$ be a predictor learned by the empirical greedy algorithm from a training set $D_m$. The predictor can be written as $\hat{y}_k^{emp} = \sum_{i=1}^T \beta_i h_i$, where $\{h_i\}_{i=1}^T$ is the final pool of functions. Assume $|y| \le Y_{\max}$ and that Assumption~\ref{assump:empirical_incoherence} holds for the set $\{h_i\}$. Then the $L_2$ norm of the learned coefficient vector is bounded:
\[ \|\beta\|_2 \le \frac{Y_{\max}}{\sqrt{\lambda_{\min, H}}} \]
Consequently, if the pool size is at most $T_{\max}$, the $L_1$ norm is bounded by:
\[ \|\beta\|_1 \le \frac{\sqrt{T_{\max}} \cdot Y_{\max}}{\sqrt{\lambda_{\min, H}}} \]
\end{lemma}

\begin{proof}
The proof mirrors that of Lemma~\ref{lem:empirical_coeff_bound} for the linear case.
The predictor $\hat{y}_k^{emp}$ is the orthogonal projection of the vector of labels $(y^{(1)}, \ldots, y^{(m)})$ onto the subspace spanned by the function vectors $\{h_i\}$. Therefore, its empirical squared norm is bounded by the empirical squared norm of the labels:
\[ \frac{1}{m}\|\hat{y}_k^{emp}\|_2^2 = \hat{\mathbb{E}}[(\hat{y}_k^{emp})^2] \le \hat{\mathbb{E}}[y^2] \le Y_{\max}^2 \]
The empirical squared norm of the predictor can also be written in terms of the empirical Gram matrix $\hat{\Sigma}_H = \hat{\mathbb{E}}[h h^T]$:
\[ \hat{\mathbb{E}}[(\hat{y}_k^{emp})^2] = \beta^T \hat{\Sigma}_H \beta \]
By Assumption~\ref{assump:empirical_incoherence}, we have $\beta^T \hat{\Sigma}_H \beta \ge \lambda_{\min, H} \|\beta\|_2^2$.
Putting these pieces together:
\[ \lambda_{\min, H} \|\beta\|_2^2 \le \beta^T \hat{\Sigma}_H \beta = \hat{\mathbb{E}}[(\hat{y}_k^{emp})^2] \le Y_{\max}^2 \]
Rearranging gives the $L_2$ norm bound. The $L_1$ norm bound follows from the standard inequality $\|\beta\|_1 \le \sqrt{T} \|\beta\|_2$.
\end{proof}

This lemma provides the explicit bound $\Lambda_1 = \frac{\sqrt{T_{\max}} \cdot Y_{\max}}{\sqrt{\lambda_{\min, H}}}$ that we can substitute into Theorem~\ref{thm:gen_bound_nonlinear} to control the generalization error of any predictor generated by our distributed algorithm.

\subsection{Main Generalization Theorem}

We can now combine our results to state the main generalization theorem for the predictors learned by the empirical Greedy Orthogonal Regression algorithm.

\begin{theorem}[Generalization Guarantee for Empirically Trained Predictors]
\label{thm:nonlinear_main_gen_guarantee}
Let $\mathcal{H}_{univ} = \bigcup_{k=1}^N \mathcal{H}_k$ be the universe of all possible local functions. Let $\hat{y}_k^{emp}$ be the predictor for agent $k$ learned from a training set $D_m$ of size $m$. Assume:
\begin{itemize}
    \item The conditions of Theorem~\ref{thm:gen_bound_nonlinear} and Lemma~\ref{lem:empirical_coeff_bound_nonlinear} hold (i.e., bounded functions, bounded labels, and empirical incoherence).
    \item The maximum number of functions selected by any agent is bounded by $T_{\max}$.
\end{itemize}
Let $\Lambda_1 = \frac{\sqrt{T_{\max}} \cdot Y_{\max}}{\sqrt{\lambda_{\min, H}}}$.
Then for any $\delta \in (0,1)$, with probability at least $1-\delta$ over the draw of $D_m$, the true risk of the learned predictor is close to its empirical risk:
\[ |R(\hat{y}_k^{emp}) - \hat{R}_{D_m}(\hat{y}_k^{emp})| \le 4 (\Lambda_1 B_h + Y_{\max}) \Lambda_1 \mathcal{R}_m(\mathcal{H}_{univ}) + (\Lambda_1 B_h + Y_{\max})^2 \sqrt{\frac{\ln(2/\delta)}{2m}}. \]
\end{theorem}

\begin{proof}
The proof is a direct application of the preceding results. From Lemma~\ref{lem:empirical_coeff_bound_nonlinear}, we know that with high probability, the learned predictor $\hat{y}_k^{emp}$ is a linear combination of at most $T_{\max}$ functions from $\mathcal{H}_{univ}$ with an $L_1$ norm bounded by $\Lambda_1$. Therefore, $\hat{y}_k^{emp}$ belongs to the class $\mathcal{G}_{\Lambda_1}(\mathcal{H}_{univ})$. The result follows immediately by applying Theorem~\ref{thm:gen_bound_nonlinear} to this class.
\end{proof}

\end{document}